\documentclass[11pt]{article}

\usepackage{inputenc} 
\usepackage[T1]{fontenc}    
\usepackage{url}            
\usepackage{booktabs}       
\usepackage{amsfonts}      
\usepackage{microtype} 
\usepackage{graphicx}
\usepackage{subfigure} 
\usepackage[unicode]{hyperref}
\usepackage[english]{babel}
\usepackage{float}
\usepackage{amsmath}
\usepackage{amssymb}
\usepackage{amsthm}
\usepackage{bbm}
\usepackage{bbold}
\usepackage{caption}
\usepackage{color}
\usepackage{multicol}
\usepackage{cite}
\usepackage{algorithmic}
\usepackage{algorithm}
\usepackage{wrapfig}
\usepackage{enumerate}

\usepackage{scrwfile}
\TOCclone[\contentsname~(\appendixname)]{toc}{atoc}
\newcommand\StartAppendixEntries{}
\AfterTOCHead[toc]{%
	\renewcommand\StartAppendixEntries{\value{tocdepth}=-10000\relax}%
}
\AfterTOCHead[atoc]{%
	\edef\maintocdepth{\the\value{tocdepth}}%
	\value{tocdepth}=-10000\relax%
	\renewcommand\StartAppendixEntries{\value{tocdepth}=\maintocdepth\relax}%
}
\newcommand*\appendixwithtoc{%
	\cleardoublepage
	\appendix
	\addtocontents{toc}{\protect\StartAppendixEntries}
	\listofatoc
}

\newtheorem{remark}{Remark}
 
\newtheorem{theorem}{Theorem}
\newtheorem{theo}{Theorem}
\newtheorem{lemma}{Lemma} 
 
\newtheorem{proposition}{Proposition}
\newtheorem{definition}{Definition}
\newtheorem{assumption}{Assumption}

\newcommand{\OM}{\mathcal{O}}
\newcommand{\cD}{\mathcal{D}}
\newcommand{\set}[1]{\ensuremath{\left\{ #1\right\}}}
\newcommand{\abs}[1]{\ensuremath{\left\vert #1\right\vert}}

\title{\vspace{-10mm}\Huge{Safe Learning under Uncertain  Objectives and Constraints}
	\vspace{5mm}}

\author{
\begin{tabular}[t]{c@{\extracolsep{8em}}c} 
\textbf{Mohammad Fereydounian}  & \textbf{Zebang Shen}\\
ESE  Department & ESE  Department\\
University of Pennsylvania & University of Pennsylvania \\
mferey@seas.upenn.edu &  {zebang@seas.upenn.edu}
\end{tabular}\\
\\
\begin{tabular}[t]{c@{\extracolsep{8em}}c} 
\textbf{Aryan Mokhtari} & \textbf{Amin Karbasi} \\
ECE  Department & EE  Department\\
University of Texas at Austin & Yale University \\
mokhtari@austin.utexas.edu & amin.karbasi@yale.edu
\end{tabular}\\
\\
\begin{tabular}[t]{c@{\extracolsep{0em}}c} 
\textbf{Hamed Hassani}\\
 ESE  Department\\
 University of Pennsylvania \\
hassani@seas.upenn.edu
\end{tabular}\\
}

\date{}
\begin{document}

\maketitle

\begin{abstract}
  In this paper, we consider non-convex optimization problems under \textit{unknown} yet safety-critical constraints. Such problems naturally arise in a variety of domains including robotics, manufacturing, and medical procedures, where it is infeasible to know or identify all the constraints. Therefore, the parameter space should be explored in a conservative way to ensure that none of the constraints are violated during the optimization process once we start from a safe initialization point. To this end, we develop an algorithm called Reliable Frank-Wolfe (Reliable-FW).  Given a general non-convex function and an unknown polytope constraint,  Reliable-FW simultaneously learns the landscape of the objective function and the boundary of the safety polytope. More precisely, by assuming that Reliable-FW  has access to a (stochastic) gradient oracle of the objective function and a noisy feasibility oracle of the safety polytope, it finds an $\epsilon$-approximate first-order stationary point with the optimal  ${\mathcal{O}}({1}/{\epsilon^2})$ gradient oracle complexity (resp. $\tilde{\mathcal{O}}({1}/{\epsilon^3})$ (also optimal) in the stochastic gradient setting), while ensuring the safety of all the iterates. Rather surprisingly, Reliable-FW only makes $\tilde{\mathcal{O}}(({d^2}/{\epsilon^2})\log 1/\delta)$  queries to the noisy feasibility oracle (resp. $\tilde{\mathcal{O}}(({d^2}/{\epsilon^4})\log 1/\delta)$ in the stochastic gradient setting) where $d$ is the dimension and $\delta$ is the reliability parameter, tightening the existing  bounds even for safe minimization of convex functions. We further specialize our results to the case that the objective function is convex. For this case we show that to achieve an $\epsilon$-accurate solution the Reliable-FW method requires $\tilde{\mathcal{O}}({1}/{\epsilon^2})$ stochastic gradient evaluations and $\tilde{\mathcal{O}}((d^2/\epsilon^4){\log 1/\delta})$  queries to the noisy feasibility oracle (NFO) for the case that we have access to stochastic gradients, and $\mathcal{O}({1}/{\epsilon})$ gradient computations and $\tilde{\mathcal{O}}((d^2/\epsilon^2){\log 1/\delta})$ NFO queries for the deterministic setting. A crucial component of our analysis is to introduce and apply a technique called  geometric shrinkage in the context of safe optimization. 
\end{abstract}

\section{Introduction}\label{Introduction}

In this paper, we study a general class of constrained non-convex optimization problems in the presence of uncertainty in both the objective function and the constraint set.  Specifically, we focus on
\begin{align}\label{P}
\min _{x \in \mathcal{D}} f(x), 
\end{align}
where the objective function $f: \mathbb{R}^{d} \rightarrow \mathbb{R}$ is a smooth and potentially non-convex function and the constraint set $\mathcal{D} \subseteq \mathbb{R}^d$ is a compact and convex set.  We focus on a scenario where we do not have access to the exact constraint set $\mathcal{D}$ nor the exact gradient oracle of the objective function~$f$.  Instead, noisy estimates of the feasibility of  a queried point with respect to $\mathcal{D}$ as well as noisy estimates of the gradient of $f$ are assumed accessible. Our goal is to provide optimization procedures to solve problem \eqref{P} which remain \emph{safe} along the \emph{entire optimization path}. That is, we require all the queries to the gradient and feasibility oracles to remain within the (unknown) feasible domain $\mathcal{D}$.

This scenario, a.k.a., \emph{safe learning}, encompasses a broad set of safety-critical applications ranging from autonomous systems \cite{SafeAut}, to robotics \cite{SafeRobot}, and medical diagnosis \cite{med}, to name a few. 
For instance, in medical diagnosis, physicians seek to find the most efficient (optimal) drug combinations or therapies. For such configurations to be tested, they must lie within a set of harmless (safe) therapies.  In such applications,  computing the exact constraint set $\mathcal{D}$ in advance is highly expensive in terms of sample complexity, and the main challenge is to optimize with the least number of queries from the function and feasibility oracles  while ensuring safety all the time. 

Optimization algorithms for solving the constrained problem \eqref{P} can be divided into two major groups: (i) Projected gradient methods \cite{nesterov2013introductory} in which the iterates are projected back to the feasible set following an optimization step. (ii) Conditional gradient methods (Frank-Wolfe methods) \cite{frank1956algorithm,jaggi2013revisiting,FW3} in which the projection operator is replaced by a linear optimization oracle. 
Stochastic variants of these two classes of methods which handle uncertainty in the gradient have been studied extensively in the literature \cite{robbins1951stochastic,nemirovski1978cezari,nemirovskii1983problem,DBLP:conf/icml/HazanK12, DBLP:conf/icml/HazanL16,lan2016conditional,mokhtari2018stochastic,pmlr-v80-qu18a}. However, since a direct access to the feasible set is essential in both of the aforementioned methods, they are incapable of solving the safe learning problem \eqref{P} when the feasible set $\cD$ is not explicitly given.
To fill this gap, prior work \cite{scl, LB}  has developed a framework for solving problem \eqref{P} under uncertain objectives and  safety-critical constraints. However, as we will elaborate below, the resulting sample complexities  are sub-optimal. 

\vspace{-1mm}
In this paper, we assume two types of noisy oracles accessible by the optimization procedure, namely, the Noisy Feasibility Oracle (NFO) and the Stochastic First-order Oracle (SFO).
Specifically, in response to a query, NFO returns a continuous noisy measurement of the feasibility and SFO returns a gradient estimation at the query point. We seek to address the following question: 
\addtolength\leftmargini{-0.3in}
\begin{quote}
\vspace{-1mm}
	\it ``Can we guarantee safe learning with optimal SFO complexity while relying only on NFO?''
	\vspace{-2mm}
\end{quote}
To this aim, a non-convex variant of Frank-Wolfe named Reliable Frank-Wolfe (Reliable-FW) is proposed, which estimates the feasibility in each iteration.
By ensuring a that the function value is maximally reduced at each iteration (via a novel convergence analysis), Reliable-FW achieves a tight SFO complexity for general non-convex objectives.  Simultaneously, Reliable-FW uses a novel technique based on a geometric shrinkage of the feasible set to keep the NFO complexity low while ensuring the feasibility of all the iterates. 

\begin{table}[t!]
	\centering
	\caption{Comparisons with most related constrained safe learning methods under uncertain constraints.  Abbreviations: stochastic (stoch.), deterministic (det.), and complexity (comp.)} 
	\label{table}
	\begin{tabular}{@{}cccccccc@{}}
		\toprule
		& \multicolumn{3}{c}{Objective function} &\phantom{}& \multicolumn{2}{c}{Constraint} \\
		\cmidrule{2-4} \cmidrule{6-7}
		Ref.    		& {Class}  & {Oracle type}			& {Oracle  comp.}			  && {Class} & {Oracle comp.} \\
		\midrule
		\cite{LB}		& nonconvex    & stoch. zeroth-order	&	$\tilde{\OM}({d^3/\epsilon^7})$ && general & $\tilde{\OM}(d^3/\epsilon^7)$ \\
		This paper 		& nonconvex    & stoch.	first-order	&	$\tilde{\OM}({1/\epsilon^3})$   	 && linear  & $\tilde{\OM}(d^2/\epsilon^4)$ \\
		This paper		& nonconvex    & det.   first-order&	$\OM({1/\epsilon^2})$  		 && linear  & $\tilde{\OM}(d^2/\epsilon^2)$ \\
		\cite{scl}		& convex        & det.   first-order	&	$\tilde{\OM}({1/\epsilon})$  && linear  & $\tilde{\OM}(d^3/\epsilon^2)$ \\
		This paper		& convex       & det.   first-order&	$\OM({1/\epsilon})$    		 && linear  & $\tilde{\OM}(d^2/\epsilon^2)$ \\
		\cite{scl}		& convex       & stoch. first-order	&	$\tilde{\OM}({1/\epsilon^3})$&& linear  & $\tilde{\OM}(d^3/\epsilon^2)$ \\
		This paper		& convex       & stoch.   first-order&	$\tilde{\OM}({1/\epsilon^2})$    		 && linear  & $\tilde{\OM}(d^2/\epsilon^4)$ \\
		\bottomrule
	\end{tabular}
\end{table}

\paragraph{Our contributions.}  We tackle the non-convex optimization problem \eqref{P} under the safety criterion when the feasible set $\cD$ is a compact polytope. Our proposed Reliable-FW method provably finds an $\epsilon$-approximate first-order stationary point by calling $\mathcal{O}({1}/{\epsilon^2})$ deterministic gradient evaluations, or by calling $\tilde{\OM}({1}/{\epsilon^3})$ stochastic gradient evaluations in the case where only stochastic gradients are available. Perhaps surprisingly, these results are  tight, i.e., they match the optimal results for optimization under \emph{known} constraints: Our sample complexity bounds match the best known rates when the constraints our known  for both cases of deterministic \cite{lacoste2016convergence,ghadimi2016accelerated,ghadimi2016mini,mokhtari2018escaping}  and stochastic \cite{shen2019complexities,yurtsever2019conditional,hassani2019stochastic, carmon2019lower,arjevani2019lower} gradient oracles. 
More importantly, under a deterministic gradient oracle, we prove that probability at least $1-\delta$, Reliable-FW 
guarantees  feasibility of all the iterates while  making  $\tilde{\mathcal{O}}((d^2/\epsilon^2){\log 1/\delta})$ and $\tilde{\mathcal{O}}((d^2/\epsilon^4){\log 1/\delta})$ NFO queries for deterministic and stochastic gradient settings , respectively. 
 A comparison between our results and state-of-the-art is shown in Table~\ref{table}.
We also study the special case that the objective function is convex. In this case, when we have access to a stochastic gradient oracle, we show that our proposed method finds an $\epsilon$-accurate solution after $\tilde{\mathcal{O}}({1}/{\epsilon^2})$ stochastic gradient evaluations and $\tilde{\mathcal{O}}((d^2/\epsilon^4){\log 1/\delta})$ NFO queries. For the setting that access to exact gradient is possible, our proposed method finds an $\epsilon$-accurate solution after $\mathcal{O}({1}/{\epsilon})$ gradient computations and $\tilde{\mathcal{O}}((d^2/\epsilon^2){\log 1/\delta})$ NFO queries. The novelty of our results lies in designing new method that handle the interplay between uncertainties in the objective and constraints, and new analysis tools (e.g. geometric shrinkage) which enable to obtain tight bounds.

\paragraph{Related Work.} The most related works to the this study are \cite{scl} and \cite{LB}. The safe learning problem in \eqref{P} was first studied in \cite{scl} for the case that  $f$ is convex and $\cD$ is a compact polytope. Specifically, \cite{scl} proposed a method named Safe Frank-Wolfe (SFW) that at each step, a certain amount of feasibility queries are made within a vicinity of the current iterates, which together with previous accumulated knowledge about $\cD$, enable a feasible set estimation $\hat{\cD}$. A standard Frank-Wolfe step is taken over the estimated polytope $\hat{\cD}$ to update the iterates.
Consequently, when $f$ \emph{convex} and its gradient is \emph{deterministically} given, to find an $\epsilon$-suboptimal solution (i.e., $x_{\epsilon}$ with $f(x_\epsilon) - f(x^*) \leq \epsilon$) while ensuring the feasibility of all the iterates with probability at least $1-\delta$, SFW requires $\tilde{\mathcal{O}}({1}/{\epsilon})$ gradient evaluations and $\tilde{\mathcal{O}}((d^3/\epsilon^2){\log 1/\delta})$  NFO queries. \cite{scl} also described a simple stochastic generalization of SFW, still for the convex case, by using a large batch in each iteration to estimate the gradient that requires $\tilde{\mathcal{O}}({1}/{\epsilon^3})$ SFO queries. 

In addition,  \cite{LB} studied the non-convex safe learning problem with general uncertain non-convex constraints using zeroth-order information of both the objective and the constraint functions.
Specifically, \cite{LB} proposed to add a log barrier regularization to the Lagrangian of the original non-linear programming problem and then approximate the gradient of the Lagrangian via the finite-difference scheme with zeroth-order information to optimize.
While the proposed 0-LBM and s0-LBM method apply to possibly all the smooth non-linear programming of interest, such generality comes at a cost of high oracle complexities.
Concretely, to achieve an $\epsilon$-approximate unscaled KKT point (an optimality criterion proposed in \cite{LB} that describes the Lagrangian gradient norm), the total zeroth-order oracle complexity of both the objective and the constraint functions is $\tilde{\OM}({d^3}/{\epsilon^7})$. As highlighted in Table~\ref{table}, our Reliable-FW method improves both SFO and NFO complexities of the method in  \cite{LB} for nonconvex settings.


Several works study robust optimization with chance constraints \cite{CB1,CB2,CB3,CB4,CB5,CB6,CB7,CB8,CB9}. Despite some similarities, our setting is fundamentally different as we rely on noisy queries rather than stochastic constraints.

\paragraph{Notation.} 
 In comparison to $\OM(\cdot)$, the notation $\tilde{\OM}(\cdot)$ hides logarithmic terms.  We define the vector $e_j\in\mathbb{R}^d$ to have a $1$ in its $j$-th element and zero otherwise.  For a positive integer $m$, $[m]:=\{1,\ldots,m\}$, $\mathbb{0}_m:=(0,\ldots,0)^{\top}\in \mathbb{R}^{m}$, and $\mathbb{1}_m:=(1,\ldots,1)^{\top}\in \mathbb{R}^{m}$.  For $x, y\in \mathbb{R}^d$, $\langle x,y\rangle:=x^{\top}y$ and $x\leq y$ means that all elements of $y-x$ are non-negative. We use $\|\cdot\|$to the 2-norm for matrices and vectors.  We use $\rho_{\min}(A)$ and $\rho_{\max}(A)$ to denote the minimum and maximum singular values of matrix $A$, respectively.
For $A\in \mathbb{R}^{d\times m}$ and $B\in \mathbb{R}^{d\times n}$, $[A,B]\in \mathbb{R}^{d\times {(m+n)}}$ denotes the row-wise concatenated matrix. We use $\mathbb{P}\{\mathcal{R}\}$ and $\mathbb{P}\{\mathcal{R}|\mathcal{J}\}$ to denote  the probability of $\mathcal{R}$ and $\mathcal{R}$ given $\mathcal{J}$, respectively. For a closed and convex set $S \subseteq \mathbb{R}^{d}$, the operation $\pi_{S}: \mathbb{R}^{d}\rightarrow \mathbb{R}^d$ denotes the projection on $S$, i.e., $\pi_{S}(x):=\arg\min_{v\in S}\|x-v\|$. The covariance of a vector-valued random variable $x$ is a matrix defined as $\operatorname{Cov}(x):=\mathbb{E}[xx^{\top}]-\mathbb{E}[x]\mathbb{E}[x]^{\top}$. We represent sequences as $\{a_t\}:=(a_0,\ldots,a_t)$.


\section{Problem Formulation} \label{PF} 

In this section, we  introduce the underlying setting of our problem. Recall our main formulation in~\eqref{P}, where we aim to minimize the objective function $f: \mathbb{R}^{d} \rightarrow \mathbb{R}$ over a compact and convex set $\mathcal{D}$.  We further assume that $f$ is differentiable and its gradient $\nabla f$ satisfies the following conditions. 

\begin{assumption}\label{f}
	The objective function gradient $\nabla f$ is  $L$-Lipschitz on $\mathcal{D}$, i.e., for all $x,y\in\cD$, we have $\|\nabla f(x)-\nabla f(y)\|\leqslant L\|x-y\|$.
	Moreover, its norm is uniformly bounded by $M$, i.e., for all $x\in\cD$, we have $ \|\nabla f(x)\|\leqslant M$.
\end{assumption}

As mentioned earlier, we focus on the case that we are uncertain about both the objective function $f$ and the feasible set $\mathcal{D}$. Specifically, we assume that we do not have access to the exact value of the function $f(x)$ and its gradient $\nabla f(x)$, and instead we only have access to a gradient estimator (oracle) denoted by $G(x,\xi)$, where $\xi$ is a random variable representing the randomness in $G(x,\xi)$. We assume that the estimator  $G(x,\xi)$ satisfies the following conditions.

 
\begin{assumption}\label{G}
	The gradient estimator $G$ satisfies the following conditions. 
	\begin{enumerate}[i)]
	\item $G(x,\xi)$ is an unbiased estimator of $\ \nabla f(x)$, i.e., $\mathbb{E}_{\xi}[G(x,\xi)] = \nabla f(x)$.
		\item $G(x,\xi)$ is $L_0$-Lipschitz, i.e., for all $x,y\in\cD$ we have $\|G(x,\xi)-G(y,\xi)\|\leqslant L_0\|x-y\|$.
		\item  The support of $G(x,\xi)$ is bounded by $\sigma_0$, i.e., $\|G(x,\xi)-\nabla f(x)\| \leqslant \sigma_0$, for all $x\in\cD$, w.p.1. 
		 \end{enumerate}
\end{assumption} 

  In connection to what was mentioned earlier in Section~\ref{Introduction}, the total number of calls to the stochastic oracle $G(x,\xi)$ throughout the procedure is referred to as the SFO complexity.

Next, we formally discuss the conditions on our feasible set $\mathcal{D}$, i.e., safety constraint, and clarify its nature of uncertainty. In this paper, we focus on the case that the constraint set $\mathcal{D}$ is a polytope and defined as the intersection of $m$ linear inequity constraints, i.e., $\cD = \{x \in \mathbb{R}^{d}: \, A x-b\leqslant \mathbb{0}_m\}$, where $A\in\mathbb{R}^{m\times d}$ and $b\in\mathbb{R}^{m}$. Since $\cD$ is an intersection of $m$ closed half-planes, compactness of $\cD$ is equivalent to its boundedness. We further assume that $\mathcal{D}$ has a bounded diameter 

\begin{assumption}\label{D}
	The feasible set $\cD$ is compact where its diameter is bounded by $\Lambda$ and its radius is bounded by $\Gamma$, i.e., $\forall x,y\in \cD:\,\, \|x\|\leqslant \Gamma,\,\,\|x-y\|\leqslant\Lambda$.
\end{assumption}

We focus on the case that we are uncertain about the set $\cD$ and the matrix $A$ and vector $b$ are unknown. More precisely, we assume that for a given point $x\in\mathbb{R}^d$, the feasibility of $x$ can only be queried via a stochastic oracle called Noisy Feasibility Oracle (NFO). In response to this query, NFO returns $y = Ax - b + \theta \in \mathbb{R}^m$ with $\theta$ being an $m$-dimensional noise term. Here, $x$ is called an NFO query point and $y$ is called a \emph{measurement} at point $x$.  The total number of queries of this type throughout our procedure is referred to as the NFO complexity. In this paper, we focus on the case that the elements of measurement  noise  $\theta$  are sub-Gaussian. Note that $X$ is called a zero mean $\sigma$-sub-Gaussian random variable if $\mathbb{E}[X]=0$ and $\mathbb{E}[e^{\alpha X}] \leqslant \exp ({\alpha^{2} \sigma^{2}}/{2})$ for all $\alpha\in\mathbb{R}$.

\begin{assumption}\label{NFO}
	Each element of the  measurement noise $\theta\in\mathbb{R}^m$ is a zero mean $\sigma$-sub-Gaussian random variable.  Moreover, every two elements of the same $\theta$ are independent and two different $\theta$s in two different NFO queries are also independent. 
\end{assumption}

Additionally, we assume that an initial feasible point $x_0\in \cD$ is given. Having $x_0$, we aim to create an optimization path consisting of $T$ iterates $x_t\in\mathbb{R}^d$ with $t\in\set{0,\ldots,T-1}$ to solve \eqref{P} (with all considerations discussed earlier in this section) such that all the following three criteria are satisfied

\begin{enumerate}[i)]
	\item All the iterates must be safe. This means that they must be feasible $x_t\in \cD$ for all $t$.
	\item All NFO query points must lie within the $r_0$-vicinity of $\cD$, where $r_0$ is called the exceeding margin. Note that $x\in\mathbb{R}^d$ lies within the $r_0$-vicinity of $\cD$ if there exists $\tilde{x}\in\cD$ such that $\|x-\tilde{x}\|\leqslant r_0$. 
	\item Iterates converge to an $\epsilon$-approximate First-Order Stationary Point (FOSP) $x_{\epsilon}^*$ of \eqref{P}, formally defined as
	\begin{equation}\label{FWgap}
		V_{\cD}\left(x_{\epsilon}^*,f\right):=\max_{v\in \cD} \left\langle \nabla f(x_{\epsilon}^*), x_{\epsilon}^*-v\right\rangle  \leq \epsilon.
	\end{equation}
	The function $V_{\cD}\left(x_{\epsilon}^*,f\right)$ is known as the Frank-Wolfe gap.  In particular, the Frank-Wolfe gap is zero for an FOSP. Hence, the definition of an $\epsilon$-approximate FOSP relaxes the definition of FOSP by letting the Frank-Wolfe gap be as large as $\epsilon$.
\end{enumerate}
\section{Algorithm}

In Section~\ref{PF}, we demonstrated the problem formulation and stated the required safety and optimality criteria. In this section, we propose the Reliable Frank-Wolfe (Reliable-FW) algorithm to tackle that problem while fulfilling these safety and optimality requirements. A pseudo code for Reliable-FW is outlined in Algorithm~\ref{alg}. A detailed description of Reliable-FW follows.

As Algorithm~\ref{alg} demonstrates, Reliable-FW starts at a given feasible point $x_0\in\cD$ and generates $T$ iterates denoted by $x_t$ where $t\in\set{0,\ldots,T-1}$. At iteration $t$, it makes $n_t$ NFO queries around $x_t$ and thus receives their corresponding $n_t$ measurements $y_t$. By appending these new data collectively, a cumulative list of query points $X_t\in\mathbb{R}^{N_t\times d}$ and measurements $Y_t\in\mathbb{R}^{N_t\times m}$ with $N_t=n_0+\ldots+n_t$ will be obtained at iteration $t$. This procedure is performed by calling a subroutine named {\scshape CollectData} in line 4 of Algorithm~\ref{alg} whose pseudo code is illustrated in Subroutine~\ref{datacollalg}. More specifically, upon the call at iteration $t$, {\scshape CollectData} picks $2d$ query points around $x_t$ with a distance $r_0$ away from $x_t$ along $2d$ main directions $\pm e_j, \,j\in[d]$ around $x_t$. {\scshape CollectData} then makes $n_t/2d$ measurements at each query points. Equivalently, the query points can be denoted by $x_{(1)},\ldots,x_{(n_t)}$ consisting of $2d$ distinct points each of which repeated $n_t/2d$ times. Having these $n_t$ query points, NFO returns the corresponding measurements $y_{(1)},\ldots,y_{(n_t)}$. The function {\scshape AskOracle} represents the NFO. It indeed accepts the list $X = [x_{(1)},\ldots,x_{(n_t)}]^{\top}$ and returns $Y = [y_{(1)},\ldots,y_{(n_t)}]^{\top}$. This structure of gathering data points matches the corresponding part of SFW in \cite{scl} and the only advantage here is that {\scshape CollectData} describes it more clearly by a proper indexing. Note that given the safety (feasibility) of $x_t$, this choice the NFO query points $x_{(1)},\ldots,x_{(n_t)}$ ensures that all of them lie within the $r_0$-vicinity of $\cD$ as was required by the problem formulation in Section~\ref{PF}. Hence, it suffices for the analysis to focus on the safety of the iterates.

		\begin{algorithm}[t!]
			\caption{Reliable Frank-Wolfe}
			{\bf Input} $x_{0},r_0,T,n_{t},\eta_{t}, \rho_t$ 
			
			\begin{algorithmic}[1]	
				\STATE ${X}_{-1}\,\leftarrow\, \emptyset$ , $Y_{-1}\,\leftarrow\, \emptyset$
				\vspace{.00cm}
				\FOR{ $t= 0$ {\bfseries to} $T-1$}
				\vspace{.00cm}
				\STATE  $\left({X}_t,Y_t\right)\leftarrow$ {\scshape CollectData}$\left({X}_{t-1},Y_{t-1},x_t,n_t,r_{0}\right)$
				
				\vspace{.00cm}
				\STATE  $\hat{\cD}_{t}\,\leftarrow$ {\scshape EstimatePolytope}$\left({X}_t,Y_t\right)$
				\vspace{.00cm}
				\IF{$t = 0$}
				\vspace{.00cm}
				\STATE $g_t \leftarrow  G_t(x_t)$
				\vspace{.00cm}
				\ELSE
				\vspace{.00cm}
				\label{r}\STATE $g_t \leftarrow  G_t(x_t) + (1-\rho_t)[g_{t-1} -  G_t(x_{t-1})]$ 
				\vspace{.00cm}
				\ENDIF
				\vspace{.00cm}
				\STATE $\hat{v}_{t}\,\leftarrow \arg \min _{v \in \hat{\cD}_{t}}\left\langle g_{t}, v \right\rangle$
				\vspace{.00cm}
				\STATE $x_{t+1}\,\leftarrow\,x_{t}+\eta_{t}\left(\hat{v}_{t}-x_{t}\right)$
				\vspace{.00cm}
				\ENDFOR
				\vspace{.00cm}
				\STATE $t_0\,\leftarrow$ {\scshape DrawUniform}$\left(0,T-1\right)$
				\vspace{.00cm}
			\end{algorithmic}
			{\bfseries Return} $x_{t_0}$
			\label{alg}
		\end{algorithm}
		\makeatletter
		\renewcommand{\ALG@name}{Subroutine}
		\makeatother
		
		\begin{algorithm}[t]
			\caption{{\scshape CollectData}}
			{\bf Input} ${X}_{t-1},Y_{t-1}, x_t, n_t, r_{0}$
			
			\begin{algorithmic}[t]	
				\FOR{$k=1$ {\bfseries to} ${n_t}/{2 d}$}
				\vspace{.00cm}
				\FOR{$j=1$ {\bfseries to} $2 d$}
				\vspace{.00cm}
				\STATE $\ell\,\leftarrow\,2 d\, (k-1)+j$
				\vspace{.0cm}
				\IF{$j \leqslant d$}
				\vspace{.0cm}
				\STATE $x_{(\ell)}\,\leftarrow\,x_t+{r_0}\,e_{j}$
				\vspace{.00cm}
				\ELSE
				\vspace{-1mm}
				\STATE $x_{(\ell)}\,\leftarrow\,x_t-{r_0}\,e_{j-d}$
				\vspace{.0cm}
				\ENDIF
				\vspace{.0cm}
				\ENDFOR
				\vspace{.0cm}
				\ENDFOR
				\vspace{.0cm}
				\STATE  $X\,\leftarrow[x_{(1)},\ldots,x_{(n_t)}]^{\top}$
				\vspace{.0cm}
				\STATE  ${Y}\,\leftarrow$ {\scshape AskOracle}$({X})$
				\vspace{.0cm}
				\STATE ${X}_{t}\,\leftarrow$ $[X_{t-1}^\top, {X}^\top]^\top$
				\vspace{.0cm}
				\STATE  ${Y}_{t}\,\leftarrow$ $[Y_{t-1}^\top, {Y}^\top]^\top$
				\vspace{.00cm}
			\end{algorithmic}
			{\bf Return} ${X}_{t},{Y}_{t}$
			\label{datacollalg}
		\end{algorithm}
Based on the earlier described cumulative list of NFO query points and their corresponding measurements at iteration $t$, namely $X_t$ and $Y_t$, Reliable-FW computes the least squares estimate of $\cD$; see the Appendix for a detailed description. This procedure is denoted by the function {\scshape EstimatePolytope} in Algorithm~\ref{alg}. The least squares method is chosen due to its statistical and algebraic advantages. Having $X_t$ and $Y_t$, {\scshape EstimatePolytope} computes $\hat{\beta}_t = (\bar{X}_t^{\top}\bar{X}_t)^{-1}\bar{X}_t^{\top}Y_t$ where $\bar{X}_t:=[{X}_t,\, {-\mathbb{1}_{N_t}}]$. As a result, by letting $[\hat{A}_t, \hat{b}_t] := \hat{\beta}_t^{\top}$, the least squares estimate of $\cD$ at iteration $t$ can be written as $\hat{\cD}_t:=\{x\in\mathbb{R}^d:\hat{A}_tx-\hat{b}_t\leqslant\mathbb{0}_m\}$. It is worth mentioning that exploiting a least squares method on a collective set of data points to estimate $\cD$ was also used by \cite{scl}. 

As the next step, Reliable-FW seeks to obtain a high-quality estimation of $\nabla f(x_t)$. To this aim, it uses the Stochastic Recursive Momentum (STORM) strategy from \cite{cutkosky2019momentum}. As stated in Section~\ref{PF}, we have access to a stochastic gradient estimator $G(x,\xi)$. At iteration $t$, STORM calls a realization of $G(x,\xi)$, that is, $G_t(x)=G(x,\xi_t)$ (where $\xi_t$ is a realization of $\xi$) and then evaluates $G_t(x)$ at points $x=x_t$ and $x=x_{t-1}$. Given a predefined step size parameter $\rho_t$, STORM then computes $g_t$ as a reduced-variance unbiased estimation of $\nabla f (x_t)$ by using the step shown in line 9 of Algorithm~\ref{alg}.

Followed by acquiring an estimation $g_t$ of the gradient at point $x_t$ and an estimation $\hat{\cD}_t$ of the feasible set, Reliable-FW then estimates the conditional gradient $\hat{v}_{t}$ through a linear optimization over the estimated polytope $\hat{\cD}_t$ in line 10 of Algorithm~\ref{alg}. This step, called the Frank-Wolfe step, is generally known for its advantage of avoiding projections in constrained optimization and thus reducing the complexity while preserving the quality of outcomes. Having $\hat{v}_{t}$, the next iterate is then generated by taking a step towards $\hat{v}_{t}$ with step size $\eta_t$. Finally, via {\scshape DrawUniform}, Reliable-FW draws one of iterates uniformly at random and returns it as the output. The total number of gradient oracle calls (SFO complexity) of Algorithm~\ref{alg} is $2T$ and its total number of feasibility oracle calls (NFO complexity) is $N_{T-1}=n_0+\ldots+n_{T-1}$. Note that to run Algorithm~\ref{alg}, we only need to specify the inputs $x_{0}$, $r_0$, $T$, $n_{t}$, $\eta_{t}$, and $\rho_t$. 
\section{Analysis}\label{A}
In this section, we first present an informal step-by-step description of our analysis and then state our main theoretical results (exact details are given in the supplementary material).
Our analysis consists of two components: safety and convergence. The goal of the safety part is to obtain a sufficient condition on the number $n_t$ of measurements at each iteration such that the feasibility of all the iterates is certified with high probability. Later, the convergence part seeks to find a sufficient condition on the number of iterations $T$ such that the final solution satisfies \eqref{FWgap} as desired. This also implies some more sufficient conditions on $n_t$. Combining these two major parts, we find final expressions for $n_t$ and $T$ such that the requirements  are fulfilled. The values $n_t$ and $T$ will be given in terms of the optimization stepsize $\eta_t$. It turns out that there is a trade-off between $T$ and $n_t$, as the stepsize $\eta_t$ varies. Since $n_t$ and $T$ directly relate to the NFO and SFO complexities, respectively, the value of $\eta_t$ also tunes the trade-off between the NFO and SFO complexities. To obtain the best SFO complexity, the non-convex nature of the problem forces us to choose a stepsize $\eta_t$ larger than $\mathcal{O}(1/\sqrt{t})$. Applying this large stepsize in the analysis of \cite{scl} leads to large NFO complexities. To resolve this issue, we introduce a novel geometric shrinkage idea to tighten the safety analysis such that larger stepsizes can be tolerated on the safety part and hence the NFO complexity is minimized. Before describing this idea, we introduce some preliminary concepts. 

We next define the notion of geometric shrinkage that is a key tool for our safety analysis. 
\begin{definition}[Geometric Shrinkage]\label{d2}
	Let $\tau\in\mathbb{R}$ and $\cD = \{x \in \mathbb{R}^{d}: \,\,A x-b\leqslant \mathbb{0}_m\}$. Then the $\tau$-shrunk version of ${\cD}$ is defined as $
	{\cD}_{\tau}:=\left\{x \in \mathbb{R}^{d}: \ A x-b+\tau\mathbb{1}_m \leqslant 0\right\}$.
\end{definition}

\begin{wrapfigure}{r}{0.4\textwidth}
	\centering
	\vspace{-2mm}
	\includegraphics[width=0.95\linewidth]{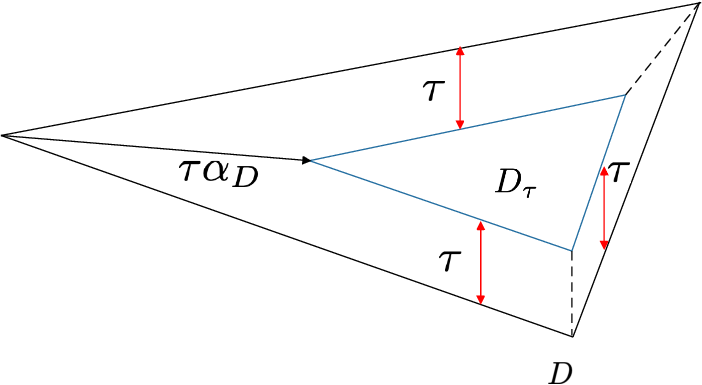}
	\vspace{-2mm}
	\caption{The geometrical interpretation of shrunk polytope ${\cD}_\tau$ and $\alpha_{\cD}$.}
	\vspace{-2mm}
	\label{alpha_shr}
\end{wrapfigure}

Denote $A=[a^1,\ldots,a^m]^{\top}\in\mathbb{R}^{m\times d}$ and $b=[b^1,\ldots,b^m]^{\top}\in\mathbb{R}^{d}$. An active point $v$ of the polytope $\cD$ is defined as a point that lies the boundary of $d$ linearly independent constraints. Note that $v$ lies on the boundary of the $i$-th constraint if $\langle a^i,v\rangle-b^i=0$. For an active point $v$, let $B_{v}=\{i_1,\ldots,i_d\}$ denote the set of $d$ constraints on which $v$ lies. Then define $A^{B_v}=[a^{i_1},\ldots,a^{i_d}]^{\top}$ and note that $A$ is invertible. Finally, letting $\operatorname{Act}(\cD)$ be the set of all active points of $\cD$ allows us to define $\rho_{\min}(\cD):=\min\{\rho_{\min}(A^{B_v}): v\in \operatorname{Act}(\cD)\}$.
Further, for $S \subseteq \mathbb{R}^{d}$, recall that $\pi_{S}(\cdot)$ denotes the projection on $S$, i.e., $\pi_{S}(x):=\arg\min_{v\in S}\|x-v\|$, for $x\in \mathbb{R}^{d}$. Next, we bound the distance between a point of $\cD$ and its projection on $\cD_\tau$.
\begin{lemma}\label{alpha00}
	Consider a positive constant $\tau >0$ and a polytope $\cD$ such that $\cD_{\tau}\neq \emptyset$. Then, for any $x \in {\cD}$, we have $\|x-\pi_{{\cD}_{\tau}} (x)\| \leqslant \alpha_{{\cD}} \tau$, where $\alpha_{{\cD}}={\sqrt{d}}/{\rho_{\min}({\cD})}$.
	Consequently, $\forall x\in\mathbb{R}^d:\, \|x-\pi_{\cD_{\tau}} (x)\|  \leqslant \|x-\pi_{\cD} (x)\| +\alpha_\cD\tau$. 
\end{lemma}

\subsection{Safety Analysis}
The general structure of this part is inspired by \cite{proof1,proof2} which was later used in \cite{scl}. However, the key component that enables us to achieve significantly improved results, is the geometric shrinkage idea. We start by defining the confidence ellipsoid.
Recall that at iteration $t$, Algorithm~\ref{alg} estimates the feasible set $\cD=\{x\in\mathbb{R}^d: Ax-b\leqslant\mathbb{0}_m\}$ by $\hat{\cD}_t:=\{x\in\mathbb{R}^d:\hat{A}_tx-\hat{b}_t\leqslant\mathbb{0}_m\}$ in which $\hat{\beta}_t = [\hat{A}_t, \hat{b}_t]^{\top}$ is the least squares estimate of  $\beta=[A,b]^{\top}$ based on the collective NFO queries. Using high-dimensional Chebyshev-type inequalities for this estimation, a confidence ellipsoid $\mathcal{E}_{t}({\zeta})$ can be constructed such that $\mathbb{P}\{\beta \in \mathcal{E}_{t}({\zeta})\} \geqslant 1-\zeta$. Considering the decomposition $\mathcal{E}_{t}({\zeta})=\prod_{i=1}^{m} \mathcal{E}_{t}^{i}({\zeta}/{m})$, each marginal ellipsoid $\mathcal{E}_{t}^{i}({{\zeta}}/{m})$ corresponds to a confidence set for $\beta^i$ (the $i$-th column of $\beta$) centered at $\hat{\beta}_t^i$  (the $i$-th column of $\hat{\beta}_t$). For our sub-Gaussian noise, $\mathcal{E}_{t}^{i}({{\zeta}}/{m})$ can be written as (see \cite{Dani,scl}),
\begin{align}\label{ellipsoid}
	\mathcal{E}_{t}^{i}\left({{\zeta}}/{m}\right) =\set{z\in\mathbb{R}^{(d+1)\times m}:\quad (\hat{\beta}_{t}^{i}-z)^{\top} \Sigma_{t}^{-1}(\hat{\beta}_{t}^{i}-z) \leq \psi^{-1}(\zeta)^{2}},
\end{align}
where $\Sigma_{t}$ is an estimation of $\operatorname{Cov}(\hat{\beta}_{t}^i)$ and $\psi^{-1}(\zeta)=\tilde{\mathcal{O}}(\sqrt{d}\log(1/\zeta))$. 
Note that each point $\beta_0\in\mathcal{E}_{t}(\zeta)$ represents an estimation of $\cD$. The safety set $S_{t}(\zeta)$ is then defined to be the intersection of all these estimated polytopes, that is, 
\begin{align}\label{safetydef}
S_{t}(\zeta)&:=\bigcap_{\beta_0 \in \mathcal{E}_{t}{(\zeta)}}\left\{x \in  \mathbb{R}^{d}:\quad \beta_0^{\top}\left[\begin{array}{c}{x} \\ {-1}\end{array}\right] \leqslant 0\right\}.
\end{align}
As a result, if $\beta \in \mathcal{E}_{t}(\zeta)$, then $S_{t}(\zeta) \subseteq {\cD}$ by definition. This means that $\beta \in \mathcal{E}_{t}(\zeta)$ and $x \in S_{t}(\zeta)$ together result in $x\in {\cD}$. Hence, feasibility of $x$ can be translated to lying in $S_t(\zeta)$. More precisely, by defining the probabilistic events $\mathcal{J}: \beta \in \bigcap_{t=0}^{T-1} \mathcal{E}_{t}(\zeta)$ and $\mathcal{R}: x_{t} \in S_{t}(\zeta), \forall t\in\{0,\ldots,T-1\}$ and letting $\zeta=\delta/T$, where $\delta\in[0,1]$ is a given confidence parameter, it can simply be shown that
\begin{equation}\label{Safecond}
 \mathbb{P}\{\mathcal{R}\mid \mathcal{J}\}=1\quad \Longrightarrow\quad \mathbb{P}\{x_0,\ldots,x_{T-1} \in \cD\} \geqslant 1-\delta.
 \end{equation}
From \eqref{Safecond} it follows that to guarantee safety, it suffices to ensure $\mathbb{P}\{\mathcal{R}\mid \mathcal{J}\}=1$. To this aim, an algebraic form of $S_t(\zeta)$ is obtained, that is,
	\begin{align}\label{Safeset}
	S_t(\zeta) = \set{x\in\mathbb{R}^d:\,\,\kappa(\zeta)^2\left(\frac{1}{N_{t}}+\left(x-\bar{x}_{t}\right)^{\top} Q_{t}\left(x-\bar{x}_{t}\right)\right) \leqslant \min _{i\in[m]} \hat{\epsilon}_{t}^{i}(x)^{2}},
	\end{align}
where $N_t=n_0+\ldots+n_t$ is the number of NFO queries up to $t$, $\hat{\epsilon}_{t}^{i}(x):=\hat{b}_{t}^{i}-\left\langle\hat{a}_{t}^{i}, x\right\rangle$ is the residual term for the $i$-th constraint. Moreover, $Q_t^{-1}:= \sum_{i=1}^{N_{t}}(x_{(i)}-\bar{x}_{t})(x_{(i)}-\bar{x}_{t})^{\top}$, where $x_{(1)},\ldots,x_{(N_t)}$ denotes all the NFO query points up to $t$ and $\bar{x}_{t}$ is their average, and $\kappa(\zeta)=\tilde{\OM}(\sqrt{\log(1/\zeta)})$.

The analysis then proceeds as follows: Given $\mathcal{J}$, to ensure $\mathbb{P}\{\mathcal{R}\mid \mathcal{J}\}=1$, we want $x_{0},\ldots,x_{T-1} \in S_{t}(\zeta)$ for $\zeta=\delta/m$. To this aim, $x_{0},\ldots,x_{T-1}$ must satisfy \eqref{Safeset}. To ensure that, we first upper-bound the left-hand side of \eqref{Safeset} by ${\kappa(\zeta)^{2}}(1+{d\Lambda^2}/{ r_0^{2}})/{N_t}$ and then lower-bound the right-hand side of \eqref{Safeset} by $h^2(\{\eta_t\},\{N_t\})$, following the notation $\{N_t\}=(N_0,\ldots,N_t)$. Finally, it suffices to determine $\{\eta_t\}$ and  $\{N_t\}$ such that ${\kappa(\zeta)^{2}}(1+{d\Lambda^2}/{ r_0^{2}})/{N_t}\leqslant h^2(\{\eta_t\},\{N_t\})$. 

We obtain improved sample complexity results because we find a sharp lower bound on $\min _{i\in[m]} \hat{\epsilon}_{t}^{i}(x_t)$, i.e., $h(\{\eta_t\},\{N_t\})$ which leads to an optimal value for $T$ (SFO complexity) while we keep $\set{N_t}$ (NFO complexity) decent. This is where the {\bf geometric shrinkage technique} plays its key role. To lower-bound the residual of $x$ over $\hat{\cD}_t$, i.e., $\hat{\epsilon}_{t}^{i}(x)$, we first relate it to its residual over $\cD$, that is, we show $\hat{\epsilon}_{t}^{i}(x_t) \geqslant  \epsilon^{i}(x_t)-\tilde{\OM}(\sqrt{d})/{\sqrt{N_t}}$, where ${\epsilon}^{i}(x_t):={b}^{i}-\left\langle{a}^{i}, x_t\right\rangle$ and thus, as the next step, we seek to lower-bound ${\epsilon}^{i}(x_t)$. Consider the Frank-Wolfe direction $\hat{v}_{t-1}$ and the optimization step $x_t=x_{t-1}+\eta_{t-1}(\hat{v}_{t-1}-x_{t-1})$ from Algorithm~\ref{alg}. Now by replacing $x_t$ from this expression in ${\epsilon}^{i}(x_t)$, we can write  ${\epsilon}^{i}(x_t)=(1-\eta_{t-1})(b^{i}-\langle a^{i}, x_{t-1}\rangle)+\eta_{t-1}(b^{i}-\langle a^{i}, \hat{v}_{t-1}\rangle)$. As the next step, add and subtract from $\hat{v}_{t-1}$, its projection on $\cD_{\tau}$ to get
\begin{align}\label{term}
\epsilon^{i}(x_{t})=(1-\eta_{t-1}) \epsilon^{i}(x_{t-1})+\eta_{t-1}(b^{i}-\langle a^{i}, \pi_{\cD_{\tau}}(\hat{v}_{t-1}) \rangle)-\eta_{t-1}\langle a^{i}, \hat{v}_{t-1}-\pi_{\cD_{\tau}}(\hat{v}_{t-1}) \rangle.
\end{align}
To lower-bound this expression, first note that $b^{i}-\langle a^{i}, \pi_{\cD_{\tau}}(\hat{v}_{t-1}) \rangle\geqslant \tau$ by definition. In addition, 
\begin{align*}
\left\|\hat{v}_{t-1}-\pi_{\cD_{\tau}} \left(\hat{v}_{t-1}\right)\right\| \leqslant \left\|\hat{v}_{t-1}-\pi_{\cD} \left(\hat{v}_{t-1}\right)\right\| +\alpha_{\cD}\tau\leqslant \frac{C_1}{\sqrt{N_{t-1}}}+\alpha_{\cD}\tau,
\end{align*}
where the first inequality follows from Lemma~\ref{alpha00} and since $\hat{v}_{t}\in\hat\cD_t$, the second inequality follows from the fact that we can bound the distance between any point in $\hat\cD_t$ and its projection on $\cD$ with $C_1/\sqrt{N_t}$, where $C_1=\tilde{\OM}(d)$. Lower-bounding the second and third terms on the right-hand side of \eqref{term} using geometric shrinkage technique, results in a recursive inequality for ${\epsilon}^{i}(x_t)$ in terms of ${\epsilon}^{i}(x_{t-1})$. By applying this recursion repeatedly until $t=0$, we obtain a formula for $h(\{\eta_t\},\{N_t\})$.

\subsection{Convergence Analysis} 
Next, we aim to find a sufficient condition on $T$ in terms of $\{\eta_t\}$ for a large enough $\{N_t\}$ such that \eqref{FWgap} holds.  Having defined $v_{t}:=\arg \min_{v\in \cD}\, \langle \nabla f(x_t), v\rangle$ and recalling $\hat{v}_{t}:=\arg \min_{v\in \hat{\cD}_t}\, \langle g_t, v\rangle$ from Algorithm~\ref{alg}, we start from the following inequality:
\begin{align}
\label{vd} V_{\cD}\left(x_{t}, f\right)=\left\langle\nabla f\left(x_{t}\right), x_{t}-v_{t}\right\rangle\leqslant\left\|\nabla f\left(x_{t}\right)-g_{t}\right\| \Lambda+\left\langle g_{t}, x_{t}-v_{t}\right\rangle.
\end{align}
To achieve \eqref{FWgap}, we need to upper-bound $\langle g_{t}, x_{t}-v_{t}\rangle$ on the right-hand side of \eqref{vd}. Since $v_t$ is not known, we instead upper-bound $\langle \nabla f(x_t), x_{t}-\hat{v}_{t}\rangle$ based on $L$-Lipschitz continuity of $\nabla f$ while controlling the difference between $\langle g_{t}, x_{t}-v_{t}\rangle$ and $\langle \nabla f(x_t), x_{t}-\hat{v}_{t}\rangle$ using the fact that the distance between a point of $\cD$ and its projection on $\hat{\cD}_t$ can be bounded by $\tilde{\OM}(d)/\sqrt{N_t}$. This results in
\begin{align*}
V_{\cD}(x_{t}, f)\leqslant\frac{f(x_{t})-f(x_{t+1})}{\eta_t} +\|\nabla f(x_{t})-g_{t}\|\left(\frac{C_1}{\sqrt{N_{t}}}+2\Lambda\right)
+\|g_t\|\frac{C_1}{\sqrt{N_{t}}}+\frac{L \eta_{t}}{2}\left(\frac{C_1}{\sqrt{N_{t}}}+\Lambda\right)^2.
\end{align*} 
The first term on the right-hand side is related to the optimization step, the second term is related to the gradient estimation error, and the third and forth terms can be controlled by $N_t$ and $\eta_t$. To fulfil \eqref{FWgap}, we require a large enough $T$ and small enough $\eta_t$ to make these four terms sufficiently small.

\subsection{Main Result} Next, we provide our main theoretical results (exact constants are in the Appendix). 
\begin{theorem}\label{maintheorem2}
	Consider Problem~\eqref{P} under Assumptions~\ref{f}-\ref{NFO}. Suppose $\epsilon >0$ is given and let $\eta_t=\rho_t=(t+2)^{-2/3}$, $n_t=\tilde{\mathcal{O}}(d^2(t+1)^{1/3}\log(1/\delta))$, and 
$	T= \tilde{\mathcal{O}}(\epsilon^{-3})$. Then, with probability at least $1-\delta$, in Algorithm~\ref{alg}, all the iterates are safe, i.e., $x_0,\ldots,x_{T-1}\in\cD$, all the query points lie within $r_0$-vicinity of $\cD$, and $
		\mathbb{E}_{[t_0]}\mathbb{E}_{0:t-1}[V_{\cD}\left(x_{t}, f\right)] \leqslant \epsilon$,
	where the expectation $\mathbb{E}_{0:t-1}$ is taken over all the randomness in Algorithm~\ref{alg} through iterations $0$ to $t-1$ and $\mathbb{E}_{[t_0]}$ is taken over choosing $t_0$ from $\set{0,\ldots,T-1}$ uniformly at random.
\end{theorem}

Based on Theorem~\ref{maintheorem2}, Reliable-FW guarantees all iterates are safe at least with probability $1-\delta$ and finds an $\epsilon$-FOSP in expectation with a total of $\tilde{\OM}(d^2\log(1/\delta)/\epsilon^4)$ NFO and $\tilde{\OM}(1/\epsilon^3)$ SFO queries. Next we study the case that exact gradient $\nabla f$ is available.  
\begin{theorem}\label{maindet2}
	Consider Problem~\eqref{P} under Assumptions~\ref{f}, \ref{D}, and \ref{NFO} and assume we have access to $\nabla f$. Suppose $\epsilon >0$ is given and let $\eta_t=(t+2)^{-1/2}$, $\rho_t=1$, $n_t=\tilde{\mathcal{O}}(d^2\log(1/\delta))$, 
	$	T= \mathcal{O}(\epsilon^{-2})$. Then, with probability at least $1-\delta$, in Algorithm~\ref{alg}, all the iterates are safe, i.e., $x_0,\ldots,x_{T-1}\in\cD$, all the query points lie within $r_0$-vicinity of $\cD$, and 
	$\mathbb{E}_{[t_0]}\mathbb{E}_{0:t-1}[V_{\cD}\left(x_{t}, f\right)] \leqslant \epsilon$,
	where $\mathbb{E}_{0:t-1}$ and $\mathbb{E}_{[t_0]}$ are specified in Theorem~\ref{maintheorem2}.
\end{theorem} 
Theorem~\ref{maindet2} shows Reliable-FW (with gradients) guarantees all iterates are safe at least with probability $1-\delta$ and finds an $\epsilon$-FOSP in expectation with a total of $\tilde{\OM}(d^2\log(1/\delta)/\epsilon^2)$ NFO and $\OM(1/\epsilon^2)$ SFO queries. Next, we consider a setting in which the objective function is convex with a stochastic gradient estimator.
\begin{theorem}\label{convexstoch3}
	Consider Problem~\eqref{P} under Assumptions~\ref{f}-\ref{NFO} and assume $f$ is convex. Suppose $\epsilon >0$ is given and let $\eta_t=\rho_t=(t+2)^{-1}$, $n_t=\tilde{\mathcal{O}}(d^2(t+1)\log(1/\delta))$, and $T= \tilde{\mathcal{O}}(\epsilon^{-2})$.
	Then, with probability at least $1-\delta$, in Algorithm~\ref{alg}, all the iterates are safe, i.e., $x_0,\ldots,x_{T-1}\in\cD$, all the query points lie within $r_0$-vicinity of $\cD$, and  $f(x_{T-1})-f(x^*)\leqslant \epsilon$.
\end{theorem} 
Based on Theorem~\ref{convexstoch3}, when $f$ is convex, Reliable-FW guarantees all iterates are safe at least with probability $1-\delta$ and finds an $\epsilon$-suboptimal solution with a total of $\tilde{\OM}(d^2\log(1/\delta)/\epsilon^4)$ NFO and $\tilde{\OM}(1/\epsilon^2)$ SFO queries. Finally, we study the case in which $f$ is convex and $\nabla f$ is known.
 \begin{theorem}\label{convext}
	Consider Problem~\eqref{P} under Assumptions~\ref{f}, \ref{D}, and \ref{NFO}. Moreover, assume $f$ is convex and we have access to $\nabla f$. Suppose $\epsilon >0$ is given and let $\eta_t=2(t+2)^{-1}$, $\rho_t=1$, $n_t=\tilde{\mathcal{O}}(d^2(t+1)\log(1/\delta))$, and $T= \mathcal{O}(\epsilon^{-1})$.
	Then, with probability at least $1-\delta$, in Algorithm~\ref{alg}, all the iterates are safe, i.e., $x_0,\ldots,x_{T-1}\in\cD$, all the query points lie within $r_0$-vicinity of $\cD$, and  $f(x_{T-1})-f(x^*)\leqslant \epsilon$.
\end{theorem} 
According to Theorem~\ref{convext}, when $f$ is convex, Reliable-FW (with gradients) guarantees all iterates are safe at least with probability $1-\delta$ and finds an $\epsilon$-suboptimal solution with a total of $\tilde{\OM}(d^2\log(1/\delta)/\epsilon^2)$ NFO and $\OM(1/\epsilon)$ SFO queries. 
\section{Experiments}\label{exp}

\begin{figure}[t]
	\centering
	\begin{minipage}{0.32\textwidth}
		\includegraphics[width=1\textwidth]{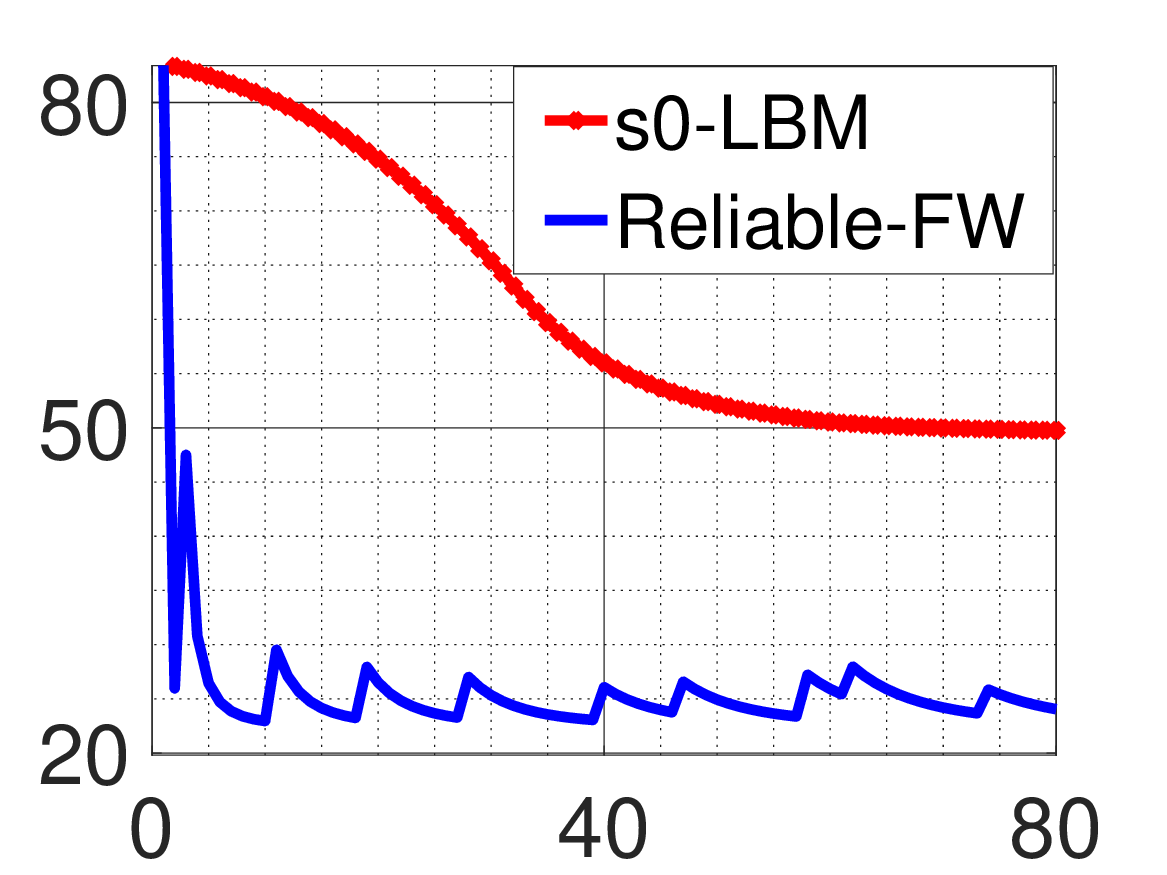}
		\caption*{(a)}
	\end{minipage}
	\begin{minipage}{0.32\textwidth}
		\includegraphics[width=1\textwidth]{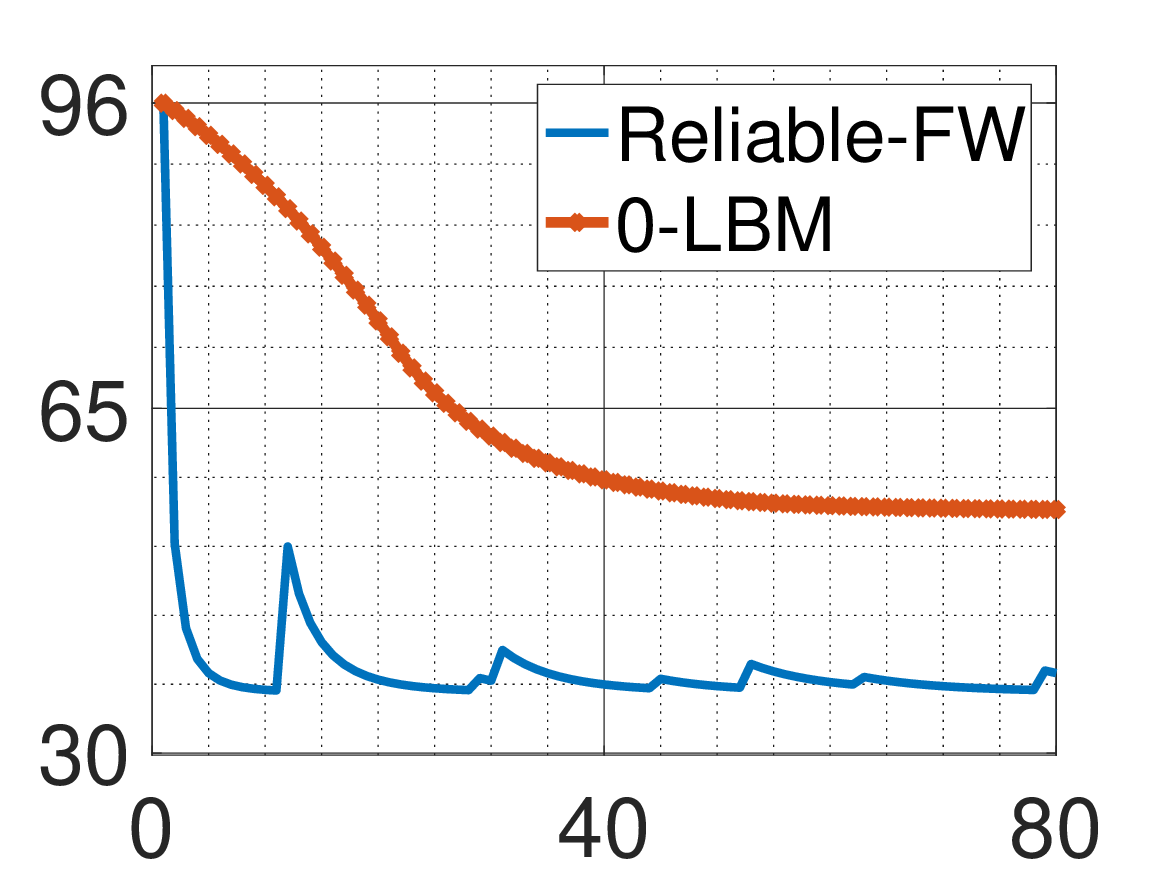}
		\caption*{(b)}
	\end{minipage}
	\begin{minipage}{0.32\textwidth}
		\includegraphics[width=1\textwidth]{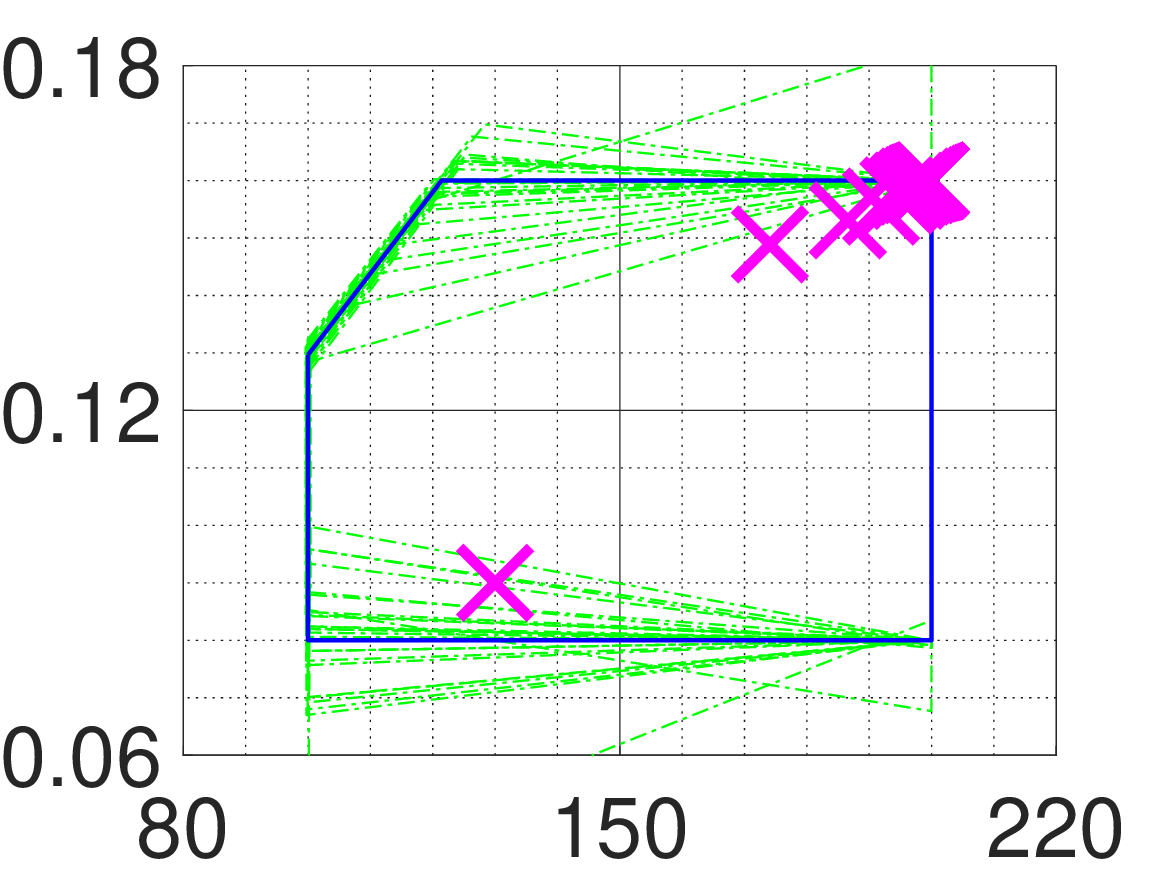}
		\caption*{(c)}
	\end{minipage}
	\caption{Safety, convergence, and comparison of Reliable-FW versus s0-LBM ans 0-LBM.}
	\label{t}
\end{figure}

We consider a similar non-convex optimization problem  used in \cite{LB} which models a cutting machine from \cite{num}. Consider the following problem
\begin{align*}
	\min_{(x,y)\in\mathbb{R}^2}\quad \frac{22}{xy}\left(50+\frac{40}{h_1(x,y)}\right),\qquad 
	\text{s.t.} \quad h_2(x,y)\leqslant 0,\,\,100\leqslant x\leqslant 200,\,\,0.08\leqslant y\leqslant 0.16,
\end{align*}
where, $h_1(x,y) = 127.5365-0.84629x-144.21 y+0.001703 x^{2}+0.3656 xy$ and $h_2(x,y)=0.0844-0.010035x+7.0877y$. We set $\sigma_0=0.001$ and $r_0 = 0.01$. For 0-LBM and s0-LBM, we used $\eta_t = 0.1, L = 7, M = 5$, and $\delta = 0.01$. We used $x_0 = [150,\,0.09]^{\top}$ and $\sigma = 0.01$ for Figure~\ref{t} (a) and (b) and $x_0 = [130, 0.09]^{\top}$ and $\sigma=0.06$ for Figure~\ref{t} (c).

Figure~\ref{t} (a) shows function-value in terms of NFO calls for Reliable-FW versus s0-LBM. Note that s0-LBM uses zeroth-order objective function queries opposite to first-order queries in Reliable-FW. Hence, for a fair comparison and to ensure that the performance difference is not affected by the choice of gradient estimator, we apply the (stochastic) gradient estimator of s0L-BM to Reliable-FW and compare only those function calls that have not been used for a gradient approximation.\\
In Figure~\ref{t} (b), function value with respect to objective function calls for Reliable-FW vs 0-LBM have been plotted. Similar to the previous experiment, we use the same gradient estimation for both methods. 
Note that since the Frank-Wolfe step in Reliable-FW is linear optimization of an stochastic vector over a stochastic set, our update is inherently random which result in some oscillation of function-value specially after reaching a local optimum which can also be seen in Figure~\ref{t} (a) and (b).
Figure~\ref{t} (c) illustrates the convergence and safety (feasibility) of iterates of Reliable-FW (marked with "$\times$"), where the original feasible set is depicted in solid and its corresponding estimations during the algorithm are shown in dashed lines.

\section*{Conclusion}
We proposed Reliable-FW to tackle non-convex learning under uncertain constraints while ensuring the safety of iterates. We incorporated a novel analysis, specifically, a mathematical technique called the geometric shrinkage, to obtain outperforming feasibility and gradient oracle complexities. Finally, we numerically verified the performance of Reliable-FW.

\section*{Acknowledgement}
The research of Mohammad Fereydounian is supported by NSF awards CIF-CRII-1755707 and CIF-1910056.  Zebang Shen is supported by NSF award CPS-1837253. Amin Karbasi is partially supported by NSF (IIS-1845032), ONR (N00014-19-1-2406), and AFOSR (FA9550-18-1-0160). Hamed Hassani is supported by the NSF HDR TRIPODS award 1934876, NSF award CPS-1837253, NSF award CIF-1910056, and NSF CAREER award CIF 1943064, and  Air Force Office of Scientific Research Young Investigator Program (AFOSR-YIP) under award FA9550-20-1-0111.

\newpage
\appendixwithtoc
\section{Proof of Lemma~\ref{alpha00}}\label{GS}
Recall from Section~\ref{A} the definition of an active point $v$ and the related concepts including its set of active constraints $B_v$, $A^{B_v}$, the set of active points $\operatorname{Act}(\cD)$, and $\rho_{m i n}({\cD})$. Before we start the proof of Lemma~\ref{alpha00}, we mention some observations regarding these concepts. 
Letting $b^{B_v}=[b^{i_1},\ldots,b^{i_d}]^{\top}\in\mathbb{R}^d$, we have $A^{B_v}v=b^{B_v}$. Based on the fact that $A^{B_v} \in \mathbb{R}^{d \times d}$ is invertible, we can write $v=(A^{B_v})^{-1} b^{B_v}$. Moreover, note that an active point of ${\cD}$ may or may not belong to ${\cD}$. Indeed, vertices of ${\cD}$ are those active points of ${\cD}$ which belong to ${\cD}$ and hence $\operatorname{Vert}({\cD})\subseteq \operatorname{Act}({\cD})$. Finally, it is important to note that the compactness of ${\cD}$ results in $\rho_{\min}({\cD})>0$. The proof of Lemma~\ref{alpha00} can now be presented as follows.

\begin{proof}
	Since the faces of $\cD_{\tau}$ are parallel to the faces of $\cD$, the maximum distance between a point $x\in \cD$ and its projection $\pi_{\cD_{\tau}} (x)\in \cD_{\tau}$ obviously occurs at some vertex $v\in \operatorname{Vert}(\cD)$. In other words, we have $\arg \max_{x\in \cD} \left\|x-\pi_{\cD_{\tau}} (x)\right\| \in \operatorname{Vert}(\cD)$.  Note that $B_v$ consists of $d$ independent hyper-plains that intersect at $v$. Now consider their shifted versions in $\cD_{\tau}$. These shifted hyper-plains are also independent. Denote their intersection by $u$. Note that  $B_v=B_u$. Therefore, for any $v\in \operatorname{Vert}(\cD)$, it follows that
	\begin{align*}
	&A^{B_v}v-b^{B_v} = 0 \quad \Rightarrow \quad v = \left(A^{B_v}\right)^{-1}b^{B_v},\\
	&A^{B_u}u-b^{B_u} +\tau\mathbb{1}_{\abs{B_u}}= 0 \quad \Rightarrow \quad u = \left(A^{B_u}\right)^{-1}\left(b^{B_u}-\tau\mathbb{1}_{\abs{B_u}}\right).
	\end{align*}
	By subtracting two previous lines and noting that $B_v=B_u$, for any $v\in\operatorname{Vert}(\cD)$, we conclude that
	\begin{align*}
	\left\|v-\pi_{\cD_{\tau}} (v)\right\| &\leqslant \left\|v-u\right\|= \tau \left\|\left(A^{B_v}\right)^{-1}\mathbb{1}_{\abs{B_v}}\right\| \leqslant \tau \,\sqrt{d} \left\|\left(A^{B_v}\right)^{-1}\right\|= \frac{\tau \,\sqrt{d}}{\rho_{\min}\left(A^{B_v}\right)}\leqslant \frac{\tau \,\sqrt{d}}{\rho_{\min}\left(\cD\right)},
	\end{align*}
	where in the last inequality, we used the fact that $\operatorname{Vert}(\cD)\subseteq \operatorname{Act}(\cD)$. Hence, for all $x\in\cD$, $\|x-\pi_{\cD_{\tau}} (x)\| \leqslant \alpha_{\cD} \tau$, where $\alpha_{\cD}=\sqrt{d}/\rho_{\min}(\cD)$.
	
	Moreover, note that $\pi_{\cD_{\tau}} (\pi_{\cD}(x))\in \cD_{\tau}$. Thus, for every $x\in\mathbb{R}^d$, we have $\|x-\pi_{\cD_{\tau}} (x)\|  \leqslant \|x-\pi_{\cD_{\tau}} (\pi_{\cD}(x))\|$ based on the definition of projection on $\cD_{\tau}$. Therefore, for every $x\in\mathbb{R}^d$,
	\begin{equation}\label{e1}
		 \|x-\pi_{\cD_{\tau}} (x)\|  \leqslant \|x-\pi_{\cD_{\tau}} (\pi_{\cD}(x))\|\leqslant \|x-\pi_{\cD} (x)\|+ \|\pi_{\cD} (x)-\pi_{\cD_{\tau}} (\pi_{\cD}(x))\|.
	\end{equation}
	Based on the earlier argument, since $\pi_{\cD} (x)\in\cD$, we have $\|\pi_{\cD} (x)-\pi_{\cD_{\tau}} (\pi_{\cD}(x))\|\leqslant \alpha_\cD\tau$. This together with \eqref{e1} results in $\|x-\pi_{\cD_{\tau}} (x)\|\leqslant \|x-\pi_{\cD} (x)\|+\alpha_\cD\tau$.
\end{proof}
\begin{remark}
	Note that $\alpha_{\cD}$ is a parameter of the polytope ${\cD}$ and does not depend on $\tau$. Figure~\ref{alpha_shr} illustrates the geometrical meaning of $\alpha_{\cD}$ for a simple polytope (here a triangle) in two dimensions. As it can be seen, the faces of ${\cD}_{\tau}$ are parallel to the faces of ${\cD}$ with a distance less than $\tau$ but not necessarily equal to $\tau$. 
\end{remark}
\section{Problem Normalization} For the ease of analysis, we consider a normalization in which $\tilde{A}={(2\alpha_{\cD} L'_A)^{-1}}A$ and $\tilde{b}={(2\alpha_{\cD} L'_A)^{-1}}b$, where $L'_{A}:=\max _{i\in[m]}\|a^{i}\|$. Based on this normalization, ${\cD}=\{x: \tilde{A} x-\tilde{b}\leqslant 0\}$. Denoting $\tilde{A}=[\tilde{a}^1,\ldots,\tilde{a}^m]^{\top}$, we now have $L_{A}:=\max _{i\in[m]}\|\tilde{a}^{i}\|={1}/{(2\alpha_{\cD})}$. For the compatibility with the NFO measurements, the variance $\sigma$ needs to be replaced with $\bar{\sigma}:=(2\alpha_{\cD} L'_A)^{-1}{\sigma}$. For the ease of notation and without loss of generality, we use $A$ and $b$ instead of $\tilde{A}$ and $\tilde{b}$. Hence, from now on, we consider 
\begin{equation}
\label{normalization} 	L_{A}=\max _{i\in[m]}\|a^{i}\|=\frac{1}{2\alpha_{\cD}},\quad \bar{\sigma}=\frac{\sigma}{2\alpha_{\cD} L'_A},
\end{equation}
and measurements can be written as $y(x)=A x-b+\theta$, where $\theta$ is a zero mean $\bar{\sigma}$-sub-Gaussian noise vector.
\section{Safety Analysis}\label{AS}
In this section, we first discuss in detail the task of {\scshape EstimatePolytope} in Algorithm~\ref{alg}, i.e., the least squares estimation of the feasible set. We then construct a confidence ellipsoid for this estimate. Based on this confidence ellipsoid, the safety set will be defined and an algebraic form for such a set will be derived. This algebraic form then helps us to acquire a sufficient condition on NFO complexity to achieve the safety of all the iterates. The general structure of this section is inspired by \cite{proof1,proof2} and was later used in \cite{scl}. Specifically, Lemma~\ref{SafetySet} and \ref{projection} and their proofs are taken from \cite{scl}. However, the key component in achieving the improved results is the geometric shrinkage idea discussed in the proof of Proposition~\ref{nk}. 

\subsection{Polytope Estimation}\label{PE}
Suppose we have $N_{t}$ measurements from the described noisy oracle at query points $x_{(1)} \ldots x_{(N_t)}$. Then for $i\in[N_t]$,
\begin{align}\label{mesurements}
y\left(x_{(i)}\right)^{\top}=x_{(i)}^{\top}A^{\top}-b^{\top}+\theta_{(i)}^{\top}.
\end{align}
Now by concatenating the quantities, we define
\begin{align*}
Y_{t}&:=\left[\begin{array}{c}{       {y\left(x_{(1)}\right)}^{\top}   } \\   {\vdots}   \\ {   {y\left(x_{(N_t)}\right)}^{\top}         }\end{array}\right] =: \left[{{y}^{1}_t} ,{\ldots}, {{y}^{m}_t}\right] \in \mathbb{R}^{N_{t}\times m}, \quad X_{t}=\left[\begin{array}{c}{x_{(1)}^{\top}} \\ {\vdots} \\ {x_{\left(N_{t}\right)}^{\top}}\end{array}\right] \in \mathbb{R}^{N_{t}\times d}.
\end{align*}
Also,
\begin{align*}
H_{t}:=\left[\begin{array}{c}{       {\theta_{(1)}}^{\top}   } \\   {\vdots}   \\ {   {\theta_{(N_t)}}^{\top}         }\end{array}\right] =: \left[{\theta}^{1}_t,{\ldots}, {{\theta}^{m}_t}\right] \in \mathbb{R}^{N_{t}\times m},
\end{align*}
together with
\begin{align*}
\beta=\left[\begin{array}{c}{A^{\top}} \\ {b^{\top}}\end{array}\right]=:  \left[{{\beta}^{1}} , {\ldots} , {{\beta}^{m}}\right]\in \mathbb{R}^{(d+1)\times m}.
\end{align*}
Concatenation of equalities in \eqref{mesurements} now can be written as
\begin{align}\label{measurements-concat}
Y_t = X_t\,A^{\top} - \mathbb{1}\,b^{\top} + H_t.
\end{align}
In order to write  \eqref{measurements-concat} column-wise, let
\begin{align*}
A=\left[\begin{array}{c}{(a^{1})^{\top}} \\ {\vdots} \\ {(a^{m})^{\top}}\end{array}\right], \quad b=\left[\begin{array}{c}{b^{1}} \\ {\vdots} \\ {b^{m}}\end{array}\right],
\end{align*}
where for $a^{i}\in\mathbb{R}^{d}$ and $b^{i}\in\mathbb{R}$ for $i\in[m]$. As a result, we get
\begin{align*}
y_t^{i}&=X_{t} \, a^{i}-b^{i}\, \mathbb{1}+\theta_t^{i}=\left[{X_{t}} , {-\mathbb{1}}\right]\left[\begin{array}{l}{a^{i}} \\ {b^{i}}\end{array}\right]+\theta^{i}_t=\bar{X}_{t}\beta^{i}+\theta^{i}_t,
\end{align*}
where $\bar{X}_{t}:=\left[{X_{t}}, {-\mathbb{1}}\right]$. Briefly, we have
\begin{align}\label{LSproblem}
y_t^{i}=\bar{X}_{t}\beta^{i}+\theta^{i}_t, \,\quad \forall i\in[m].
\end{align}
 The least squares estimate of $\beta^i$ from \eqref{LSproblem} is
\begin{align}\label{beta}
\hat{\beta}_{t}^i=\left[\bar{X}_{t}^{\top} \bar{X}_{t}\right]^{-1} \bar{X}_{t}^{\top} y_t^{i},\,\quad \forall i\in[m].
\end{align}
We use the following notation 
\begin{align*}
\hat{\beta}_t= \left[{{\hat{\beta}_t}^{1}}, {\ldots} , {{\hat{\beta}_t}^{m}}\right]=\left[\begin{array}{c}{\hat{A}_t^{\top}} \\ {\hat{b}_t^{\top}}\end{array}\right]=\left[\begin{array}{lcr}{\hat{a}_t^{1}}&\ldots&{\hat{a}_t^{m}} \\ {\hat{b}_t^{1}}&{\ldots}&{\hat{b}_t^{m}}\end{array}\right].
\end{align*}
Therefore, based on the estimation in \eqref{beta}, the estimated polytope will be
\begin{align}\label{Dt}
\hat{\cD}_t=\left\{x \in \mathbb{R}^{d}: \quad\hat{A}_tx-\hat{b}_t \leqslant 0\right\},\quad\text{ where}\quad \left[\hat{A}_t,\hat{b}_t\right]^{\top}=\hat{\beta}_t.
\end{align}

\subsection{Confidence Ellipsoids and Safety Sets} 
Using the notation of Section~\ref{PE}, the goal in this part is to construct a confidence ellipsoid $\mathcal{E}_{t}({\zeta})$ such that $\mathbb{P}\{\beta \in \mathcal{E}_{t}({\zeta})\} \geqslant 1-\zeta$. Based on \cite{Dani,scl}, such a confidence set can be written as
\begin{align*}
\mathcal{E}_{t}({\zeta}):=\prod_{i=1}^{m} \mathcal{E}_{t}^{i}\left(\frac{{\zeta}}{m}\right),
\end{align*}
where each marginal ellipsoid $\mathcal{E}_{t}^{i}({{\zeta}}/{m})$ corresponds to a confidence set for $\beta^i$ centered at $\hat{\beta}_t^i$. For our sub-Gaussian noise, these marginal ellipsoids can be written as
\begin{align*}
\mathcal{E}_{t}^{i}\left({{\zeta}}/{m}\right) =\set{z\in\mathbb{R}^{(d+1)\times m}:\quad (\hat{\beta}_{t}^{i}-z)^{\top} \Sigma_{t}^{-1}(\hat{\beta}_{t}^{i}-z) \leq \psi^{-1}(\zeta)^{2}},
\end{align*}
where $\Sigma_{t}=\bar{\sigma}^2(\bar{X}_t^{\top}\bar{X}_t)^{-1}$ is an estimation of $\operatorname{Cov}(\hat{\beta}_{t}^i)$. Moreover, if $N_t\geqslant e^{1/16}\zeta$, then under the sub-Gaussian noise, $\psi^{-1}(.)$ can be chosen as 
\begin{align}\label{si}
\psi^{-1}(\zeta)=\max \left\{\sqrt{128 d \log N_{t} \log \left(\frac{N_{t}^{2}}{\zeta}\right)}, \frac{8}{3} \log \frac{N_{t}^{2}}{\zeta}\right\}.
\end{align}
Having this construction, it follows that $\mathbb{P}\{\beta^{i} \in \mathcal{E}_{t}^{i}({\zeta}/{m})\} \geqslant 1-{\zeta}/{m}$. As a result, we have
\begin{align}
	\label{19}\mathbb{P}\left\{\beta \in \mathcal{E}_{t}(\zeta)\right\}  \geqslant 1-\zeta, \quad \forall \zeta \in[0,1].
\end{align}
Note that each point $\beta_0\in\mathcal{E}_{t}(\zeta)$ represents an estimation of the polytope  $\cD$. The safety set $S_{t}(\zeta)$ is then defined to be the intersection of all these estimated polytopes, that is, 
\begin{align*}
\vspace{-.1cm}
 S_{t}(\zeta)&:=\bigcap_{\beta_0 \in \mathcal{E}_{t}{(\zeta)}}\left\{x \in  \mathbb{R}^{d}:\quad \beta_0^{\top}\left[\begin{array}{c}{x} \\ {-1}\end{array}\right] \leqslant 0\right\}.
 \vspace{-.1cm}
\end{align*}
Note that if $\beta \in \mathcal{E}_{t}(\zeta)$, then $S_{t}(\zeta) \subseteq {\cD}$. This means that $\beta \in \mathcal{E}_{t}(\zeta)$ and $x \in S_{t}(\zeta)$ together result in $x\in {\cD}$. Hence,
\begin{equation}
\label{prob}\mathbb{P}\left\{x \in {\cD} \,\mid \, x \in S_{t}(\zeta), \,\beta \in \mathcal{E}_{t}(\zeta)\right\}=1.
\end{equation}
Based on \eqref{prob}, the safety of the iterates can be translated to lying in $S_{t}(\zeta)$. To this aim, for a given confidence parameter $\delta \in[0,1]$, let $\zeta={\delta}/{T}$ and define the following probabilistic events:
\begin{align}\label{jr}
\mathcal{J}: \,\beta \in \bigcap_{t=0}^{T-1} \mathcal{E}_{t}(\zeta),\quad\quad\quad \mathcal{R}:\, x_t \in S_{t}(\zeta),\,\,\forall t\in\set{0,\ldots,T-1}. 
\end{align}
Then assuming $\mathbb{P}\{\mathcal{R}\mid \mathcal{J}\}=1$, it follows that
\begin{align}
\nonumber\mathbb{P}\left\{x_{0},\ldots,x_{T-1}\in \cD\right\}&\geqslant\mathbb{P}\left\{x_{0},\ldots,x_{T-1}\in \cD\,\mid\, \mathcal{R},\mathcal{J}\right\}\mathbb{P}\left\{\mathcal{R}, \mathcal{J}\right\} \\
\label{jj}&\overset{\eqref{prob}}{=}\mathbb{P}\left\{\mathcal{R}, \mathcal{J}\right\}=\mathbb{P}\left\{\mathcal{J}\right\} \mathbb{P}\left\{\mathcal{R} \mid \mathcal{J}\right\}=\mathbb{P}\left\{\mathcal{J}\right\}\\
\nonumber&\geqslant 1-\sum_{t=0}^{T-1} \mathbb{P}\left\{\beta \notin \mathcal{E}_{t}(\zeta)\right\}\overset{\eqref{19}}{\geqslant} 1-T \zeta=1-\delta.
\end{align}
Hence,
\begin{equation}\label{Safecond2}
\mathbb{P}\{\mathcal{R}\mid \mathcal{J}\}=1\quad \Longrightarrow\quad \mathbb{P}\{x_0,\ldots,x_{T-1} \in \cD\} \geqslant 1-\delta.
\end{equation}

\subsection{Intermediate Lemmas}
From \eqref{Safecond2} it follows that to guarantee safety, it suffices to ensure $\mathbb{P}\{\mathcal{R}\mid \mathcal{J}\}=1$. To this aim, some intermediate steps will be taken. First, an algebraic form for $S_{t}(\zeta)$ will be computed in Lemma~\ref{SafetySet}. Next, an upper-bound for an essential component of this algebraic form will be obtained in Lemma~\ref{RtBound}. Having these, the distance between a point in $\hat{\cD}_t$ and its projection on $\cD$ will be upper-bounded in Lemma~\ref{projection}. Finally, based on the analysis provided in this section, the final safety result will be obtained in Section~\ref{FSR}. We start with the following lemma from \cite{scl}.

\begin{lemma}\label{SafetySet}
	Let $Q_t^{-1}:= \sum_{i=1}^{N_{t}}(x_{(i)}-\bar{x}_{t})(x_{(i)}-\bar{x}_{t})^{\top}$, where $x_{(1)},\ldots,x_{(N_t)}$ are all the query points of Algorithm~\ref{alg} up to iteration $t$ with $\bar{x}_{t}$ denoting their average. Further, let $\hat{\epsilon}_{t}^{i}(x):=\hat{b}_{t}^{i}-\langle\hat{a}_{t}^{i}, x\rangle$ and $\kappa=\kappa(\zeta):=\bar{ \sigma}\psi^{-1}({\zeta}/{m})$, where $\bar{\sigma}$ and $\psi^{-1}(\cdot)$ are defined in \eqref{normalization} and \eqref{si}, respectively. Then under Assumption~\ref{NFO}, we have
	\begin{align*}
	S_t(\zeta) = \set{x\in\mathbb{R}^d:\,\,\kappa^2\left(\frac{1}{N_{t}}+\left(x-\bar{x}_{t}\right)^{\top} Q_{t}\left(x-\bar{x}_{t}\right)\right) \leqslant \min _{i\in[m]} \hat{\epsilon}_{t}^{i}(x)^{2}}.
	\end{align*}
\end{lemma}
\begin{proof}
	We start with re-parametrizing $\mathcal{E}_{t}^{i}$.  Note that due to Lemma~\ref{ellipse}, we can rewrite \eqref{ellipsoid} as
	\begin{equation}\label{confiset}
	\mathcal{E}_{t}^{i}\left(\frac{\zeta}{m}\right)=\left\{\hat{\beta}_{t}^{i}-\psi^{-1}\left(\frac{\zeta}{m}\right) \Sigma_{t}^{1 / 2} u:\quad u\in\mathbb{R}^d,\,\,\|u\| \leqslant 1\right\}.
	\end{equation}
	Using \eqref{confiset}, the safety set can be rewritten as
	\begin{align*}
	S_{t}(\zeta)&=\left\{x \in \mathbb{R}^{d}:\,\, \forall i \in[m]\,\forall \beta^{i} \in \mathcal{E}_{t}^{i}\left(\frac{\zeta}{m}\right):\,\, \beta^{i \top}\left[\begin{array}{c}{x} \\ {-1}\end{array}\right] \leqslant 0 \right\}\\
	&=\left\{x \in \mathbb{R}^{d}: \,\, \forall i \in [m]\,\forall u\in\mathbb{R}^d,\,\,\|u\| \leqslant 1:\quad \left(\hat{\beta}_{t}^{i}-\psi^{-1}\left(\frac{\zeta}{m}\right) \Sigma_{t}^{1 / 2} u\right)^{\top}\left[\begin{array}{c}{x} \\ {-1}\end{array}\right] \leqslant 0 \right\}.
	\end{align*}
	Note that 
	\begin{equation*}
	\left(\hat{\beta}_{t}^{i}-\psi^{-1}\left(\frac{\zeta}{m}\right) \Sigma_{t}^{1 / 2} u\right)^{\top}\left[\begin{array}{c}{x} \\ {-1}\end{array}\right]=\hat{\beta}_{t}^{i \top}\left[\begin{array}{c}{x} \\ {-1}\end{array}\right]-\psi^{-1}\left(\frac{\zeta}{m}\right) u^{\top} \Sigma_{t}^{1 / 2}\left[\begin{array}{c}{x} \\ {-1}\end{array}\right].
	\end{equation*}
	Therefore, for any given $u$ with $\|u\| \leqslant 1$,  we have
	\begin{equation*}
	u^{\top} \Sigma_{t}^{1 / 2}\left[\begin{array}{c}{x} \\ {-1}\end{array}\right] \geqslant-\|u\|\left\|\Sigma_{t}^{1 / 2}\left[\begin{array}{c}{x} \\ {-1} \end{array}\right]\right\| \geqslant-\left\|\Sigma_{t}^{1 / 2}\left[\begin{array}{c}{x} \\ {-1}\end{array}\right]\right\|.
	\end{equation*}
	Hence,
	\begin{equation*}
	\forall u\in\mathbb{R}^d,\,\,\|u\| \leqslant 1:\quad \left(\hat{\beta}_{t}^{i}-\psi^{-1}\left(\frac{\zeta}{m}\right) \Sigma_{t}^{1 / 2} u\right)^{\top}\left[\begin{array}{c}{x} \\ {-1}\end{array}\right] \leqslant 0,
	\end{equation*}
	if and only if
	\begin{equation*}
	\hat{\beta}_{t}^{i \top}\left[\begin{array}{c}{x} \\ {-1}\end{array}\right]+\psi^{-1}\left(\frac{\zeta}{m}\right)\left\|\Sigma_{t}^{1/2}\left[\begin{array}{c}{x} \\ {-1}\end{array}\right]\right\| \leqslant 0.
	\end{equation*}
	As a result, 
	\begin{align}\label{final-safety}
	S_{t}(\zeta)=\left\{x \in \mathbb{R}^{d}: \quad \forall i \in [m]:\quad \hat{\beta}_{t}^{i \top}\left[\begin{array}{c}{x} \\ {-1}\end{array}\right]+\psi^{-1}\left(\frac{\zeta}{m}\right)\left\|\Sigma_{t}^{1/2}\left[\begin{array}{c}{x} \\ {-1}\end{array}\right]\right\| \leqslant 0 \right\}.
	\end{align}
	It now remains to compute and replace $\Sigma_{t}$ in \eqref{final-safety}. It follows that
	\begin{align*}
	\bar{X}_{t}^{\top} \bar{X}_{t}=\left[\begin{array}{c}{X_{t}^{\top}} \\ {-\mathbbm{1}^{\top}}\end{array}\right]\left[{X_{t}} , {-\mathbbm{1}}\right]=\left[\begin{array}{cc}{X_{t}^{\top} X_{t}} & {-X_{t}^{\top} \mathbbm{1}} \\ {-\mathbbm{1}^{\top} X_{t}} & {N_{t}}\end{array}\right].
	\end{align*}
	Define
	\begin{align*}
	\left[\begin{array}{ll}{O_{11}} & {O_{12}} \\ {O_{21}} & {O_{22}}\end{array}\right]:= \left(\bar{X}_{t}^{\top} \bar{X}_{t}\right)^{-1}.
	\end{align*}
	In order to compute the elemets of this block matrix, define $\bar{x}_{t}=(x_{(1)}+\cdots+x_{\left(N_{t}\right)})/N_t={X_{t}^{\top} \mathbbm{1}}/{N_{t}}$ and note that due to Lemma~\ref{matrix-inversion}, we have
	\begin{align}
	\nonumber O_{11} &=\left(X_{t}^{\top} X_{t}-\left(-X_{t}^{\top} \mathbbm{1}\right) \frac{1}{N_{t}}\left(-\mathbbm{1}^{\top} X_{t}\right)\right)^{-1}=\left(X_{t}^{\top} X_{t}-N_{t} \bar{x}_{t} \bar{x}_{t}^{\top}\right)^{-1}=: Q_t, \\ O_{22} \nonumber&=\left(N_{t}-\left(-\mathbbm{1}^{\top} X_{t}\right)\left(X_{t}^{\top} X_{t}\right)^{-\mathbbm{1}}\left(-X_{t}^{\top} \mathbbm{1}\right)\right)^{-1}=\left(N_{t}-N_{t}^{2} \bar{x}_{t}^{\top}\left(X_{t}^{\top} X_{t}\right)^{-1} \bar{x}_{t}\right)^{-1} \\ 
	\label{1}&=N_{t}^{-1}+N_{t}^{-1} N_{t}^{2} \bar{x}_{t}^{\top}\left(X_{t}^{\top} X_{t}-\bar{x}_{t} N_{t}^{-1} N_{t}^{2} \bar{x}_{t}^{\top}\right)^{-1} \bar{x}_{t} N_{t}^{-1} \\
	\nonumber&=\frac{1}{N_{t}}+\bar{x}_{t}^{\top} Q_t \bar{x}_{t},\\
	\nonumber O_{12} &=-Q_t\left(-X_{t}^{\top} \mathbbm{1}\right) \frac{1}{N_{t}}=Q_t \bar{x}_{t},\\
	\nonumber O_{21} &=-\frac{1}{N_{t}}\left(-\mathbbm{1}^{\top} X_{t}\right) Q_t=\bar{x}_{t}^{\top} Q_t,
	\end{align}
	where \eqref{1} holds due to Lemma~\ref{matrix-inversion2}. Moreover, we have
	\begin{align}
	\nonumber Q_t&=\left(X_{t}^{\top} X_{t}-N_{t} \bar{x}_{t} \bar{x}_{t}^{\top}\right)^{-1}=\left(\left[x_{(1)}, \cdots, x_{\left(N_{t}\right)}\right]\left[\begin{array}{c}{x_{(1)}^{\top}} \\ {\vdots} \\ {x_{\left(N_t\right)}^{\top}}\end{array}\right]-N_{t} \bar{x}_{t} \bar{x}_{t}^{\top}\right)^{-1}\\
	\label{Rt}&=\left(\sum_{i=1}^{N_{t}} x_{(i)} x_{(i)}^{\top}-N_{t} \bar{x}_{t}{\bar{x}_{t}}^{\top}\right)^{-1}=\left(\sum_{i=1}^{N_{t}}\left(x_{(i)}-\bar{x}_{t}\right)\left(x_{(i)}-\bar{x}_{t}\right)^{\top}\right)^{-1}.
	\end{align}
	Putting things together leads to
	\begin{align}\label{sigma}
	\Sigma_{t}=\bar{\sigma}^{2}\left(\bar{X}_{t}^{\top} \bar{X}_{t}\right)^{-1} = \bar{\sigma}^{2}\left[\begin{array}{cc}{Q_t} & {Q_t \bar{x}_{t}} \\ {\bar{x}_{t}^{\top} Q_t} & {\frac{1}{N_{t}}+\bar{x}_{t}^{\top} Q_t \bar{x}_{t}}\end{array}\right].
	\end{align}
	Having the algebraic form of $\Sigma_t$ in \eqref{sigma}, we now derive an algebraic form for $S_t(\zeta)$. Note that based on \eqref{final-safety}, $x \in S_{t}(\zeta)$ if and only if 
	\begin{align}\label{3}
	\forall i \in[m]: \quad \psi^{-1}\left(\frac{\zeta}{m}\right)\left\|\Sigma_{t}^{1 / 2}\left[\begin{array}{c}{x} \\ {-1}\end{array}\right]\right\| \leqslant-\hat{\beta}_{t}^{i \top}\left[\begin{array}{c}{x} \\ {-1}\end{array}\right]=\hat{b}_{t}^{i}-\left\langle\hat{a}_{t}^{i}, x\right\rangle.
	\end{align}
	Recall that $\hat{\epsilon}_{t}^{i}(x)=\hat{b}_{t}^{i}-\left\langle\hat{a}_{t}^{i}, x\right\rangle$ and $\kappa=\bar{\sigma} \psi^{-1}({\zeta}/{m})$. Therefore,
	\begin{align}
	\nonumber x \in S_{t}(\zeta)
	\quad&\Leftrightarrow\quad \forall i \in[m]: \quad \psi^{-1} \left(\frac{\zeta}{m}\right)^{2}\left[\begin{array}{c}{x} \\ {-1}\end{array}\right]^{\top} \Sigma_{t}\left[\begin{array}{c}{x} \\ {-1}\end{array}\right] \leqslant \hat{\epsilon}_{t}^{i}(x)^{2}\\
	\nonumber&\Leftrightarrow\quad \forall i \in[m]: \quad\bar{\sigma}^{2} \psi^{-1}\left(\frac{\zeta}{m}\right)^{2}\left(\frac{1}{N_{\mathrm{t}}}+\left(x-\bar{x}_{t}\right)^{\top} Q_t\left(x-\bar{x}_{t}\right)\right) \leqslant \hat{\epsilon}_{t}^{i}(x)^{2}\\
	\nonumber&\Leftrightarrow\quad \kappa^2\left(\frac{1}{N_{t}}+\left(x-\bar{x}_{t}\right)^{\top} Q_t\left(x-\bar{x}_{t}\right)\right) \leqslant \min _{i\in[m]} \hat{\epsilon}_{t}^{i}(x)^{2}.
	\end{align}
	Hence,
	\begin{align*}
	S_{t}(\zeta)=\left\{x \in \mathbb{R}^{d}: \quad \kappa^2\left(\frac{1}{N_{t}}+\left(x-\bar{x}_{t}\right)^{\top} Q_t\left(x-\bar{x}_{t}\right)\right) \leqslant \min _{i\in[m]} \hat{\epsilon}_{t}^{i}(x)^{2}\right\}.
	\end{align*}
\end{proof}
Prior to our next step, consider the matrix inequality notation $\succcurlyeq$ based on the Loewner order, meaning that $A \succcurlyeq B$ if and only if $A-B$ is a positive semi-definite matrix. Having this, the next lemma, provides an upper-bound on $\|Q_t\|$.
\begin{lemma}\label{RtBound}
	Consider Subroutine~\ref{datacollalg}. Then for $Q_t$ defined in Lemma~\ref{SafetySet}, we have $\|Q_t\|\leqslant {d}/{(N_{t}\, r_0^{2})}$.
\end{lemma}
\begin{proof}
	By substituting in $Q_t$ the indexing introduced by Subroutine~\ref{datacollalg}, we can write
	\begin{align}
	\nonumber Q_t^{-1}&=\sum_{i=1}^{N_{t}}\left(x_{(i)}-\bar{x}_{t}\right)\left(x_{(i)}-\bar{x}_{t}\right)^{\top}\\
	\nonumber&=\sum_{s=0}^{t} \sum_{k=0}^{\frac{n_{s}}{2d}-1} \sum_{\ell=N_{s-1}+2 d k+1}^{N_{s-1}+2 dk+2 d}\left(x_{(\ell)}-\bar{x}_{t}\,\pm \,x_s \right)\left(x_{(\ell)}-\bar{x}_{t}\,\pm \,x_s \right)^{\top}\\
	\label{6}&=\sum_{s=0}^{t} \sum_{k=0}^{\frac{n_s}{2d}-1}\left(\left[\sum_{\ell=N_{s-1} +2 d k+1}^{N_{s-1}+2 d k+2 d}\left(x_{(\ell)}-x_{s}\right)\left(x_{(\ell)}-x_{s}\right)^{\top}\right]+2 d\underbrace{\left(x_{s}-\bar{x}_{t}\right)\left(x_{s}-\bar{x}_{t}\right)^{\top}}_{\succcurlyeq 0}\right)\\
	\label{7}&\succcurlyeq\sum_{s=0}^{t} \sum_{k=0}^{\frac{n_{s}}{2d}-1} \sum_{j=1}^{d}2r_0^2 \, e_je_j^{\top}\\ 
	\label{8}&=\sum_{s=0}^{t} \sum_{k=0}^{\frac{n_{s}}{2d}-1} 2r_0^{2} \,I =\sum_{s=0}^{t}\frac{n_{s}}{2 d} \,2r_0^{2}\, I=\frac{N_{t}\,r_0^{2}}{d}\, I,
	\end{align}
	where step \eqref{6} holds due to the fact that
	\begin{align*}
	x_{s}=\frac{1}{2 d} \sum_{\ell=N_{s-1}+2 d k+1}^{N_{s-1}+2 d k+2 d} x_{(\ell)}.
	\end{align*}
	Further, step \eqref{7} follows from the fact that for fixed $k$ and $d$, we have
	\begin{align*}
	\left(x_{(\ell)}-x_{s}\right)\left(x_{(\ell)}-x_{s}\right)^{\top} = r_0^{2}\, e_{j} e_{j}^{\top},\quad \text{when}\quad j\in\set{\ell-N_{s-1}-2 d k,\ell-N_{s-1}-2 d k-d}.
	\end{align*}
	Having \eqref{8}, we conclude that 
	\begin{align*}
	\rho_{\min }\left(Q_t^{-1}\right) \geqslant \frac{N_{t}}{ d}\, r_0^{2}\quad \Rightarrow \quad \left\|Q_t\right\|=\rho_{\max }\left(Q_t\right)=\frac{1}{\rho_{\min }\left(Q_t^{-1}\right)} \leqslant \frac{d}{N_{t}\, r_0^{2}}.
	\end{align*}
\end{proof}

Next, we restate the following lemma from \cite{scl} with a slight modification to bound the distance between a point of $\cD$ and its projection on $\hat{\cD}_t$ and vice versa.
\begin{lemma}\label{projection}
	Consider Problem~\eqref{P} under Assumptions~\ref{D} and \ref{NFO}. Moreover, let $\beta \in \mathcal{E}_{t}(\zeta)$ and $N_{t} \geqslant {C_1^2}/{\left(1+\Gamma\right)^2}$, where 
	\begin{align}\label{Cd}
	C_1 :=\frac{2\kappa\,d\left(1+\Gamma\right)}{\rho_{\min }\left(\cD\right)}\, \sqrt{\frac{1+\Gamma^{2}}{r_0^{2}}+1}.
	\end{align}
	 Then 
	\begin{align}
	\label{22}\forall x \in  \hat{\cD}_{t}:& \quad \left\|x-\pi_{\cD} (x)\right\| \leqslant \frac{C_1}{\sqrt{N_{t}}}, \\
	\label{23}\forall x \in \cD:&\quad \left\|x-\pi_{\hat{\cD}_t} (x)\right\| \leqslant \frac{C_1}{\sqrt{N_{t}}},
	\end{align}
	
\end{lemma}
\begin{proof}
	Using Lemma~\ref{RtBound} and \eqref{sigma}, we can bound $\left\|\Sigma_{t}^{1 / 2}\right\|$ as follows.
	\begin{align}
	\nonumber\left\|\Sigma_{t}^{1 / 2}\right\|&=\left\|\Sigma_{t}\right\|^{1 / 2}=\bar{\sigma}\left\|\left[\begin{array}{cc}{Q_t} & {Q_t \bar{x}_{t}} \\ {\bar{x}_{t}^{\top} Q_t} & {\frac{1}{N_{t}}+\bar{x}_{t}^{\top}  Q_t \bar{x}_{t}}\end{array}\right]\right\|^{1 / 2}\\
	& = \bar{\sigma}\left\|\left[\begin{array}{l}{I} \\ {\bar{x}_{t}^{\top}}\end{array}\right] Q_t\left[I,\,\bar{x}_{t}\right]+\left[\begin{array}{ll}{0} & {0} \\ {0} & {\frac{1}{N_{t}}}\end{array}\right]\right\|^{1 / 2}\\
	\label{11}&\leqslant\bar{\sigma} \sqrt{\left\|\left[I,\,\bar{x}_{t}\right]\right\|^{2}\left\|Q_t\right\|+\frac{1}{N_{t}}}=\bar{\sigma} \sqrt{\left(1+\left\|\bar{x}_{t}\right\|^{2}\right)\left\|Q_t\right\|+\frac{1}{N_{t}}}\\
	\nonumber&\leqslant\bar{ \sigma} \sqrt{\left(1+\Gamma^{2}\right) \frac{d}{N_{t} r_0^{2}}+\frac{1}{N_{t}}}=\bar{\sigma} \sqrt{\frac{d}{N_{t}}} \sqrt{\frac{1+\Gamma^{2}}{r_0^{2}}+\frac{1}{d}}\leqslant \bar{ \sigma} \sqrt{\frac{d}{N_{t}}} \sqrt{\frac{1+\Gamma^{2}}{r_0^{2}}+1},
	\end{align}
	where the equality in \eqref{11} holds due to Lemma~\ref{9}. Therefore, we showed
	\begin{align}\label{sigmahalf}
	\left\|\Sigma_{t}^{1 / 2}\right\| \leqslant \frac{\bar{ \sigma}\,C_0}{\sqrt{N_t}},
	\end{align}
	where
	\begin{equation}\label{C0}
	C_0: =	 \sqrt{{d} \left(\frac{1+\Gamma^{2}}{r_0^{2}}+1\right)}.
	\end{equation}
	At this stage, we aim to bound the distance between vertices of $\cD$ and their corresponding estimation on $\hat{\cD}_t$. 
	Consider $u\in \operatorname{Act}(\cD)$ and let $B=B_u=\set{i_1,\ldots,i_d}$. Also, let $\hat{A}_t^B$ and $\hat{b}_t^B$ respectively be the corresponding estimations of $A^B$ and $b^B$. Moreover, by $\hat{u}_t$ denote the intersection of the constraints of $\hat{\cD}_t$ with indices in $B$. In other words,
	\begin{align*}
	\left[{\hat{A}_t^{B}}, {\hat{b}_t^{B}}\right] = \left[{{\hat{\beta}_t}^{i_1}} , {\ldots} , {{\hat{\beta}_t}^{i_d}}\right]^{\top}.
	\end{align*}
	Further, define $\gamma_{t}=\hat{b}_{t}^{B}-b^{B}$ and  $R_{t}=\hat{A}_{t}^{B}-A^{B}$.
	\begin{align}
	\nonumber\hat{u}_{t}-u &=\left(\hat{A}_{t}^{B}\right)^{-1} \hat{b}_{t}^{B}-\left(A^{B}\right)^{-1}b^{B} \\
	\nonumber &=\left(A^B+R_{t}\right)^{-1}\left(b^{B}+\gamma_{t}\right)-\left(A^{B}\right)^{-1} b^{B} \\ 
	\label{13}&=\left(\left(A^{B}\right)^{-1}-\left(I+\left(A^{B}\right)^{-1}R_{t}\right)\left(A^{B}\right)^{-1}R_{t} \left(A^{B}\right)^{-1}\right)\left(b^{B}+\gamma_{t}\right)-\left(A^{B}\right)^{-1}b^{B} \\ 
	\nonumber&=\left(A^{B}\right)^{-1} \gamma_{t}-\left(I+\left(A^{B}\right)^{-1}R_{t}\right) \left(A^{B}\right)^{-1}R_t\left(u+\left(A^{B}\right)^{-1}\gamma_t\right),
	\end{align}
	where \eqref{13} holds due to Lemma~\ref{matrix-inversion3}. Therefore, 
	\begin{align}\label{diff}
	\left\|\hat{u}_{t}-u\right\| \leqslant\left\|\left(A^{B}\right)^{-1} \gamma_{t}\right\|+\left\|\left(I+\left(A^{B}\right)^{-1} R_{t}\right)^{-1}\right\|\left\|\left(A^{B}\right)^{-1}R_{t}\right\|\left\|u+\left(A^{B}\right)^{-1} \gamma_{t}\right\|.
	\end{align}
	We now bound the terms appearing on the right-hand side of \eqref{diff}. Note that if $\beta \in \mathcal{E}_{t}(\zeta)$, then $\forall i \in[m]: \,\,\beta^{i} \in\mathcal{E}_{t}^{i}\left({\zeta}/{m}\right)$. As a result, 
	\begin{align}\label{betadiff}
	\sqrt{\left\|\hat{a}_{t}^{i}-a^{i}\right\|^{2}+\left|\hat{b}_{t}^{i}-b^{i}\right|^{2}}=\left\|\hat{\beta}_{t}^{i}-\beta^{i}\right\| \overset{\eqref{confiset}}{\leqslant} \psi^{-1}\left(\frac{\zeta}{m}\right)\left\|\Sigma_{t}^{1 / 2}\right\|\overset{\eqref{sigmahalf}}{\leqslant} \kappa \frac{C_0}{\sqrt{N_t}}.
	\end{align}
	It then immediately follows that 
	\begin{align*}
	\left\|\hat{a}_{t}^{i}-a^{i}\right\| \leqslant \kappa \frac{C_0}{\sqrt{N_t}},\quad\left|\hat{b}_{t}^{i}-b^{i}\right| &\leqslant  \kappa \frac{C_0}{\sqrt{N_t}}.
	\end{align*}
	Moreover,
	\begin{align*}
	&\left\|R_{t}\right\| \leqslant\left\|R_{t}\right\|_{F} =\sqrt{\sum_{i \in B}\left\|\hat{a}_{t}^{i}-a^{i}\right\|^{2}} \leqslant \sqrt{d} \,\psi^{-1}\left(\frac{\zeta}{m}\right)\left\|\Sigma_{t}^{1 / 2}\right\|  \leqslant  \kappa\sqrt{d}\, \frac{C_0}{\sqrt{N_t}},\\
	&\left\|\gamma_{t}\right\|=\sqrt{\sum_{i \in B}\left|\hat{b}_{t}^{i}-b^{i}\right|} \leqslant \sqrt{d} \,\psi^{-1}\left(\frac{\zeta}{m}\right)\left\|\Sigma_{t}^{1 / 2}\right\|  \leqslant \kappa\sqrt{d}\, \frac{C_0}{\sqrt{N_t}},\\
	&\left\| \left(A^{B}\right)^{-1} \right\|=\rho_{\max }\left(\left(A^{B}\right)^{-1}\right)=\frac{1}{\rho_{\min }\left(A^{B}\right)} \leqslant \frac{1}{\rho_{\min }\left(\cD\right)}.
	\end{align*}
	These equalities lead to
	\begin{align*}
	&\left\|\left(A^{B}\right)^{-1} \gamma_{t}\right\| \leqslant\left\|\left(A^{B}\right)^{-1} \right\|\left\|\gamma_{t}\right\| \leqslant
	\frac{\kappa\sqrt{d}}{\rho_{\min }\left(\cD\right)}\, \frac{C_0}{\sqrt{N_t}},\\
	&\left\|\left(A^{B}\right)^{-1} R_{t}\right\| \leqslant\left\|\left(A^{B}\right)^{-1} \right\|\left\|R_{t}\right\| \leqslant
	\frac{\kappa\sqrt{d}}{\rho_{\min }\left(\cD\right)}\, \frac{C_0}{\sqrt{N_t}}.
	\end{align*}
	Now letting
	\begin{equation*}
	U := \frac{\kappa\sqrt{d}\,C_0}{\rho_{\min }\left(\cD\right)}=\frac{\kappa\, d}{\rho_{\min }\left(\cD\right)} \left(\frac{1+\Gamma^{2}}{r_0^{2}}+1\right),
	\end{equation*}
	results in
	\begin{align}
	\label{14}\left\|\left(A^{B}\right)^{-1} \gamma_{t}\right\|\leqslant \frac{U}{\sqrt{N_{t}}}, \quad \left\|\left(A^{B}\right)^{-1} R_{t}\right\| \leqslant \frac{U}{\sqrt{N_{t}}}.
	\end{align}
	Assuming $N_{t} \geqslant 4U^{2}$ now gives $ {U}/{\sqrt{N_{t}}} \leqslant {1}/{2}$. As a result, due to Lemma~\ref{matrix-norm}, it follows that
	\begin{align}
	\label{15}\left\|\left(I+A^{B} R_{t}\right)^{-1}\right\| \leqslant \frac{1}{1-\left\|A^{B} R_{t}\right\|} \leqslant \frac{1}{1-1 / 2}=2.
	\end{align}
	Moreover,
	\begin{align}
	\label{16}\left\|u+\left(A^{B}\right)^{-1} \gamma_{t}\right\| \leqslant	\left\|u\right\|+\left\|\left(A^{B}\right)^{-1} \gamma_{t}\right\| \leqslant\Gamma+ \frac{U}{\sqrt{N_{t}}}  \leqslant\Gamma+ \frac 12.
	\end{align}
	We now can substitute \eqref{14}, \eqref{15}, and \eqref{16} into \eqref{diff} to obtain
	\begin{align}\label{w8}
	\left\|\hat{u}_{t}-u\right\| \leqslant \frac{U}{\sqrt{N_{t}}}+2 \frac{U}{\sqrt{N_{t}}}\left(\Gamma+\frac 12\right)=\frac{2 U}{\sqrt{N_{t}}}\left(\Gamma+1\right)=\frac{C_1}{\sqrt{N_{t}}},
	\end{align}
	where the last equality holds because
	\begin{align*}
	C_1 = 2U\left(\Gamma+1\right) =  \frac{2\kappa\,\sqrt{d}\,C_0\left(1+\Gamma\right)}{\rho_{\min }\left(\cD\right)}\, =\frac{2\kappa\,d\left(1+\Gamma\right)}{\rho_{\min }\left(\cD\right)}\, \sqrt{\frac{1+\Gamma^{2}}{r_0^{2}}+1}.
	\end{align*}
	Recall the assumption $N_t\geqslant 4U^2= C_1^2/(1+\Gamma)^2$ and note that it coincides with the condition mentioned in the statement. Finally, having \eqref{w8} and noting that any point in a convex polytope can be written as a convex combination of the polytope's vertices, it follows that
	\begin{align*}
	&\forall x \in  \hat{\cD}_{t}: \quad \left\|x-\pi_{\cD} (x)\right\| \leqslant \frac{C_1}{\sqrt{N_{t}}}, \\
	&\forall x \in \cD:\quad \left\|x-\pi_{\hat{\cD}_t} (x)\right\| \leqslant \frac{C_1}{\sqrt{N_{t}}}.
	\end{align*}
\end{proof}

\subsection{Final Safety Results} \label{FSR}
For the ease of statement of our main safety result, we recall some previously used constants in the following definition together with introducing the constant $\epsilon_0$.
\begin{definition}\label{defc}
	Given the confidence parameter $\delta\in(0,1)$, define
	\begin{equation*}
	\epsilon_{0}:=\min _{i\in[m]} b^i-\left\langle a^i,x_0\right\rangle,\quad L_A:=\frac{\rho_{\min}(\cD)}{2\sqrt{d}},\quad\kappa := \frac{\sigma\,L_A\,\psi^{-1}\left({\delta}/{m} \right)}{\max_{i\in[m]}\|a^i\|},
	\end{equation*}
	\begin{equation*}
	C_0: =	 \sqrt{{d} \left(\frac{1+\Gamma^{2}}{r_0^{2}}+1\right)}, \quad C_1 :=\frac{\kappa\,\left(1+\Gamma\right)C_0}{L_A}.
	\end{equation*}
\end{definition}
Next, we state our final safety result in the following proposition.
\begin{proposition}\label{nk}
	Consider Problem~\eqref{P} under Assumptions~\ref{D} and \ref{NFO} and the constants in Definition~\ref{defc}. Further, suppose $\tau>0$ is such that $\cD_{\tau}\neq \emptyset$ and let $N_t=n_0+\ldots+n_t$. Then with probability at least $1-\delta$, all the iterates of Algorithm~\ref{alg} are safe, i.e., $x_0,\ldots,x_{T-1}\in\cD$, if for all $t\in\{1,\ldots,T-1\}$, $N_{t} \geqslant {C_1^2}/{(1+\Gamma)^2}$ and $\{N_t\}$ and $\{\eta_t\}$ satisfy the following inequality:
	\begin{align}\label{Nt2}
	\frac{\kappa^{2}}{N_t}\left(1+\frac{d\Lambda^2}{ r_0^{2}} \right)\leqslant h^2\left(\set{\eta_t},\set{N_t}\right),
	\end{align}
	where
	\begin{align}\label{w32}
	h\left(\set{\eta_t},\set{N_t}\right):=\epsilon_0\, \prod_{j=0}^{t-1}\left(1-\eta_{j}\right) +\sum_{k=0}^{t-1}\left(\frac{\tau}{2}-\frac{C_1 L_{A}}{\sqrt{N_{k}}}\right) \eta_{k} \prod_{j=k+1}^{t-1}\left(1-\eta_{j}\right) -\kappa\frac{C_0}{\sqrt{N_t}}.
	\end{align}
\end{proposition}

\begin{proof}
	From \eqref{Safecond2}, it follows that to ensure the safety of the iterates with probability at least $1-\delta$, it suffices to ensure $\mathbb{P}\{\mathcal{R}\mid \mathcal{J}\}=1$. To this end, we assume that $\mathcal{J}$ holds and seek to find a sufficient condition such that $\mathcal{R}$ holds. To find such a sufficient condition, note that $\mathcal{R}$ holds if and only if $x_t$ satisfy the inequality in Lemma~\ref{SafetySet}. To achieve this we use Lemma~\ref{projection} along the way. Note that having $\mathcal{J}$ together with the assumption  $N_{t} \geqslant {C_1^2}/{(1+\Gamma)^2}$ allows us to use Lemma~\ref{projection} for all $t\in\{0,\ldots,T-1\}$. Letting $\epsilon^{i}(x)={b}^{i}-\left\langle{a}^{i}, x\right\rangle$, the proof starts by relating $\hat{\epsilon}_{t}^{i}(x)$ in Lemma~\ref{SafetySet} to $\epsilon^{i}(x)$ as follows:  
	\begin{align}
	\nonumber\hat{\epsilon}_{t}^{i}(x)&=\hat{b}_{t}^{i}-\left\langle\hat{a}_{t}^{i}, x\right\rangle=-\left\langle \hat{\beta}_{t}^{i},\left[\begin{array}{c}{x} \\ {-1}\end{array}\right]\right\rangle=-\left\langle\beta^{i} ,\left[\begin{array}{c}{x} \\ {-1}\end{array}\right]\right\rangle+\left\langle\beta^{i}-\hat{\beta}_{t}^{i},\left[\begin{array}{c}{x} \\ {-1}\end{array}\right]\right\rangle\\
	\nonumber&=b^{i}-\left\langle a^{i}, x\right\rangle+\left\langle\beta^{i}-\hat{\beta}_{t}^{i},\left[\begin{array}{c}{x} \\ {-1}\end{array}\right]\right\rangle\geqslant \epsilon^{i}(x)-\left\|\beta^{i}-\hat{\beta}_{t}^{i}\right\|\left\|\left[\begin{array}{c}{x} \\ {-1}\end{array}\right]\right\| \\
	\label{w}&\overset{\eqref{betadiff}}{\geqslant} \epsilon^{i}(x)-\kappa \frac{C_0}{\sqrt{N_t}}\,\left\|\left[\begin{array}{c}{x} \\ {-1}\end{array}\right]\right\| \geqslant  \epsilon^{i}(x)-\kappa\frac{C_0}{\sqrt{N_t}}.
	\end{align}
	Therefore, in order to lower-bound $\hat{\epsilon}_{t}^{i}(x)$, it suffices to lower-bound ${\epsilon}^{i}(x)$. To this aim, note that
	\begin{align}
	\nonumber\epsilon^{i}\left(x_{t}\right)&=b^{i}-\left\langle a^{i}, x_{t}\right\rangle= b^{i}-\left\langle a^{i}, x_{t-1}+\eta_{t-1}\left(\hat{v}_{t-1}-x_{t-1}\right)\right\rangle\\
	\nonumber&=\left(1-\eta_{t-1}\right)\left(b^{i}-\left\langle a^{i}, x_{t-1}\right\rangle\right)+\eta_{t-1}\left(b^{i}-\left\langle a^{i}, \hat{v}_{t-1}\right\rangle\right)\\
	\nonumber&=\left(1-\eta_{t-1}\right) \epsilon^{i}\left(x_{t-1}\right)+\eta_{t-1}\left(b^{i}-\left\langle a^{i}, \hat{v}_{t-1} \,\pm \pi_{\cD_{\tau}}\left(\hat{v}_{t-1}\right)  \right\rangle\right)\\
	\nonumber&=\left(1-\eta_{t-1}\right) \epsilon^{i}\left(x_{t-1}\right)+\eta_{t-1}\left(b^{i}-\left\langle a^{i}, \pi_{\cD_{\tau}}\left(\hat{v}_{t-1}\right) \right\rangle\right)-\eta_{t-1}\left\langle a^{i}, \hat{v}_{t-1}-\pi_{\cD_{\tau}}\left(\hat{v}_{t-1}\right) \right\rangle\\
	\label{26}&\geqslant\left(1-\eta_{t-1}\right) \epsilon^{i}\left(x_{t-1}\right)+\eta_{t-1}\left(b^{i}-\left\langle a^{i}, \pi_{\cD_{\tau}}\left(\hat{v}_{t-1}\right)\right\rangle\right) -\eta_{t-1} L_{A} \left\|\hat{v}_{t-1}-\pi_{\cD_{\tau}}\left(\hat{v}_{t-1}\right)\right\|.
	\end{align}
	Note that $\pi_{\cD_{\tau}}\left(\hat{v}_{t-1}\right)\in \cD_{\tau}$. Therefore, from Definition~\ref{d2}, it follows that 
	\begin{align}\label{27}
	b^{i}-\left\langle a^{i}, \pi_{\cD_{\tau}}\left(\hat{v}_{t-1}\right)\right\rangle \geqslant \tau.
	\end{align}
	Further from Lemma~\ref{alpha00}, for any $x\in\mathbb{R}^d$, we can write
	\begin{align*}
	\left\|x-\pi_{\cD_{\tau}} \left(x\right)\right\| \leqslant \left\|x-\pi_{\cD} \left(x\right)\right\| +\alpha_\cD\tau.
	\end{align*}
	By replacing $x=\hat{v}_{t-1}$ into the previous line, it follows that
	\begin{align}\label{28}
	\left\|\hat{v}_{t-1}-\pi_{\cD_{\tau}} \left(\hat{v}_{t-1}\right)\right\| \leqslant \left\|\hat{v}_{t-1}-\pi_{\cD} \left(\hat{v}_{t-1}\right)\right\| +\alpha_{\cD}\tau\overset{\eqref{22}}{\leqslant} \frac{C_1}{\sqrt{N_{t-1}}}+\alpha_{\cD}\tau.
	\end{align}
	Replacing the bounds in \eqref{27} and \eqref{28} into \eqref{26}, then results in
	\begin{align}
	\nonumber\epsilon^{i}\left(x_{t}\right) &\geqslant\left(1-\eta_{t-1}\right) \epsilon^{i}\left(x_{t-1}\right)+\tau\,\eta_{t-1}\left(1-\alpha_{\cD} L_A\right) -\eta_{t-1} L_{A} \frac{C_1}{\sqrt{N_{t-1}}}\\
	\label{30}&=\left(1-\eta_{t-1}\right) \epsilon^{i}\left(x_{t-1}\right)+\frac{\tau}{2}\,\eta_{t-1} -\eta_{t-1} L_{A} \frac{C_1}{\sqrt{N_{t-1}}},
	\end{align}
	where the last equality is due to the fact that $L_A=1/2\alpha_{\cD}$. Note that \eqref{30} introduces a recursion in terms of $t$. Continuing this recursion leads to an explicit bound for $\epsilon^{i}\left(x_{t}\right)$ as follows
	\begin{align}\label{w2}
	\epsilon^{i}\left(x_{t}\right)\geqslant \epsilon^{i}\left(x_{0}\right)\, \prod_{j=0}^{t-1}\left(1-\eta_{j}\right)+\frac{\tau}{2} \sum_{k=0}^{t-1} \eta_{k} \prod_{j=k+1}^{t-1}\left(1-\eta_{j}\right)-C_1 L_{A} \sum_{k=0}^{t-1}\frac{\eta_{k}}{\sqrt{N_{k}}} \prod_{j=k+1}^{t-1}\left(1-\eta_{j}\right).
	\end{align}
	As the next step, by replacing \eqref{w2} into \eqref{w} and letting $\epsilon_{0}:=\min _{i\in[m]} \epsilon^{i}\left(x_{0}\right)$, it follows that
	\begin{align}\label{32}
	\min _{i\in[m]} \hat{\epsilon}_t^{i}\left(x_{t}\right)\geqslant h\left(\set{\eta_t},\set{N_t}\right),
	\end{align}
	where 
	\begin{align}\label{w3}
	h\left(\set{\eta_t},\set{N_t}\right):=\epsilon_0\, \prod_{j=0}^{t-1}\left(1-\eta_{j}\right) +\sum_{k=0}^{t-1}\left(\frac{\tau}{2}-\frac{C_1 L_{A}}{\sqrt{N_{k}}}\right) \eta_{k} \prod_{j=k+1}^{t-1}\left(1-\eta_{j}\right) -\kappa\frac{C_0}{\sqrt{N_t}}.
	\end{align}
	Recall that we seek to find a sufficient condition such that $\mathcal{R}$ holds, i.e., $x_t\in S_t(\zeta)$ for all $t$. Now consider the inequality in the algebraic expression of $S_t(\zeta)$ in Lemma~\ref{SafetySet}. The expression \eqref{32} gives a lower-bound for the right-hand side of this inequality. An upper-bound for the left-hand side of this inequality can also be found as follows:
	\begin{align}\label{left}
	\kappa^2\left(\frac{1}{N_{t}}+\left(x-\bar{x}_{t}\right)^{\top} Q_t\left(x-\bar{x}_{t}\right)\right) \leqslant \kappa^2\left(\frac{1}{N_{t}}+\left\|Q_t\right\|\left\|x_{t}-\bar{x}_{t}\right\|^{2}\right)\leqslant \frac{\kappa^{2}}{N_t}\left(1+\frac{d\Lambda^2}{ r_0^{2}} \right),
	\end{align}
	where the last inequality follows from Lemma~\ref{RtBound}. Hence, by comparing \eqref{32}, \eqref{left}, and Lemma~~\ref{SafetySet}, in order to ensure $\mathcal{R}$, it suffices to have
	\begin{align*}
	\frac{\kappa^{2}}{N_t}\left(1+\frac{d\Lambda^2}{ r_0^{2}} \right)\leqslant h^2\left(\set{\eta_t},\set{N_t}\right).
	\end{align*}
\end{proof}
\section{Convergence Analysis}\label{AC}
The main result regarding the convergence part of Algorithm~\ref{alg} is presented below. 
\begin{proposition}\label{conv}
		Consider Problem~\eqref{P} under Assumptions~\ref{f}, \ref{D}, and \ref{NFO}, $C_1$ from Definition~\ref{defc}, and let $N_t=n_0+\ldots+n_t$. Suppose at iteration $t$ of Algorithm~\ref{alg}, we have \eqref{22}, \eqref{23}, and $x_t\in\cD$, then
		\begin{align*}
			V_{\cD}(x_{t}, f)\leqslant\frac{f(x_{t})-f(x_{t+1})}{\eta_t} +\|\nabla f(x_{t})-g_{t}\|(\frac{C_1}{\sqrt{N_{t}}}+2\Lambda)+\|g_{t} \|\frac{C_1}{\sqrt{N_{t}}}+\frac{L \eta_{t}}{2}(\frac{C_1}{\sqrt{N_{t}}}+\Lambda)^2.
		\end{align*}
\end{proposition}

\begin{proof}
	Consider
	\begin{align}
	\label{20} v_{t}&:=\arg \min_{v\in \cD}\, \left\langle \nabla f(x_t), v\right\rangle, \\
	\label{21} \hat{v}_{t}&:=\arg \min_{v\in \hat{\cD}_t}\, \left\langle g_t, v\right\rangle.
	\end{align}
	The goal here is to find the number of iterations $T$ in terms of the accuracy $\epsilon$. Note that 
	\begin{align}
	\nonumber V_{\cD}\left(x_{t}, f\right)\overset{\eqref{20}}{=}\left\langle\nabla f\left(x_{t}\right), x_{t}-v_{t}\right\rangle&=\left\langle\nabla f\left(x_{t}\right)-g_{t}, x_{t}-v_{t}\right\rangle+\left\langle g_{t}, x_{t}-v_{t}\right\rangle\\
	\label{42}&\leqslant\left\|\nabla f\left(x_{t}\right)-g_{t}\right\| \Lambda+\left\langle g_{t}, x_{t}-v_{t}\right\rangle.
	\end{align}
	We need to bound the right-hand side of \eqref{42}. To this aim, note that $L$-smoothness of $f$ implies
	\begin{align}
	\nonumber f\left(x_{t+1}\right) &\leqslant f\left(x_{t}\right)+\left\langle\nabla f\left(x_{t}\right), x_{t+1}-x_{t}\right\rangle+\frac{L}{2}\left\|x_{t+1}-x_{t}\right\|^{2}\\
	\label{33}&=f\left(x_{t}\right)+\eta_{t}\left\langle\nabla f\left(x_{t}\right), \hat{v}_{t}-x_{t}\right\rangle+\frac{L \eta_{t}^{2}}{2}\left\|\hat{v}_{t}-x_{t}\right\|^{2}.
	\end{align}
	In order to bound the terms appearing on the right-hand side of \eqref{33}, we can write 
	\begin{align}
	\label{41}\left\|\hat{v}_{t}-x_{t}\right\| \leqslant\left\|\hat{v}_{t}-\pi_{\cD}\left(\hat{v}_{t}\right)\right\|+\left\|\pi_{\cD}\left(\hat{v}_{t}\right)-x_{t}\right\|\overset{\eqref{22}}{ \leqslant}\frac{C_1}{\sqrt{N_{t}}}  +\Lambda.
	\end{align}
	Further, note that
	\begin{align}
	\label{34}\left\langle\nabla f\left(x_{t}\right), \hat{v}_{t}-x_{t}\right\rangle=\left\langle\nabla f\left(x_{t}\right)-g_{t}, \hat{v}_{t}-x_{t}\right\rangle+\left\langle g_{t}, \hat{v}_{t}-v_{t}\right\rangle+\left\langle g_{t}, v_{t}-x_{t}\right\rangle.
	\end{align}
	To bound the right-hand side of \eqref{34}, we have
	\begin{align}
	\label{35}\left\langle\nabla f\left(x_{t}\right)-g_{t}, \hat{v}_{t}-x_{t}\right\rangle&\leqslant\left\|\nabla f\left(x_{t}\right)-g_{t}\right\|\left\|\hat{v}_{t}-x_{t}\right\|\overset{\eqref{41}}{\leqslant}\left\|\nabla f\left(x_{t}\right)-g_{t}\right\|\left(\frac{C_1}{\sqrt{N_{t}}}+\Lambda\right).
	\end{align}
	Further, to bound $\left\langle g_{t}, \hat{v}_{t}-v_{t}\right\rangle$, first note that
	\begin{align}
	\label{25}\left\langle g_{t}, \pi_{\hat{\cD}_t}\left(v_{t}\right) -v_{t}\right\rangle  \leqslant \left\|g_{t}\right\|  \left\|\pi_{\hat{\cD}_t}\left(v_{t}\right) -v_{t}\right\|  \overset{\eqref{23}}{ \leqslant} \left\|g_{t}\right\| \frac{C_1}{\sqrt{N_{t}}}.
	\end{align}
	Hence,
	\begin{align}
	\label{37}\left\langle g_{t}, \hat{v}_{t}\right\rangle\overset{\eqref{21}}{ \leqslant}\left\langle g_{t}, \pi_{\hat{\cD}_t}\left(v_{t}\right) \right\rangle \overset{\eqref{25}}{ \leqslant} \left\langle g_{t}, v_{t}\right\rangle+\left\|g_{t}\right\| \frac{C_1}{\sqrt{N_{t}}} \quad&\Rightarrow\quad \left\langle g_{t}, \hat{v}_{t}-v_{t}\right\rangle \leqslant\left\|g_{t}\right\| \frac{C_1}{\sqrt{N_{t}}}.
	\end{align}
	
	Now replacing \eqref{37} and \eqref{35} into \eqref{34} implies
	\begin{align}
	\label{40}\left\langle\nabla f\left(x_{t}\right), \hat{v}_{t}-x_{t}\right\rangle&\leqslant\left\langle g_{t}, v_{t}-x_{t}\right\rangle+\left\|\nabla f\left(x_{t}\right)-g_{t}\right\|\left(\frac{C_1}{\sqrt{N_{t}}}+\Lambda\right)+\left\|g_{t} \right\| \frac{C_1}{\sqrt{N_{t}}}.
	\end{align}
	Moreover, substituting \eqref{41} and \eqref{40} into \eqref{33} results in
	\begin{align*}
	\langle g_{t}, x_{t}-v_{t}\rangle\leqslant\frac{f(x_{t})-f(x_{t+1})}{\eta_t}+\|\nabla f(x_{t})-g_{t}\|(\frac{C_1}{\sqrt{N_{t}}}+\Lambda)+\|g_{t} \|\frac{C_1}{\sqrt{N_{t}}}+\frac{L \eta_{t}}{2}(\frac{C_1}{\sqrt{N_{t}}}+\Lambda)^2.
	\end{align*}
	Finally, substituting this into \eqref{42} leads to 
	\begin{align*}
	V_{\cD}(x_{t}, f)\leqslant\frac{f(x_{t})-f(x_{t+1})}{\eta_t} +\|\nabla f(x_{t})-g_{t}\|(\frac{C_1}{\sqrt{N_{t}}}+2\Lambda)+\|g_{t} \|\frac{C_1}{\sqrt{N_{t}}}+\frac{L \eta_{t}}{2}(\frac{C_1}{\sqrt{N_{t}}}+\Lambda)^2.
	\end{align*}
\end{proof}
\section{Proofs of The Main Results}\label{FR}
In this section, based on the general results from the safety and convergence analysis, i.e., Propositions~\ref{nk} and \ref{conv}, we obtain sufficient values for the stepsize $\eta_t$, the number of NFO queries per iteration $n_t$, and the total number of iterations $T$ such that all the requirements determined in Section~\ref{PF} are fulfilled. More specifically, previously mentioned Theorems~\ref{maintheorem2} and \ref{maindet2} (in Section~\ref{A}) are precisely presented and proved below. Moreover, our result under a convex objective function is provided in Theorem~\ref{convex}. We start with introducing the following notation.

\paragraph{Notation.} Let $F(x_t)$ be any function of the iterate $x_t$. For $0\leqslant t'<t$, we denote by $\mathbb{E}_{0:t'}[F(x_t)]$, the expectation over all the randomness in Algorithm~\ref{alg} through iterations $0$ to $t'$. Further, $\mathbb{E}_{[t_0]}[F(x_t)]$ denotes the expectation over choosing $t_0$ from $\set{0,\ldots,T-1}$ uniformly at random, that is, $\mathbb{E}_{[t_0]}[F(x_t)]=(F(x_0)+\ldots+F(x_{T-1}))/T$.

Next, we define constants that will be later used in our main theorems. 
\begin{definition}\label{d3}
	Consider the constants in Definition~\ref{defc} and a given constant $\tau>0$. Then define
	\begin{equation}\label{C}
	C_2:=\max\set{\left(\frac{4C_1 L_{A}}{\tau}\right)^2,\,\left(\frac{8\kappa\,C_0}{\tau}\right)^2,\,\frac{64\,\kappa^2}{\tau^2}\left(1+\frac{d \Lambda^{2}}{r_0^{2}}\right),\,  C_1^2}.
	\end{equation}
	\begin{equation}\label{T}
	C_3 := \left(18\sqrt{2}(2\Lambda+1)(\sigma_0+L_0\Lambda)\sqrt{\log\left(\frac{4}{\delta}\right)}\right)^3,\,\, C_4:=\left(9 (M+\sigma_0)+6L(1+\Lambda)^{2}\right)^{\frac 32},
	\end{equation}
	\begin{equation}\label{T2}
	C_5 := \left(4 M+2L\left(1+\Lambda^{2}\right)\right)^2,
	\end{equation}
	\begin{align}
	C_6 := \left(2\max\left\{16\sqrt{2}(1+\Lambda)\left( L_0 \Lambda +  \sigma_0\right)\sqrt{ \log\left(\frac{4}{\delta}\right)},4\sqrt{2}\left(M+\sigma_0\right),L^2(\Lambda+2)\right\}\right)^2.
	\end{align}
\end{definition}

Next, we restate the following lemma from \cite{xie2019stochastic}, in which the proof is also provided.
\begin{lemma}\label{GSL}
	Consider the STORM estimator in Algorithm \ref{alg} with $\eta_t = \rho_t = (t + 2)^{-\alpha}$ for some $\alpha \in (0, 1]$ and let $\delta\in(0,1)$. Then under Assumption~\ref{G}, for any $t \in\{0,\ldots,T-1\}$, with probability at least $1 - \delta$,
	\begin{align}\label{storm}
		\| g_t - \nabla f(x_t) \| \leqslant \frac{2}{(t + 2)^{ \frac{\alpha}{2}} } \left(2 L_0 \Lambda + \frac{3^{\alpha} \sigma_0}{3^{\alpha} - 1}\right)\sqrt{2\, \log\left(\frac{4}{\delta}\right)}.
	\end{align}
\end{lemma}

The following theorem evaluates Algorithm~\ref{alg} in solving Problem~\eqref{P}.
\begin{theo}
	Consider Problem~\eqref{P} under Assumptions~\ref{f}-\ref{NFO}. Suppose $\epsilon >0$ is given and let $\tau>0$ be such that $\cD_{\tau}\neq \emptyset$. Further, let $\eta_t=\rho_t=(t+2)^{-2/3}$, $n_t=2C_2(t+1)^{\frac 13}$, and
	\begin{align}\label{T3}
	T= \max \left\{\frac{216}{\epsilon^3}\left(f\left(x_{0}\right)-f\left(x^{*}\right)\right)^3,\,\frac{C_3}{\epsilon^3},\,\, \frac{C_4}{\epsilon^{\frac 32}} \right\},
	\end{align}
	 where $C_2$, $C_3$, and $C_4$ are given in Definition~\ref{d3}. Then, with probability at least $1-\delta$, in Algorithm~\ref{alg}, all the iterates are safe, i.e., $x_0,\ldots,x_{T-1}\in\cD$, all the query points lie within $r_0$-vicinity of $\cD$, and $\mathbb{E}_{[t_0]}\mathbb{E}_{0:t-1}[V_{\cD}(x_{t}, f)] \leqslant \epsilon$.
\end{theo}

\begin{proof}
	Note that if $x_0,\ldots,x_{T-1}\in\cD$, then all the query points lie within $r_0$-vicinity of $\cD$ due to the construction of Subroutine~\ref{datacollalg}. Now to achieve $x_0,\ldots,x_{T-1}\in\cD$ with probability at least $1-\delta$, based on Proposition~\ref{nk}, it suffices to have $N_t\geqslant C_1^2/(1+\Gamma)^2$ and \eqref{Nt2}. For these to hold, we prove that it suffices to choose $\eta_t=\rho_t=(t+2)^{-2/3}$ and $N_t\geqslant C_2(t+1)^{\frac 43}$. To this aim, we let $N_t= C(t+1)^{\frac 43}$ and show that it suffices to have $C=C_2$. As a result, to get $N_t\geqslant C_2(t+1)^{\frac 43}$, it then suffices to choose $n_t=2C_2(t+1)^{\frac 13}$, where $N_t=n_0+\ldots+n_t$. As the first step, note that
	\begin{align}\label{1w7}
	\forall t\in\set{0,\ldots,T-1}:\,N_t=C(t+1)^{\frac 43}\geqslant \frac{C_1^2}{(1+\Gamma)^2} \quad\Longleftrightarrow\quad C \geqslant \frac{C_1^2}{\left(1+\Gamma\right)^2}.
	\end{align}
	Now, let $h_1$ be the second term in the right-hand side of \eqref{w32}, i.e.,
	\begin{align}
	\label{155} h_1\left(\set{\eta_t},\set{N_t}\right):=\sum_{k=0}^{t-1}\left(\frac{\tau}{2}-\frac{C_1 L_{A}}{\sqrt{N_{k}}}\right) \eta_{k} \prod_{j=k+1}^{t-1}\left(1-\eta_{j}\right).
	\end{align}
	To get a sufficient condition on $C$, we now enforce
	\begin{align}\label{1cond1}
	\forall t\in\set{0,\ldots, T-1}:\, \frac{\tau}{2}-\frac{C_1 L_{A}}{\sqrt{N_{t}}} = \frac{\tau}{2}-\frac{C_1 L_{A}}{\sqrt{C\left(t+1\right)^{\frac 43}}} \geqslant \frac{\tau}{4}\,\,\, \Leftrightarrow \,\,\,C \geqslant \left(\frac{4C_1 L_{A}}{\tau}\right)^2.
	\end{align}
	Replacing \eqref{1cond1} and the values of $\eta_t$ and $N_t$ into \eqref{155}, results in
	\begin{align}
	\label{156} h_1\left(\set{\eta_t},\set{N_t}\right)\geqslant\frac{\tau}{4}\frac{2^{\frac 23}}{\left(t+1\right)^{\frac 23}}\left(1-\left(1-\frac{1}{2^{\frac 23}}\right)^{t}\right)\geqslant\frac{\tau}{4\left(t+1\right)^{\frac 23}}.
	\end{align}
	As a result, having \eqref{156} and \eqref{w32}, we can write
	\begin{align}\label{1w5}
	h\left(\set{\eta_t},\set{N_t}\right)\geqslant h_1\left(\set{\eta_t},\set{N_t}\right)-\kappa\frac{C_0}{\sqrt{N_t}}\geqslant\left(\frac{\tau}{4}-\frac{\kappa\,C_0}{\sqrt{C}}\right)\,\frac{1}{\left(t+1\right)^{\frac 23}}.
	\end{align}
	To get another sufficient condition on $C$, we enforce
	\begin{align}\label{1w4}
	\frac{\tau}{4}-\frac{\kappa\, C_0}{\sqrt{C}}\geqslant \frac{\tau}{8} \quad \Longleftrightarrow \quad C\geqslant \left(\frac{8\kappa\,C_0}{\tau}\right)^2.
	\end{align}
	Substitute \eqref{1w4} into \eqref{1w5} to obtain
	\begin{align}\label{1w6}
	h\left(\set{\eta_t},\set{N_t}\right)\geqslant \frac{\tau}{8\left(t+1\right)^{\frac 23}}.
	\end{align}
	Replacing \eqref{1w6} into \eqref{Nt2} leads to another sufficient condition on $C$, that is,
	\begin{align}\label{1cond2}
	C \geqslant \frac{64\,\kappa^2}{\tau^2}\left(1+\frac{d \Lambda^{2}}{r_0^{2}}\right).
	\end{align}
	Note that \eqref{1w7}, \eqref{1cond1}, \eqref{1w4}, and \eqref{1cond2} together guarantee that $N_t\geqslant C_1^2/(1+\Gamma)^2$ and \eqref{Nt2} hold. Therefore, due to Proposition~\ref{nk}, with probability at least $1-\delta$, we have $x_0,\ldots,x_{T-1}\in\cD$. Now the construction of Subroutine~\ref{datacollalg} ensures that all the query points lie within $r_0$-vicinity of $\cD$. 
	
	Next, we show the convergence. Recall the notation in \eqref{jr} and note that under the assumptions of Proposition~\ref{nk}, it ensures that $\mathbb{P}\{\mathcal{R}\mid\mathcal{J}\}=1$ and thus from \eqref{jj}, it follows that having $\mathcal{J}$ results in $x_0,\ldots,x_{T-1}\in\cD$. Knowing that $\mathcal{J}$ holds with probability at least $1-\delta$, to show that convergence holds simultaneously with the safety together with probability at least $1-\delta$, it suffices to show that the convergence holds under having $\mathcal{J}$. To this aim, suppose $\mathcal{J}$ holds. Having $\mathcal{J}$ and \eqref{1w7} together, we conclude from Lemma~\ref{projection} that \eqref{22} and \eqref{23} hold for every $t\in\{0,\ldots,T-1\}$. Therefore, from Proposition~\ref{conv}, for every $t\in\{0,\ldots,T-1\}$, we have  
	\begin{align*}
	V_{\cD}(x_{t}, f)\leqslant\frac{f(x_{t})-f(x_{t+1})}{\eta_t} +\|\nabla f(x_{t})-g_{t}\|(\frac{C_1}{\sqrt{N_{t}}}+2\Lambda)+\|g_{t} \|\frac{C_1}{\sqrt{N_{t}}}+\frac{L \eta_{t}}{2}(\frac{C_1}{\sqrt{N_{t}}}+\Lambda)^2.
	\end{align*}
	To simplify, we enforce the following sufficient condition:
	\begin{align}\label{1cond3}
	\frac{C_1}{\sqrt{C}}\leqslant 1\quad \Longleftrightarrow\quad C\geqslant C_1^2.
	\end{align}
	From Assumption~\ref{G}, it follows that $\|g_t\|\leqslant M+\sigma_0$. By replacing this together with \eqref{1cond3} and the values of $N_t$ and $\eta_t$ in the above expression, we arrive at  
	\begin{align*}
	V_{\cD}\left(x_{t}, f\right) \leqslant ({t+2})^{\frac 23} \left(f\left(x_{t}\right)-f\left(x_{t+1}\right)\right)&+\left\|\nabla f\left(x_{t}\right)-g_{t}\right\|\left(2\Lambda+1\right)\\
	&+\frac{1}{\left({t+1}\right)^{\frac 23}}\left(M+\sigma_0+\frac{L}{2}\left(1+\Lambda\right)^{2}\right).
	\end{align*}
	As the next step, taking the expectation of both sides from $0$ to $t-1$ implies
	\begin{align*}
	\mathbb{E}_{0:t-1}\left[V_{\cD}\left(x_{t}, f\right)\right] &\leqslant ({t+2})^{\frac 23} \left(\mathbb{E}_{0:t-1}\left[f\left(x_{t}\right)\right]-\mathbb{E}_{0:t}\left[f\left(x_{t+1}\right)\right]\right)\\
	&+\mathbb{E}_{0:t}\left[\left\|\nabla f\left(x_{t}\right)-g_{t}\right\|\right]\left(2\Lambda+1\right)+\frac{1}{\left({t+1}\right)^{\frac 23}}\left(M+\sigma_0+\frac{L}{2}\left(1+\Lambda\right)^{2}\right).
	\end{align*}
	Note that Algorithm~\ref{alg} picks $t_0$ uniformly at random from $\set{0,\ldots,T-1}$. Taking the expectation over this $t_0$ then leads to 
	\begin{align}
	\nonumber\mathbb{E}_{[t_0]}\mathbb{E}_{0:t-1}\left[V_{\cD}\left(x_{t}, f\right)\right] &\leqslant\frac{ \left({T+1}\right)^{\frac 23}}{T} \left(f\left(x_0\right)-f\left(x^*\right)\right)\\
	\nonumber&+\frac{2\Lambda+1}{T}\sum_{t=0}^{T-1}\mathbb{E}_{0:t}\left[\left\|\nabla f\left(x_{t}\right)-g_{t}\right\|\right]\\
	\label{ee}&+\frac{1}{T}\sum_{t=0}^{T-1}\frac{1}{\left({t+1}\right)^{\frac 23}}\left(M+\sigma_0+\frac{L}{2}\left(1+\Lambda\right)^{2}\right).
	\end{align}
	The idea for the rest of this proof is to bound the right-hand side of \eqref{ee} by $\epsilon$ and see what sufficient conditions will be enforced on $T$. To this end, it suffices to bound each term by $\epsilon/3$. Applying this bound on the first term, results in
	\begin{align}\label{146}
	\frac{2}{T^{\frac 13}}\left(f\left(x_{0}\right)-f\left(x^{\star}\right)\right) \leqslant \frac{\epsilon}{3}\quad \Longleftrightarrow\quad \frac{216}{\epsilon^3}\left(f\left(x_{0}\right)-f\left(x^{*}\right)\right)^3 \leqslant T.
	\end{align}
	To apply the $\epsilon/3$ bound to the second term, by replacing $\alpha=2/3$ in Lemma~\ref{GSL}, it suffices to have
	\begin{equation*}
	\frac{1}{T}\sum_{t=0}^{T-1} \frac{2(2\Lambda+1)}{(t + 2)^{ \frac{1}{3}} } \left(2 L_0 \Lambda + \frac{3^{\frac 23} \sigma_0}{3^{\frac 23} - 1}\right)\sqrt{2\, \log\left(\frac{4}{\delta}\right)}\leqslant 6(2\Lambda+1)(\sigma_0+L_0\Lambda)\sqrt{2\log\left(\frac{4}{\delta}\right)}\leqslant \frac{\epsilon}{3}.
	\end{equation*}
	Hence, it suffices to have
	\begin{equation}\label{second}
		\frac{C_3}{\epsilon^{3}}\leqslant T.
	\end{equation}
	Moreover, to apply the $\epsilon/3$ bound to the third term, we can write
	\begin{align*}
	\frac{1}{T}\sum_{t=0}^{T-1}\frac{1}{\left({t+1}\right)^{\frac 23}}\left(M+\sigma_0+\frac{L}{2}\left(1+\Lambda\right)^{2}\right)\leqslant\frac{3(M+\sigma_0)+2L\left(1+\Lambda\right)^{2}}{T^{\frac 23}}\leqslant \frac{\epsilon}{3}.
	\end{align*}
	Thus, it suffices to have
	\begin{equation}\label{third}
		\frac{C_4}{\epsilon^{\frac 32}}\leqslant T.
	\end{equation}
	Finally, from \eqref{146}, \eqref{second}, and \eqref{third}, it suffices to choose $T$ as \eqref{T3}. Further, by putting together \eqref{1w7}, \eqref{1cond1}, \eqref{1w4}, \eqref{1cond2}, and \eqref{1cond3}, it suffices to choose $C=C_2$.
\end{proof}

Under having access to the gradient of the objective function, we present the following result:
\begin{theo}
	Consider Problem~\eqref{P} under Assumptions~\ref{f}, \ref{D}, and \ref{NFO} and assume we have access to $\nabla f$. Suppose $\epsilon >0$ is given and let $\tau>0$ be such that $\cD_{\tau}\neq \emptyset$. Further, let $\eta_t=(t+2)^{-1/2}$, $\rho_t=1$, $n_t=C_2$, and
	\begin{align}\label{T4}
	T= \frac{1}{\epsilon^{2}}\max \left\{8\left(f\left(x_{0}\right)-f\left(x^{*}\right)\right)^2, C_5\right\},
	\end{align}
	where $C_2$ and $C_5$ are given in Definition~\ref{d3}. Then, with probability at least $1-\delta$, in Algorithm~\ref{alg}, all the iterates are safe, i.e., $x_0,\ldots,x_{T-1}\in\cD$, all the query points lie within $r_0$-vicinity of $\cD$, and $
	\mathbb{E}_{[t_0]}\mathbb{E}_{0:t-1}[V_{\cD}(x_{t}, f)] \leqslant \epsilon$.
\end{theo} 

\begin{proof}
	Note that if $x_0,\ldots,x_{T-1}\in\cD$, then all the query points lie within $r_0$-vicinity of $\cD$ due to the construction of Subroutine~\ref{datacollalg}. Now to achieve $x_0,\ldots,x_{T-1}\in\cD$ with probability at least $1-\delta$, based on Proposition~\ref{nk}, it suffices to have $N_t\geqslant C_1^2/(1+\Gamma)^2$ and \eqref{Nt2}. For these to hold, we prove that it suffices to choose $\eta_t=(t+2)^{-1/2}$ and $N_t\geqslant C_2(t+1)$. To this aim, we let $N_t= C(t+1)$ and show that it suffices to have $C=C_2$. As a result, to get $N_t\geqslant C_2(t+1)$, it then suffices to choose $n_t=C_2$, where $N_t=n_0+\ldots+n_t$. As the first step, note that
	\begin{align}\label{1w78}
	\forall t\in\set{0,\ldots,T-1}:\,N_t=C(t+1)\geqslant \frac{C_1^2}{(1+\Gamma)^2} \quad\Longleftrightarrow\quad C \geqslant \frac{C_1^2}{\left(1+\Gamma\right)^2}.
	\end{align}
	Now, let $h_1$ be the second term in the right-hand side of \eqref{w32}, i.e.,
	\begin{align}
	\label{1558} h_1\left(\set{\eta_t},\set{N_t}\right):=\sum_{k=0}^{t-1}\left(\frac{\tau}{2}-\frac{C_1 L_{A}}{\sqrt{N_{k}}}\right) \eta_{k} \prod_{j=k+1}^{t-1}\left(1-\eta_{j}\right).
	\end{align}
	To get a sufficient condition on $C$, we now enforce
	\begin{align}\label{1cond18}
	\forall t\in\set{0,\ldots, T-1}:\, \frac{\tau}{2}-\frac{C_1 L_{A}}{\sqrt{N_{t}}} = \frac{\tau}{2}-\frac{C_1 L_{A}}{\sqrt{C\left(t+1\right)}} \geqslant \frac{\tau}{4}\,\,\, \Leftrightarrow \,\,\,C \geqslant \left(\frac{4C_1 L_{A}}{\tau}\right)^2.
	\end{align}
	Replacing \eqref{1cond18} and the values of $\eta_t$ and $N_t$ into \eqref{1558}, results in
	\begin{align}\label{1568}
	h_1\left(\set{\eta_k},\set{N_k}\right)\geqslant\frac{\tau}{4}\sqrt{\frac{2}{t+1}}\left(1-\left(1-\frac{1}{\sqrt{2}}\right)^{t}\right)\geqslant\frac{\tau}{4\sqrt{t+1}}.
	\end{align}
	As a result, having \eqref{1568} and \eqref{w32}, we can write
	\begin{align}\label{1w58}
	h\left(\set{\eta_t},\set{N_t}\right)\geqslant h_1\left(\set{\eta_t},\set{N_t}\right)-\kappa\frac{C_0}{\sqrt{N_t}}\geqslant\left(\frac{\tau}{4}-\frac{\kappa\,C_0}{\sqrt{C}}\right)\,\frac{1}{\sqrt{t+1}}.
	\end{align}
	To get another sufficient condition on $C$, we enforce
	\begin{align}\label{1w48}
	\frac{\tau}{4}-\frac{\kappa\, C_0}{\sqrt{C}}\geqslant \frac{\tau}{8} \quad \Longleftrightarrow \quad C\geqslant \left(\frac{8\kappa\,C_0}{\tau}\right)^2.
	\end{align}
	Substitute \eqref{1w48} into \eqref{1w58} to obtain
	\begin{align}\label{1w68}
	h\left(\set{\eta_t},\set{N_t}\right)\geqslant \frac{\tau}{8\sqrt{t+1}}.
	\end{align}
	Replacing \eqref{1w68} into \eqref{Nt2} leads to another sufficient condition on $C$, that is,
	\begin{align}\label{1cond28}
	C \geqslant \frac{64\,\kappa^2}{\tau^2}\left(1+\frac{d \Lambda^{2}}{r_0^{2}}\right).
	\end{align}
	Note that \eqref{1w78}, \eqref{1cond18}, \eqref{1w48}, and \eqref{1cond28} together guarantee that $N_t\geqslant C_1^2/(1+\Gamma)^2$ and \eqref{Nt2} hold. Therefore, due to Proposition~\ref{nk}, with probability at least $1-\delta$, we have $x_0,\ldots,x_{T-1}\in\cD$. Now the construction of Subroutine~\ref{datacollalg} ensures that all the query points lie within $r_0$-vicinity of $\cD$. 
	
	Next, we show the convergence. Recall the notation in \eqref{jr} and note that under the assumptions of Proposition~\ref{nk}, it ensures that $\mathbb{P}\{\mathcal{R}\mid\mathcal{J}\}=1$ and thus from \eqref{jj}, it follows that having $\mathcal{J}$ results in $x_0,\ldots,x_{T-1}\in\cD$. Knowing that $\mathcal{J}$ holds with probability at least $1-\delta$, to show that convergence holds simultaneously with the safety together with probability at least $1-\delta$, it suffices to show that the convergence holds under having $\mathcal{J}$. To this aim, suppose $\mathcal{J}$ holds. Having $\mathcal{J}$ and \eqref{1w78} together, we conclude from Lemma~\ref{projection} that \eqref{22} and \eqref{23} hold for every $t\in\{0,\ldots,T-1\}$. Therefore, from Proposition~\ref{conv}, for every $t\in\{0,\ldots,T-1\}$, we have  
	\begin{align*}
	V_{\cD}(x_{t}, f)\leqslant\frac{f(x_{t})-f(x_{t+1})}{\eta_t} +\|\nabla f(x_t) \|\frac{C_1}{\sqrt{N_{t}}}+\frac{L \eta_{t}}{2}(\frac{C_1}{\sqrt{N_{t}}}+\Lambda)^2.
	\end{align*}
	To simplify, we enforce the following sufficient condition:
	\begin{align}\label{1cond38}
	\frac{C_1}{\sqrt{C}}\leqslant 1\quad \Longleftrightarrow\quad C\geqslant C_1^2.
	\end{align}
	By replacing $\|\nabla f(x_t)\|\leqslant M$, the condition \eqref{1cond38}, and the values of $N_t$ and $\eta_t$ in the above expression, we arrive at  
	\begin{align*}
	V_{\cD}\left(x_{t}, f\right) \leqslant \sqrt{t+2} \left(f\left(x_{t}\right)-f\left(x_{t+1}\right)\right)+\frac{1}{\sqrt{t+1}}\left(M+\frac{L}{2}\left(1+\Lambda\right)^{2}\right).
	\end{align*}
	As the next step, taking the expectation of both sides from $0$ to $t-1$ implies
	\begin{align*}
	\mathbb{E}_{0:t-1}\left[V_{\cD}\left(x_{t}, f\right)\right] &\leqslant \sqrt{t+2} \left(\mathbb{E}_{0:t-1}\left[f\left(x_{t}\right)\right]-\mathbb{E}_{0:t}\left[f\left(x_{t+1}\right)\right]\right)+\frac{1}{\sqrt{t+1}}\left(M+\frac{L}{2}\left(1+\Lambda\right)^{2}\right).
	\end{align*}
	Note that Algorithm~\ref{alg} picks $t_0$ uniformly at random from $\set{0,\ldots,T-1}$. Taking the expectation over this $t_0$ then leads to 
	\begin{align}
	\nonumber\mathbb{E}_{[t_0]}\mathbb{E}_{0:t-1}\left[V_{\cD}\left(x_{t}, f\right)\right] &\leqslant\frac{\sqrt{T+1}}{T} \left(f\left(x_0\right)-f\left(x^*\right)\right)\\
	\label{ee8}&+\frac{1}{T}\sum_{t=0}^{T-1}\frac{1}{\sqrt{t+1}}\left(M+\frac{L}{2}\left(1+\Lambda\right)^{2}\right).
	\end{align}
	The idea for the rest of this proof is to bound the right-hand side of \eqref{ee8} by $\epsilon$ and see what sufficient conditions will be enforced on $T$. To this end, it suffices to bound each term by $\epsilon/2$. Applying this bound on the first term, results in
	\begin{align}\label{1468}
	\frac{\sqrt{2}}{\sqrt{T}}\left(f\left(x_{0}\right)-f\left(x^{\star}\right)\right) \leqslant \frac{\epsilon}{2}\quad \Longleftrightarrow\quad \frac{8}{\epsilon^2}\left(f\left(x_{0}\right)-f\left(x^{*}\right)\right)^2 \leqslant T.
	\end{align}
	Moreover, to apply the $\epsilon/2$ bound to the second term, we can write
	\begin{align*}
	\frac{1}{T}\sum_{t=0}^{T-1}\frac{1}{\sqrt{t+1}}\left(M+\frac{L}{2}\left(1+\Lambda\right)^{2}\right)\leqslant\frac{2M+L\left(1+\Lambda\right)^{2}}{\sqrt{T}}\leqslant \frac{\epsilon}{2}.
	\end{align*}
	Thus, it suffices to have
	\begin{equation}\label{third8}
	\frac{C_5}{\epsilon^2}\leqslant T.
	\end{equation}
	Finally, from \eqref{1468} and \eqref{third8}, it suffices to choose $T$ as \eqref{T4}. Further, by putting together \eqref{1w78}, \eqref{1cond18}, \eqref{1w48}, \eqref{1cond28}, and \eqref{1cond38}, it suffices to choose $C=C_2$.
\end{proof}

The following theorem considers the problem under a convex objective function with a stochastic gradient estimator.
\begin{theo}\label{convexstoch}
	Consider Problem~\eqref{P} under Assumptions~\ref{f}-\ref{NFO} and assume $f$ is convex. Suppose $\epsilon >0$ is given and let $\tau>0$ be such that $\cD_{\tau}\neq \emptyset$. Further, let $\eta_t=\rho_t=(t+2)^{-1}$, $n_t=2C_2(t+1)$, and
	\begin{align}\label{T51}
	T= \frac{1}{\epsilon^2}\max\{4(f(x_0)-f(x^*))^2,C_6\}
	\end{align}
	where $C_2$ is given in Definition~\ref{d3}. Then, with probability at least $1-\delta$, in Algorithm~\ref{alg}, all the iterates are safe, i.e., $x_0,\ldots,x_{T-1}\in\cD$, all the query points lie within $r_0$-vicinity of $\cD$, and  $f(x_{T-1})-f(x^*)\leqslant \epsilon$.
\end{theo} 

\begin{proof}
	Note that if $x_0,\ldots,x_{T-1}\in\cD$, then all the query points lie within $r_0$-vicinity of $\cD$ due to the construction of Subroutine~\ref{datacollalg}. Now to achieve $x_0,\ldots,x_{T-1}\in\cD$ with probability at least $1-\delta$, based on Proposition~\ref{nk}, it suffices to have $N_t\geqslant C_1^2/(1+\Gamma)^2$ and \eqref{Nt2}. For these to hold, we prove that it suffices to choose $\eta_t=(t+2)^{-1}$ and $N_t\geqslant C_2(t+1)^2$. To this aim, we let $N_t= C(t+1)^2$ and show that it suffices to have $C=C_2$. As a result, to get $N_t\geqslant C_2(t+1)^2$, it then suffices to choose $n_t=2C_2(t+1)$, where $N_t=n_0+\ldots+n_t$. As the first step, note that
	\begin{align}\label{1w7890}
	\forall t\in\set{0,\ldots,T-1}:\,N_t=C(t+1)\geqslant \frac{C_1^2}{(1+\Gamma)^2} \quad\Longleftrightarrow\quad C \geqslant \frac{C_1^2}{\left(1+\Gamma\right)^2}.
	\end{align}
	Now, let $h_1$ be the second term in the right-hand side of \eqref{w32}, i.e.,
	\begin{align}
	\label{155890} h_1\left(\set{\eta_t},\set{N_t}\right):=\sum_{k=0}^{t-1}\left(\frac{\tau}{2}-\frac{C_1 L_{A}}{\sqrt{N_{k}}}\right) \eta_{k} \prod_{j=k+1}^{t-1}\left(1-\eta_{j}\right).
	\end{align}
	To get a sufficient condition on $C$, we now enforce
	\begin{align}\label{1cond1890}
	\forall t\in\set{0,\ldots, T-1}:\, \frac{\tau}{2}-\frac{C_1 L_{A}}{\sqrt{N_{t}}} = \frac{\tau}{2}-\frac{C_1 L_{A}}{\sqrt{C}\left(t+1\right)} \geqslant \frac{\tau}{4}\,\,\, \Leftrightarrow \,\,\,C \geqslant \left(\frac{4C_1 L_{A}}{\tau}\right)^2.
	\end{align}
	Replacing \eqref{1cond1890} and the values of $\eta_t$ and $N_t$ into \eqref{155890}, results in
	\begin{align}\label{156890}
	h_1\left(\set{\eta_k},\set{N_k}\right)\geqslant\frac{\tau}{4}\sum_{k=0}^{t-1} \eta_{k} \prod_{j=k+1}^{t-1}\left(1-\eta_{j}\right)\geqslant \frac{\tau}{4(t+1)}.
	\end{align}
	As a result, having \eqref{156890} and \eqref{w32}, we can write
	\begin{align}\label{1w5890}
	h\left(\set{\eta_t},\set{N_t}\right)\geqslant h_1\left(\set{\eta_t},\set{N_t}\right)-\kappa\frac{C_0}{\sqrt{N_t}}\geqslant \left(\frac{\tau}{2}-\frac{\kappa\, C_0}{\sqrt{C}}\right)\frac{1}{t+1}.
	\end{align}
	To get another sufficient condition on $C$, we enforce
	\begin{align}\label{1w4890}
	\frac{\tau}{2}-\frac{\kappa\, C_0}{\sqrt{C}}\geqslant \frac{\tau}{4} \quad \Longleftrightarrow \quad C\geqslant \left(\frac{4\kappa\,C_0}{\tau}\right)^2.
	\end{align}
	Substitute \eqref{1w4890} into \eqref{1w5890} to obtain
	\begin{align}\label{1w6890}
	h\left(\set{\eta_t},\set{N_t}\right)\geqslant \frac{\tau}{8(t+1)}.
	\end{align}
	Replacing \eqref{1w6890} into \eqref{Nt2} leads to another sufficient condition on $C$, that is,
	\begin{align}\label{1cond2890}
	C \geqslant \frac{16\,\kappa^2}{\tau^2}\left(1+\frac{d \Lambda^{2}}{r_0^{2}}\right).
	\end{align}
	Note that \eqref{1w7890}, \eqref{1cond1890}, \eqref{1w4890}, and \eqref{1cond2890} together guarantee that $N_t\geqslant C_1^2/(1+\Gamma)^2$ and \eqref{Nt2} hold. Therefore, due to Proposition~\ref{nk}, with probability at least $1-\delta$, we have $x_0,\ldots,x_{T-1}\in\cD$. Now the construction of Subroutine~\ref{datacollalg} ensures that all the query points lie within $r_0$-vicinity of $\cD$. 
	
	Next, we show the convergence. Recall the notation in \eqref{jr} and note that under the assumptions of Proposition~\ref{nk}, it ensures that $\mathbb{P}\{\mathcal{R}\mid\mathcal{J}\}=1$ and thus from \eqref{jj}, it follows that having $\mathcal{J}$ results in $x_0,\ldots,x_{T-1}\in\cD$. Knowing that $\mathcal{J}$ holds with probability at least $1-\delta$, to show that convergence holds simultaneously with the safety together with probability at least $1-\delta$, it suffices to show that the convergence holds under having $\mathcal{J}$. To this aim, suppose $\mathcal{J}$ holds. Having $\mathcal{J}$ and \eqref{1w78} together, we conclude from Lemma~\ref{projection} that \eqref{22} and \eqref{23} hold for every $t\in\{0,\ldots,T-1\}$. Using \eqref{22} and \eqref{23}, the convergence can be proved as follows: First consider
	\begin{align}
	\label{2090} v_{t}&:=\arg \min_{v\in \cD}\, \left\langle \nabla f(x_t), v\right\rangle, \\
	\label{2190} \hat{v}_{t}&:=\arg \min_{v\in \hat{\cD}_t}\, \left\langle g_t, v\right\rangle.
	\end{align}
	Next, based on the concept of curvature constant for the convex function $f$ over $\cD$, we can proceed as follows (see \cite{jaggi2013revisiting}):
	\begin{align}
	\nonumber f(x_{t+1}) &\leqslant f(x_t) + \eta_t\left\langle \nabla f(x_t),\hat{v}_t-x_t\right\rangle+\frac{\eta_t^2}{2}\,L^2\Lambda_t\\
	\label{fe0}&\leqslant f(x_t) + \eta_t\left\langle \nabla f(x_t),\hat{v}_t-v_t\right\rangle+\eta_t\left\langle \nabla f(x_t),v_t-x_t\right\rangle+\frac{\eta_t^2}{2}\,L^2\Lambda_t,
	\end{align}
	where $\Lambda_t$ is the diameter of $\hat{\cD}_t$ and it can be simply shown that $\Lambda_t \leqslant \Lambda +2C_1/\sqrt{N_t}$.
	Further, to bound $\langle \nabla f(x_t), \hat{v}_{t}-v_{t}\rangle$, first note that
	\begin{align}
	\label{259000} \left\langle \nabla f(x_t), \hat{v}_{t}-v_{t}\right\rangle = \left\langle  \nabla f(x_t)-g_t, \hat{v}_{t}-v_{t}\right\rangle +\left\langle g_t, \hat{v}_{t}-v_{t}\right\rangle.
	\end{align}
	To bound the first term in the right-hand side of \eqref{259000}, we can write
	\begin{align}
	\nonumber\left\langle\nabla f\left(x_{t}\right)-g_{t}, \hat{v}_{t}-v_{t}\right\rangle&\leqslant\left\|\nabla f\left(x_{t}\right)-g_{t}\right\|\left(\left\|\hat{v}_{t}-\pi_{{\cD}}\left(\hat{v}_{t}\right) \right\|+\left\|\pi_{{\cD}}\left(\hat{v}_{t}\right)-v_t\right\|\right)\\
	\label{350}&\overset{\eqref{23}}{\leqslant}\left\|\nabla f\left(x_{t}\right)-g_{t}\right\|\left(\frac{C_1}{\sqrt{N_{t}}}+\Lambda\right).
	\end{align}
	Moreover, the second term in the right-hand side of \eqref{259000} can be bounded as follows: First,
	\begin{align}
	\label{2590}\left\langle g_t, \pi_{\hat{\cD}_t}\left(v_{t}\right) -v_{t}\right\rangle  \leqslant \left\|g_t\right\|  \left\|\pi_{\hat{\cD}_t}\left(v_{t}\right) -v_{t}\right\|  \overset{\eqref{23}}{ \leqslant} \left\|g_t\right\| \frac{C_1}{\sqrt{N_{t}}}.
	\end{align}
	Then,
	\begin{align*}
	\left\langle g_t, \hat{v}_{t}\right\rangle\overset{\eqref{2190}}{ \leqslant}\left\langle g_t, \pi_{\hat{\cD}_t}\left(v_{t}\right) \right\rangle \overset{\eqref{2590}}{ \leqslant} \left\langle g_t, v_{t}\right\rangle+ \left\|g_t\right\| \frac{C_1}{\sqrt{N_{t}}}, 
	\end{align*}
	which results in 
	\begin{align}
	\label{3790}\left\langle g_t, \hat{v}_{t}-v_{t}\right\rangle \leqslant \left\|g_t\right\|  \frac{C_1}{\sqrt{N_{t}}}.
	\end{align}
	Replace \eqref{3790} and \eqref{350} into \eqref{259000}, then results in 
	\begin{align}
	\label{2590002} \left\langle \nabla f(x_t), \hat{v}_{t}-v_{t}\right\rangle \leqslant 
	\left\|\nabla f\left(x_{t}\right)-g_{t}\right\|\left(\frac{C_1}{\sqrt{N_{t}}}+\Lambda\right) + \left\|g_t\right\|  \frac{C_1}{\sqrt{N_{t}}}.
	\end{align}
	Replacing into \eqref{fe0}, the inequality \eqref{2590002}, the bound on $\Lambda_t$, and the fact that $f$ is convex, leads to
	\begin{align*}
	f(x_{t+1})-f(x^*) &\leqslant \left(1-\eta_t\right)\left(f(x_t)-f(x^*)\right) + \eta_t\left\|\nabla f\left(x_{t}\right)-g_{t}\right\|\left(\frac{C_1}{\sqrt{N_{t}}}+\Lambda\right)\\
	& + \eta_t\left\|g_t\right\|  \frac{C_1}{\sqrt{N_{t}}}+\frac{\eta_t^2}{2}\,L^2\left(\Lambda+\frac{2C_1}{\sqrt{N_t}}\right),
	\end{align*}
	where $x^*=\arg\min_{x\in\cD}f(x)$.
	To simplify, we enforce the following sufficient condition:
	\begin{align}\label{1cond3890}
	\frac{C_1}{\sqrt{C}}\leqslant 1\quad \Longleftrightarrow\quad C\geqslant C_1^2,
	\end{align}
	and we consider the bound on $\|g_t\|$ and the bound on $\|\nabla f(x_{t})-g_{t}\|$ from Lemma~\ref{GSL} (with $\alpha=1$). These lead to
	\begin{align}
	\label{311}f(x_{t+1})-f(x^*) \leqslant (1-\eta_t)(f(x_t)-f(x^*)) + \eta_t \frac{C_7}{\sqrt{t+2}}(1+\Lambda) +\eta_t\frac{M+\sigma_0}{\sqrt{t+1}}+\frac{\eta_t^2}{2}L^2(\Lambda+2),
	\end{align}
	where 
	\begin{align}\label{storm2}
	C_7:= 4 \left( L_0 \Lambda +  \sigma_0\right)\sqrt{2\, \log\left(\frac{4}{\delta}\right)}.
	\end{align}
	Now define 
	\begin{align}
	C_8 := \max\{f(x_0)-f(x^*),4C_7(1+\Lambda),4\sqrt{2}\left(M+\sigma_0\right),L^2(\Lambda+2)\}.
	\end{align}
	We claim that $f(x_{t})-f(x^*)\leqslant 2C_8/\sqrt{t+2}$ holds for all $t$. To show this claim by induction, note that it clearly holds for $t=0$ and using \eqref{311}, we can write	 
	\begin{align*}
	f(x_{t+1})-f(x^*) &\leqslant \left(1-\frac{1}{t+2}\right)\frac{2C_8}{\sqrt{t+2}} +\frac{C_8}{4(t+2)^{\frac 32}}+ \frac{C_8}{4(t+2)^{\frac 32}}\, +\frac{C_8}{2\left(t+2\right)^2}\\
	&\leqslant \frac{2C_8}{(t+2)^{\frac 32}}\left(t+\frac 32\right) \leqslant \frac{2C_8}{\sqrt{t+3}}.
	\end{align*}
	Hence, $f(x_{T-1})-f(x^*) \leqslant {2C_8}/{\sqrt{T+1}}\leqslant{2C_8}/\sqrt{T} $. Thus, we can write
	\begin{align*}
	f(x_{T-1})-f(x^*) \leqslant \frac{2C_8}{\sqrt{T}} \leqslant\epsilon\quad \Longleftrightarrow \quad\frac{4C_8^2}{\epsilon^2}\leqslant T.
	\end{align*}
	Note that $C_8^2=\max\{(f(x_0)-f(x^*))^2,C_6\}$. Further, by putting together \eqref{1w7890}, \eqref{1cond1890}, \eqref{1w4890}, \eqref{1cond2890}, and \eqref{1cond3890}, it suffices to set $C=C_2$.
\end{proof}

Finally, the following theorem studies the problem under a convex objective function with a known gradient.
\begin{theo}\label{convex}
	Consider Problem~\eqref{P} under Assumptions~\ref{f}, \ref{D}, and \ref{NFO}. Moreover, assume $f$ is convex and we have access to $\nabla f$. Suppose $\epsilon >0$ is given and let $\tau>0$ be such that $\cD_{\tau}\neq \emptyset$. Further, let $\eta_t=2(t+2)^{-1}$, $\rho_t=1$, $n_t=2C_2(t+1)$, and
	\begin{align}\label{T49}
	T= \frac{2}{\epsilon}\max\{f(x_0)-f(x^*),4M,2L^2(\Lambda+2)\},
	\end{align}
	where $C_2$ is given in Definition~\ref{d3}. Then, with probability at least $1-\delta$, in Algorithm~\ref{alg}, all the iterates are safe, i.e., $x_0,\ldots,x_{T-1}\in\cD$, all the query points lie within $r_0$-vicinity of $\cD$, and  $f(x_{T-1})-f(x^*)\leqslant \epsilon$.
\end{theo} 
\begin{proof}
	Note that if $x_0,\ldots,x_{T-1}\in\cD$, then all the query points lie within $r_0$-vicinity of $\cD$ due to the construction of Subroutine~\ref{datacollalg}. Now to achieve $x_0,\ldots,x_{T-1}\in\cD$ with probability at least $1-\delta$, based on Proposition~\ref{nk}, it suffices to have $N_t\geqslant C_1^2/(1+\Gamma)^2$ and \eqref{Nt2}. For these to hold, we prove that it suffices to choose $\eta_t=2(t+2)^{-1}$ and $N_t\geqslant C_2(t+1)^2$. To this aim, we let $N_t= C(t+1)^2$ and show that it suffices to have $C=C_2$. As a result, to get $N_t\geqslant C_2(t+1)^2$, it then suffices to choose $n_t=2C_2(t+1)$, where $N_t=n_0+\ldots+n_t$. As the first step, note that
	\begin{align}\label{1w789}
	\forall t\in\set{0,\ldots,T-1}:\,N_t=C(t+1)\geqslant \frac{C_1^2}{(1+\Gamma)^2} \quad\Longleftrightarrow\quad C \geqslant \frac{C_1^2}{\left(1+\Gamma\right)^2}.
	\end{align}
	Now, let $h_1$ be the second term in the right-hand side of \eqref{w32}, i.e.,
	\begin{align}
	\label{15589} h_1\left(\set{\eta_t},\set{N_t}\right):=\sum_{k=0}^{t-1}\left(\frac{\tau}{2}-\frac{C_1 L_{A}}{\sqrt{N_{k}}}\right) \eta_{k} \prod_{j=k+1}^{t-1}\left(1-\eta_{j}\right).
	\end{align}
	To get a sufficient condition on $C$, we now enforce
	\begin{align}\label{1cond189}
	\forall t\in\set{0,\ldots, T-1}:\, \frac{\tau}{2}-\frac{C_1 L_{A}}{\sqrt{N_{t}}} = \frac{\tau}{2}-\frac{C_1 L_{A}}{\sqrt{C}\left(t+1\right)} \geqslant \frac{\tau}{4}\,\,\, \Leftrightarrow \,\,\,C \geqslant \left(\frac{4C_1 L_{A}}{\tau}\right)^2.
	\end{align}
	Replacing \eqref{1cond189} and the values of $\eta_t$ and $N_t$ into \eqref{15589}, results in
	\begin{align}\label{15689}
	h_1\left(\set{\eta_k},\set{N_k}\right)\geqslant\frac{\tau}{4}\sum_{k=0}^{t-1} \eta_{k} \prod_{j=k+1}^{t-1}\left(1-\eta_{j}\right)\geqslant \frac{\tau}{2(t+1)}.
	\end{align}
	As a result, having \eqref{15689} and \eqref{w32}, we can write
	\begin{align}\label{1w589}
	h\left(\set{\eta_t},\set{N_t}\right)\geqslant h_1\left(\set{\eta_t},\set{N_t}\right)-\kappa\frac{C_0}{\sqrt{N_t}}\geqslant \left(\frac{\tau}{2}-\frac{\kappa\, C_0}{\sqrt{C}}\right)\frac{1}{t+1}.
	\end{align}
	To get another sufficient condition on $C$, we enforce
	\begin{align}\label{1w489}
	\frac{\tau}{2}-\frac{\kappa\, C_0}{\sqrt{C}}\geqslant \frac{\tau}{4} \quad \Longleftrightarrow \quad C\geqslant \left(\frac{4\kappa\,C_0}{\tau}\right)^2.
	\end{align}
	Substitute \eqref{1w489} into \eqref{1w589} to obtain
	\begin{align}\label{1w689}
	h\left(\set{\eta_t},\set{N_t}\right)\geqslant \frac{\tau}{4(t+1)}.
	\end{align}
	Replacing \eqref{1w689} into \eqref{Nt2} leads to another sufficient condition on $C$, that is,
	\begin{align}\label{1cond289}
	C \geqslant \frac{16\,\kappa^2}{\tau^2}\left(1+\frac{d \Lambda^{2}}{r_0^{2}}\right).
	\end{align}
	Note that \eqref{1w789}, \eqref{1cond189}, \eqref{1w489}, and \eqref{1cond289} together guarantee that $N_t\geqslant C_1^2/(1+\Gamma)^2$ and \eqref{Nt2} hold. Therefore, due to Proposition~\ref{nk}, with probability at least $1-\delta$, we have $x_0,\ldots,x_{T-1}\in\cD$. Now the construction of Subroutine~\ref{datacollalg} ensures that all the query points lie within $r_0$-vicinity of $\cD$. 
	
	Next, we show the convergence. Recall the notation in \eqref{jr} and note that under the assumptions of Proposition~\ref{nk}, it ensures that $\mathbb{P}\{\mathcal{R}\mid\mathcal{J}\}=1$ and thus from \eqref{jj}, it follows that having $\mathcal{J}$ results in $x_0,\ldots,x_{T-1}\in\cD$. Knowing that $\mathcal{J}$ holds with probability at least $1-\delta$, to show that convergence holds simultaneously with the safety together with probability at least $1-\delta$, it suffices to show that the convergence holds under having $\mathcal{J}$. To this aim, suppose $\mathcal{J}$ holds. Having $\mathcal{J}$ and \eqref{1w78} together, we conclude from Lemma~\ref{projection} that \eqref{22} and \eqref{23} hold for every $t\in\{0,\ldots,T-1\}$. Using \eqref{22} and \eqref{23}, the convergence can be proved as follows: First consider
	\begin{align}
	\label{209} v_{t}&:=\arg \min_{v\in \cD}\, \left\langle \nabla f(x_t), v\right\rangle, \\
	\label{219} \hat{v}_{t}&:=\arg \min_{v\in \hat{\cD}_t}\, \left\langle \nabla f(x_t), v\right\rangle.
	\end{align}
	Next, based on the concept of curvature constant for the convex function $f$ over $\cD$, we can proceed as follows (see \cite{jaggi2013revisiting}):
	\begin{align}
	\nonumber f(x_{t+1}) &\leqslant f(x_t) + \eta_t\left\langle \nabla f(x_t),\hat{v}_t-x_t\right\rangle+\frac{\eta_t^2}{2}\,L^2\Lambda_t\\
	\label{fe}&\leqslant f(x_t) + \eta_t\left\langle \nabla f(x_t),\hat{v}_t-v_t\right\rangle+\eta_t\left\langle \nabla f(x_t),v_t-x_t\right\rangle+\frac{\eta_t^2}{2}\,L^2\Lambda_t,
	\end{align}
	where $\Lambda_t$ is the diameter of $\hat{\cD}_t$ and it can be simply shown that $\Lambda_t \leqslant \Lambda +2C_1/\sqrt{N_t}$.
	Further, to bound $\left\langle \nabla f(x_t), \hat{v}_{t}-v_{t}\right\rangle$, first note that
	\begin{align}
	\label{259}\left\langle \nabla f(x_t), \pi_{\hat{\cD}_t}\left(v_{t}\right) -v_{t}\right\rangle  \leqslant \left\|\nabla f(x_t)\right\|  \left\|\pi_{\hat{\cD}_t}\left(v_{t}\right) -v_{t}\right\|  \overset{\eqref{23}}{ \leqslant} M \frac{C_1}{\sqrt{N_{t}}}.
	\end{align}
	Hence,
	\begin{align*}
	\left\langle \nabla f(x_t), \hat{v}_{t}\right\rangle\overset{\eqref{219}}{ \leqslant}\left\langle \nabla f(x_t), \pi_{\hat{\cD}_t}\left(v_{t}\right) \right\rangle \overset{\eqref{259}}{ \leqslant} \left\langle \nabla f(x_t), v_{t}\right\rangle+M \frac{C_1}{\sqrt{N_{t}}}, 
	\end{align*}
	which results in 
	\begin{align}
	\label{379}\left\langle \nabla f(x_t), \hat{v}_{t}-v_{t}\right\rangle \leqslant M \frac{C_1}{\sqrt{N_{t}}}.
	\end{align}
	Replacing into \eqref{fe}, the inequality \eqref{379}, the bound on $\Lambda_t$, and the fact that $f$ is convex, leads to
	\begin{align*}
	f(x_{t+1})-f(x^*) \leqslant \left(1-\eta_t\right)\left(f(x_t)-f(x^*)\right) + \eta_t\, \frac{M C_1}{\sqrt{N_{t}}}+\frac{\eta_t^2}{2}\,L^2\left(\Lambda+\frac{2C_1}{\sqrt{N_t}}\right),
	\end{align*}
	where $x^*=\arg\min_{x\in\cD}f(x)$.
	To simplify, we enforce the following sufficient condition:
	\begin{align}\label{1cond389}
	\frac{C_1}{\sqrt{C}}\leqslant 1\quad \Longleftrightarrow\quad C\geqslant C_1^2,
	\end{align}
	which results in
	\begin{align*}
	f(x_{t+1})-f(x^*) \leqslant \left(1-\eta_t\right)\left(f(x_t)-f(x^*)\right) + \eta_t\, \frac{M}{t+1} +\frac{\eta_t^2}{2}\,L^2\left(\Lambda+2\right).
	\end{align*}
	Now define 
	\begin{align}
	C_9 := \max\{f(x_0)-f(x^*),4M,2L^2(\Lambda+2)\}.
	\end{align}
	We claim that $f(x_{t})-f(x^*)\leqslant 2C_9/(t+2)$ holds for all $t$. To show this claim by induction, note that it clearly holds for $t=0$ and	 
	\begin{align*}
	f(x_{t+1})-f(x^*) &\leqslant \left(1-\frac{2}{t+2}\right)\frac{2C_9}{t+2} + \frac{4}{(t+2)^2}\, M +\frac{1}{2}\left(\frac{2}{t+2}\right)^2L^2\left(\Lambda+2\right)\\
	&\leqslant\left(1-\frac{2}{t+2}\right)\frac{2C_9}{t+2} + \frac{C_9}{(t+2)^2} +\frac{C_9}{4}\left(\frac{2}{t+2}\right)^2\\
	&=\frac{2C_9(t+1)}{(t+2)^2}\leqslant \frac{2C_9}{(t+3)}.
	\end{align*}
	Hence, $f(x_{T-1})-f(x^*) \leqslant {2C_9}/{(T+1)}\leqslant{2C_9}/T $. Thus, we can write
	\begin{align*}
	f(x_{T-1})-f(x^*) \leqslant \frac{2C_9}{T} \leqslant\epsilon\quad \Longleftrightarrow \quad\frac{2C_6}{\epsilon}\leqslant T.
	\end{align*}
	Further, by putting together \eqref{1w789}, \eqref{1cond189}, \eqref{1w489}, \eqref{1cond289}, and \eqref{1cond389}, it suffices to set $C=C_2$.
\end{proof}
\section{Linear Algebra Lemmas}
In this section, we provide some lemmas from linear algebra which were previously used in Appendix.
\begin{lemma}[Re-parametrization of an ellipsoid]\label{ellipse}
	Suppose $x_0\in\mathbb{R}^d$ and $\Sigma\in\mathbb{R}^{d\times d}$ is positive definite. Then 
	\[E=\left\{x\in\mathbb{R}^{d}:\,\,\left(x_{0}-x\right)^{\top} \Sigma^{-1}\left(x_{0}-x\right) \leqslant r^{2}\right\}=\left\{x_{0}-r\, \Sigma^{1 / 2}u: \,\, u\in\mathbb{R}^d,\,\, \|u\|\, \leqslant 1\right\}.\]
\end{lemma}
\begin{proof}
	Define
	\begin{equation*}
	u:=\frac{1}{r}\,\Sigma^{-1 / 2}\left(x_{0}-x\right).
	\end{equation*}
	Therefore,
	\begin{align*}
	x\in E \quad &\Leftrightarrow\quad x= x_{0}-r\, \Sigma^{1 / 2} u, \quad r^2\,u^{\top}u \leqslant r^2\\
	&\Leftrightarrow \quad x= x_{0}-r\, \Sigma^{1 / 2} u, \quad \|u\|\, \leqslant 1.
	\end{align*}
\end{proof}

\begin{lemma}[Sherman Morrison Woodbury Formula]\label{matrix-inversion2} 
	Suppose $A$, $B$, and $C$ are matrices such that $A$, $C$, and $A+BCD$ are well-defined and invertible. Then 
	\begin{align*}
	(A+B C D)^{-1}=A^{-1}-A^{-1} B\left(C^{-1}+D A^{-1} B\right)^{-1} D A^{-1}.
	\end{align*}
\end{lemma}
\begin{proof}
	Note that if $P$ is a square matrix such that $I+P$ is invertible, then
	\begin{align}\label{50}
	(I+P)\left(I-(I+P)^{-1} P\right)  = I+P-P=I\quad  \Rightarrow \quad(I+P)^{-1}= I-(I+P)^{-1} P.
	\end{align}
	Also, note that if $U$ and $V$ are such that $I+UV$ is invertible, then $I+VU$ is also invertible. To see that, suppose $\left(I+VU\right)x=0$. We need to show that $x=0$. We have
	\begin{align}\label{52}
	\left(I+VU\right)x=0\,\, \Rightarrow \,\,U\left(I+VU\right)x=0\,\,  \Rightarrow \,\, \left(I+UV\right)Ux=0 \,\, \overset{\left(I+UV\right)^{-1}}{\Longrightarrow} \,\, Ux=0.
	\end{align}
	Thus,
	\begin{align*}
	\left(I+VU\right)x=0\,\, \Rightarrow \,\,x = -V(Ux) \overset{\eqref{52}}{=}0.
	\end{align*}
	Hence, if $U$ and $V$ are such that $I+UV$ is invertible, then $I+VU$ is invertible and we can write
	\begin{align}\label{51}
	(I+UV)U = U(I+VU)\quad  \Rightarrow \quad U(I+V U)^{-1}=(I+U V)^{-1} U.
	\end{align}
	Moreover, we can conclude that under these assumptions,
	\begin{align}\label{53}
	(I+U V)^{-1} \overset{\eqref{50}}{=} I-(I+U V)^{-1} U V \overset{\eqref{51}}{=} I-U(I+V U)^{-1} V.
	\end{align}
	Now by replacing $U = A^{-1}B$ and $V=CD$, \eqref{53} results in
	\begin{align*}
	(A+B C D)^{-1} &= \left(I+A^{-1}BCD\right)^{-1}A^{-1}\overset{\eqref{53}}{=}\left(I- A^{-1}B\left(I+CDA^{-1}B\right)^{-1}CD\right)A^{-1}\\
	&=A^{-1}- A^{-1}B\left(I+CDA^{-1}B\right)^{-1}CDA^{-1}\\
	&=A^{-1}- A^{-1}B\left(C^{-1}+DA^{-1}B\right)^{-1}DA^{-1}.
	\end{align*}
\end{proof}

\begin{lemma}\label{matrix-inversion}
	If the right-hand side of \eqref{48} is well-defined, we have 
	\begin{align}\label{48}
	\left[\begin{array}{ll}{A} & {B} \\ {C} & {D}\end{array}\right]^{-1}=\left[\begin{array}{cc}{\left(A-B D^{-1} C\right)^{-1}} & {-\left(A-B D^{-1} C\right)^{-1} B D^{-1}} \\ {-D^{-1} C\left(A-B D^{-1} C\right)^{-1}} & {\left(D-C A^{-1} B\right)^{-1}}\end{array}\right].
	\end{align}
\end{lemma}
\begin{proof}
	Simply through a matrix multiplication and using Lemma~\ref{matrix-inversion2}.
\end{proof}

\begin{lemma}\label{9}
	Consider $x\in\mathbb{R}^d$ and the identity matrix $I\in\mathbb{R}^{d\times d}$. Then $\left\|\left[{I}, {x}\right]\right\|=\sqrt{1+\left\|x\right\|^{2}}$.
\end{lemma}
\begin{proof}
	Let $\lambda_{\max}(\cdot)$ denote the maximum eigenvalue of the corresponding matrix. Then
	\begin{align*}
	\left\|\left[{I}, {x}\right]\right\|=\sqrt{\lambda_{\max}\left(\left[I, x\right]\left[\begin{array}{l}{I} \\ {x^{\top}}\end{array}\right] \right)}=\sqrt{\lambda_{\max}\left(I+xx^{\top}\right)}=\sqrt{I+\lambda_{\max} \left(xx^{\top}\right)}=\sqrt{1+\left\|x\right\|^{2}}.
	\end{align*}
\end{proof}

\begin{lemma}\label{matrix-inversion3}
	Suppose $A$ and $B$ are matrices such that $A$ and $A+B$ are well-defined and invertible. Then 
	\begin{align*}
	{(A+B)^{-1}=A^{-1}-\left(I+A^{-1} B\right)^{-1} A^{-1} B A^{-1}}.
	\end{align*}
\end{lemma}
\begin{proof}
	Replace $P=A^{-1}B$ in \eqref{50} to get
	\begin{align*}
	(A+B)^{-1}&=\left(I+A^{-1}B\right)^{-1}A^{-1}   \overset{\eqref{50}}{=} \left(I-\left(I+A^{-1}B\right)^{-1} A^{-1}B\right)A^{-1}\\
	&= A^{-1}-\left(I+A^{-1} B\right)^{-1} A^{-1} B A^{-1}.
	\end{align*}
\end{proof}

\begin{lemma}\label{matrix-norm}
	If  $\|A\| \leqslant 1$,  then  $\left\|(I+A)^{-1}\right\| \leqslant \frac{1}{1-\|A\|}$.
\end{lemma}
\begin{proof}
	Note that
	\begin{align*}
	\left\|(I+A)^{-1}\right\|\overset{\eqref{50}}{=}\left\|I-(I+A)^{-1} A\right\| \leqslant 1+\left\|(I+A)^{-1}\right\|\|A \|\,\Rightarrow\,\left\|(I+A)^{-1}\right\|\left(1-\|A\|\right) \leqslant 1.
	\end{align*}
	Since  $\|A\| \leqslant 1$,  we get $\left\|(I+A)^{-1}\right\| \leqslant \frac{1}{1-\|A\|}$.
\end{proof}

\bibliographystyle{ieeetr} 
\bibliography{ref}

\begin{thebibliography}{10}

\bibitem{SafeAut}
M.~Cummings and D.~Britton, ``Chapter 6 - regulating safety-critical autonomous
  systems: past, present, and future perspectives,'' in {\em Living with
  Robots} (R.~Pak, E.~J. de~Visser, and E.~Rovira, eds.), pp.~119 -- 140,
  Academic Press, 2020.

\bibitem{SafeRobot}
L.~{Wellhausen}, R.~{Ranftl}, and M.~{Hutter}, ``Safe robot navigation via
  multi-modal anomaly detection,'' {\em IEEE Robotics and Automation Letters},
  vol.~5, pp.~1326--1333, April 2020.

\bibitem{med}
H.~Alemzadeh, R.~K. Iyer, Z.~Kalbarczyk, and J.~Raman, ``Analysis of
  safety-critical computer failures in medical devices,'' {\em IEEE Security \&
  Privacy}, vol.~11, no.~4, pp.~14--26, 2013.

\bibitem{nesterov2013introductory}
Y.~Nesterov, {\em Introductory lectures on convex optimization: A basic
  course}, vol.~87.
\newblock Springer Science \& Business Media, 2013.

\bibitem{frank1956algorithm}
M.~Frank and P.~Wolfe, ``An algorithm for quadratic programming,'' {\em Naval
  research logistics quarterly}, vol.~3, no.~1-2, pp.~95--110, 1956.

\bibitem{jaggi2013revisiting}
M.~Jaggi, ``Revisiting frank-wolfe: Projection-free sparse convex
  optimization.,'' in {\em Proceedings of the 30th international conference on
  machine learning}, no.~CONF, pp.~427--435, 2013.

\bibitem{FW3}
R.~M. Freund and P.~Grigas, ``New analysis and results for the frank?wolfe
  method,'' {\em Mathematical Programming}, vol.~155, pp.~199--230, 2016.

\bibitem{robbins1951stochastic}
H.~Robbins and S.~Monro, ``A stochastic approximation method,'' {\em The annals
  of mathematical statistics}, pp.~400--407, 1951.

\bibitem{nemirovski1978cezari}
A.~Nemirovski and D.~Yudin, ``On {Cezari's} convergence of the steepest descent
  method for approximating saddle point of convex-concave functions,'' in {\em
  Soviet Math. Dokl}, vol.~19, 1978.

\bibitem{nemirovskii1983problem}
A.~Nemirovskii, D.~B. Yudin, and E.~R. Dawson, ``Problem complexity and method
  efficiency in optimization,'' 1983.

\bibitem{DBLP:conf/icml/HazanK12}
E.~Hazan and S.~Kale, ``Projection-free online learning,'' in {\em Proceedings
  of the 29th International Conference on Machine Learning, {ICML} 2012,
  Edinburgh, Scotland, UK, June 26 - July 1, 2012}, pp.~1843--1850, 2012.

\bibitem{DBLP:conf/icml/HazanL16}
E.~Hazan and H.~Luo, ``Variance-reduced and projection-free stochastic
  optimization,'' in {\em Proceedings of the 33nd International Conference on
  Machine Learning, {ICML} 2016, New York City, NY, USA, June 19-24, 2016},
  pp.~1263--1271, 2016.

\bibitem{lan2016conditional}
G.~Lan and Y.~Zhou, ``Conditional gradient sliding for convex optimization,''
  {\em SIAM Journal on Optimization}, vol.~26, no.~2, pp.~1379--1409, 2016.

\bibitem{mokhtari2018stochastic}
A.~Mokhtari, H.~Hassani, and A.~Karbasi, ``Stochastic conditional gradient
  methods: From convex minimization to submodular maximization,'' {\em Journal
  of Machine Learning Research}, vol.~21, no.~105, pp.~1--49, 2020.

\bibitem{pmlr-v80-qu18a}
C.~Qu, Y.~Li, and H.~Xu, ``Non-convex conditional gradient sliding,'' in {\em
  Proceedings of the 35th International Conference on Machine Learning}, 2018.

\bibitem{scl}
I.~Usmanova, A.~Krause, and M.~Kamgarpour, ``Safe convex learning under
  uncertain constraints,'' in {\em The 22nd International Conference on
  Artificial Intelligence and Statistics}, pp.~2106--2114, 2019.

\bibitem{LB}
I.~Usmanova, A.~Krause, and M.~Kamgarpour, ``Log barriers for safe non-convex
  black-box optimization,'' {\em arXiv preprint arXiv:1912.09478}, 2019.

\bibitem{lacoste2016convergence}
S.~Lacoste-Julien, ``Convergence rate of frank-wolfe for non-convex
  objectives,'' {\em arXiv preprint arXiv:1607.00345}, 2016.

\bibitem{ghadimi2016accelerated}
S.~Ghadimi and G.~Lan, ``Accelerated gradient methods for nonconvex nonlinear
  and stochastic programming,'' {\em Mathematical Programming}, 2016.

\bibitem{ghadimi2016mini}
S.~Ghadimi, G.~Lan, and H.~Zhang, ``Mini-batch stochastic approximation methods
  for nonconvex stochastic composite optimization,'' {\em Mathematical
  Programming}, vol.~155, no.~1-2, pp.~267--305, 2016.

\bibitem{mokhtari2018escaping}
A.~Mokhtari, A.~Ozdaglar, and A.~Jadbabaie, ``Escaping saddle points in
  constrained optimization,'' in {\em Advances in Neural Information Processing
  Systems}, pp.~3629--3639, 2018.

\bibitem{shen2019complexities}
Z.~Shen, C.~Fang, P.~Zhao, J.~Huang, and H.~Qian, ``Complexities in
  projection-free stochastic non-convex minimization,'' in {\em The 22nd
  International Conference on Artificial Intelligence and Statistics},
  pp.~2868--2876, 2019.

\bibitem{yurtsever2019conditional}
A.~Yurtsever, S.~Sra, and V.~Cevher, ``Conditional gradient methods via
  stochastic path-integrated differential estimator,'' in {\em ICML}, 2019.

\bibitem{hassani2019stochastic}
H.~Hassani, A.~Karbasi, A.~Mokhtari, and Z.~Shen, ``Stochastic conditional
  gradient++,'' {\em arXiv preprint arXiv:1902.06992}, 2019.

\bibitem{carmon2019lower}
Y.~Carmon, J.~C. Duchi, O.~Hinder, and A.~Sidford, ``Lower bounds for finding
  stationary points i,'' {\em Mathematical Programming}, pp.~1--50, 2019.

\bibitem{arjevani2019lower}
Y.~Arjevani, Y.~Carmon, J.~C. Duchi, D.~J. Foster, N.~Srebro, and B.~Woodworth,
  ``Lower bounds for non-convex stochastic optimization,'' {\em arXiv preprint
  arXiv:1912.02365}, 2019.

\bibitem{CB1}
H.~Yu, M.~J. Neely, and X.~Wei, ``Online convex optimization with stochastic
  constraints,'' in {\em NIPS}, 2017.

\bibitem{CB2}
H.~Yu and M.~J. Neely, ``A low complexity algorithm with $o(\sqrt{T})$ regret
  and $o(1)$ constraint violations for online convex optimization with long
  term constraints,'' 2016.

\bibitem{CB3}
S.~Zymler, D.~Kuhn, and B.~Rustem, ``Distributionally robust joint chance
  constraints with second-order moment information,'' {\em Mathematical
  Programming}, vol.~137, no.~1, pp.~167--198, 2013.

\bibitem{CB4}
G.~Calafiore and M.~C. Campi, ``Uncertain convex programs: randomized solutions
  and confidence levels,'' {\em Mathematical Programming}, vol.~102, no.~1,
  pp.~25--46, 2005.

\bibitem{CB5}
W.~Sun, D.~Dey, and A.~Kapoor, ``Safety-aware algorithms for adversarial
  contextual bandit,'' in {\em Proceedings of the 34th International Conference
  on Machine Learning}, vol.~70, pp.~3280--3288, PMLR, 2017.

\bibitem{CB6}
M.~Mahdavi, R.~Jin, and T.~Yang, ``Trading regret for efficiency: online convex
  optimization with long term constraints,'' {\em J. Mach. Learn. Res.},
  vol.~13, pp.~2503--2528, 2011.

\bibitem{CB7}
L.~E. Ghaoui, F.~Oustry, and H.~Lebret, ``Robust solutions to uncertain
  semidefinite programs,'' {\em SIAM J. on Optimization}, vol.~9, no.~1,
  pp.~33--52, 1998.

\bibitem{CB8}
A.~Ben-Tal and A.~Nemirovski, ``Robust solutions of linear programming problems
  contaminated with uncertain data,'' {\em Mathematical Programming}, vol.~88,
  pp.~411--424, 2000.

\bibitem{CB9}
K.~Balasubramanian and S.~Ghadimi, ``Zeroth-order (non)-convex stochastic
  optimization via conditional gradient and gradient updates,'' in {\em
  NeurIPS}, 2018.

\bibitem{cutkosky2019momentum}
A.~Cutkosky and F.~Orabona, ``Momentum-based variance reduction in non-convex
  sgd,'' in {\em Advances in Neural Information Processing Systems},
  pp.~15210--15219, 2019.

\bibitem{proof1}
D.~Bertsimas, D.~B. Brown, and C.~Caramanis, ``Theory and applications of
  robust optimization,'' {\em SIAM Review}, vol.~53, pp.~464--501, 2011.

\bibitem{proof2}
A.~Ben-Tal, T.~And, and A.~Nemirovski, ``Robust convex optimization,'' {\em
  Mathematics of Operations Research - MOR}, vol.~23, 11 1998.

\bibitem{Dani}
V.~Dani, T.~P. Hayes, and S.~M. Kakade, ``Stochastic linear optimization under
  bandit feedback,'' 2008.

\bibitem{num}
M.~Maier, A.~Rupenyan, M.~Akbari, R.~Zwicker, and K.~Wegener Procedia CIRP,
  2019-12-31.
\newblock 13th CIRP Conference on Intelligent Computation in Manufacturing
  Engineering, CIRP ICME ˈ19; Conference Location: Gulf of Naples, Italy;
  Conference Date: 14 - 17 July 2020.

\bibitem{xie2019stochastic}
J.~Xie, Z.~Shen, C.~Zhang, H.~Qian, and B.~Wang, ``Stochastic recursive
  gradient-based methods for projection-free online learning,'' {\em arXiv
  preprint arXiv:1910.09396}, 2019.

\end{thebibliography}

\end{document}